\documentclass{article}

\PassOptionsToPackage{numbers,compress,sort&compress}{natbib}

\usepackage[preprint]{neurips_2026}      

\usepackage[utf8]{inputenc}
\usepackage[T1]{fontenc}
\usepackage{amsmath,amssymb,amsthm}
\usepackage{mathtools}
\usepackage{bm}
\usepackage{booktabs}
\usepackage{graphicx}
\usepackage[table]{xcolor}
\usepackage{microtype}
\usepackage{nicefrac}
\usepackage{url}
\usepackage[hidelinks]{hyperref}
\usepackage{enumitem}
\usepackage{algorithm}
\usepackage{algpseudocode}
\usepackage{multirow}
\usepackage{caption}
\usepackage{placeins}

\usepackage{titlesec}
\titlespacing*{\paragraph}{0pt}{1pt plus 1pt minus 0pt}{3pt}
\titlespacing*{\section}{0pt}{6pt plus 2pt minus 1pt}{3pt plus 1pt}
\titlespacing*{\subsection}{0pt}{5pt plus 1pt minus 1pt}{2pt plus 1pt}

\newtheorem{theorem}{Theorem}
\newtheorem{lemma}[theorem]{Lemma}
\newtheorem{proposition}[theorem]{Proposition}
\newtheorem{corollary}[theorem]{Corollary}
\theoremstyle{definition}

\newtheorem{assumption}[theorem]{Assumption}

\theoremstyle{remark}
\providecommand{\E}{\mathbb{E}}
\providecommand{\Prob}{\mathbb{P}}
\providecommand{\Var}{\mathrm{Var}}
\providecommand{\Cov}{\mathrm{Cov}}
\providecommand{\Ind}{\mathbb{1}}
\providecommand{\R}{\mathbb{R}}

\providecommand{\Normal}{\mathcal{N}}

\providecommand{\diag}{\operatorname{diag}}
\providecommand{\tr}{\operatorname{tr}}

\providecommand{\Cset}{\hat{C}}                         

\newcommand{\meth}[1]{\textsc{\small #1}}

\providecommand{\FactorCGIF}{\meth{FactorCGIF}}
\providecommand{\ACIFactorCGIF}{\meth{ACIFactorCGIF}}
\providecommand{\AgACIGroupCGIF}{\meth{AgACIGroupCGIF}}

\providecommand{\GroupCGIF}{\meth{GroupCGIF}}
\providecommand{\SpectralCGIF}{\meth{SpectralCGIF}}
\providecommand{\ACIPerGroupFactorCGIF}{\meth{ACIPerGroupFactorCGIF}}

\providecommand{\CGIFFiltered}{\meth{CGIFFiltered}}
\providecommand{\GraphLGSSM}{\meth{GraphLGSSM}}
\providecommand{\GraphNeuralSSM}{\meth{GraphNeuralSSMFiltered}}

\providecommand{\ACI}{\meth{ACI}}
\providecommand{\AgACI}{\meth{AgACI}}

\providecommand{\metrla}{\textsc{\small metr-la}}
\providecommand{\pemsbay}{\textsc{\small pems-bay}}
\providecommand{\aqi}{\textsc{\small aqi}}
\providecommand{\elec}{\textsc{\small elec}}
\providecommand{\ett}{\textsc{\small ett}}
\providecommand{\solar}{\textsc{\small solar}}
\providecommand{\jena}{\textsc{\small jena}}
\providecommand{\loopseattle}{\textsc{\small loop-sea}}

\providecommand{\Acl}{A_{\mathrm{cl}}}

\providecommand{\rhosnr}{\rho_{\mathrm{snr}}}
\providecommand{\rhotr}{\rho_{\mathrm{tr}}}
\providecommand{\rhoMzero}{\rho_{M_0}}

\providecommand{\SNR}{\mathrm{SNR}}

\providecommand{\Mzero}{M_0}

\providecommand{\state}{H}                          
\providecommand{\res}{r}                            
\providecommand{\score}{R}                          
\providecommand{\rscore}{S}                         
\providecommand{\Fcal}{\mathcal{F}}
\providecommand{\bigO}{\mathcal{O}}

\providecommand{\Lscore}{L_{\rscore}}                
\providecommand{\Mg}{M_G}                            
\providecommand{\Lg}{L_G}                            
\providecommand{\Cj}{C_J}                            
\providecommand{\dzero}{D_0}                         
\providecommand{\Dpi}{D_\pi}                         
\providecommand{\dens}{g_\star}                      
\providecommand{\sstar}{s_\star}                     
\providecommand{\Gst}{G_\star}                       

\providecommand{\dG}{d_{\mathcal G}}                 
\providecommand{\Qth}[1]{\mathcal Q_{#1}}            
\providecommand{\dth}[1]{d_{#1}}                     
\providecommand{\kmod}[1]{\kappa(#1)}                
\providecommand{\Dctr}[1]{\Delta_{\mathrm{ctr}}(#1)} 
\providecommand{\Plawth}[2]{P_{#1,#2}}               
\providecommand{\rhostar}{\rho_\star}                
\providecommand{\rhoobs}{\widehat\rho_{\mathrm{obs}}}
\providecommand{\rhodG}{\widehat\rho_{\dG}}          
\providecommand{\rhoDL}{\widehat\rho_{D}}            
\providecommand{\rhoscore}{\widehat\rho_{\mathrm{score}}}  
\providecommand{\rhoind}{\widehat\rho_{\Gamma}}      
\providecommand{\rhoaug}{\rho_{\mathrm{aug}}}        
\providecommand{\tauint}{\tau_{\mathrm{int}}}        
\providecommand{\eNLL}{\varepsilon_{\rm NLL}}        
\providecommand{\meff}{m_{\mathrm{eff}}}             
\providecommand{\meffcov}{m_{\mathrm{eff}}^{\mathrm{cov}}} 
\providecommand{\Robs}{\mathcal R}                   
\providecommand{\Fcond}{\mathcal F_{\mathrm{cond}}}  
\providecommand{\Vfun}{\mathcal V}                   
\providecommand{\Vhat}{\widehat{\mathcal V}}         
\providecommand{\Lemit}{L_{\mathrm{emit}}}           
\providecommand{\cobs}{c_{\mathrm{obs}}}             
\providecommand{\Lin}{L_{\mathrm{in}}}               
\providecommand{\Cid}{C_{\mathrm{id}}}               
\providecommand{\Hobs}{H_{\mathrm{obs}}}             
\providecommand{\Dobs}[1]{D_{#1,\Hobs}}              
\providecommand{\wproj}{w_{\mathrm{proj}}}           
\providecommand{\DeepVAR}{\textsc{\methsize DeepVAR-FCP}}   
\providecommand{\DeepVARcov}{\textsc{\methsize DeepVAR-Cov}}
\providecommand{\GNFcov}{\textsc{\methsize gnf-Cov}}        
\providecommand{\LRGRUcov}{\textsc{\methsize lr-gru-Cov}}   
\providecommand{\DiagGRUcov}{\textsc{\methsize diag-gru-Cov}}

\providecommand{\best}[1]{\textbf{#1}}
\providecommand{\secondbest}[1]{\underline{#1}}

\providecommand{\widthstd}[2]{#1\,$\scriptstyle\pm #2$}

\definecolor{oursBlue}{HTML}{3568B8}
\colorlet{heroRow}{oursBlue!8}
\newcommand{\heroname}[1]{{\color{oursBlue!80!black}#1}}
\newcommand{\heromidrule}{%
  \specialrule{\lightrulewidth}{\aboverulesep}{0pt}%
  \arrayrulecolor{heroRow}\specialrule{\belowrulesep}{0pt}{0pt}\arrayrulecolor{black}}
\newcommand{\heroend}{%
  \arrayrulecolor{heroRow}\specialrule{0.35ex}{0pt}{0pt}\arrayrulecolor{black}}

\makeatletter
\DeclareRobustCommand{\methsize}{\ifdim\f@size pt>9.5pt\relax\small\fi}
\makeatother
\renewcommand{\meth}[1]{\textsc{\methsize #1}}
\providecommand{\GNF}{\textsc{\methsize gnf}}                    
\providecommand{\KalmanFCP}{\textsc{\methsize Kalman-FCP}}       
\providecommand{\StaticCGIF}{\textsc{\methsize Static-CGIF}}     
\providecommand{\DiagGRU}{\textsc{\methsize Diag-GRU}}           
\providecommand{\GCNrankzero}{\textsc{\methsize GCN-rank-0}}     
\renewcommand{\GraphNeuralSSM}{\GNF}
\renewcommand{\CGIFFiltered}{\KalmanFCP}

\renewcommand{\GraphLGSSM}{\textsc{\methsize Graph-LGSSM}}
\renewcommand{\ACIPerGroupFactorCGIF}{\textsc{\methsize ACI-PerGroup-FactorCGIF}}
\renewcommand{\AgACIGroupCGIF}{\textsc{\methsize AgACI-GroupCGIF}}

\renewcommand{\ACIFactorCGIF}{\textsc{\methsize ACI-FactorCGIF}}

\title{%
  Filtered Conformal Ellipsoids for Graph-Native Time Series
}

\author{%
  Yannick Limmer \\ DRW \\ United Kingdom \\ \url{ylimmer@drw.com}%
}

\begin{document}
\maketitle

\begin{abstract}
  Joint prediction sets for multivariate time series should control a
single event while adapting to cross-coordinate dependence. We study
\emph{filtered conformal ellipsoids}: a frozen state-space filter
emits a one-step predictive mean and covariance, and split-conformal
calibration is applied to the resulting Mahalanobis scores. The
filter is used to choose the ellipsoid \emph{shape}; conformal
calibration chooses the scalar \emph{radius}, so the construction
benefits from a learned predictive covariance without relying on
Gaussian tail probabilities for coverage. The main difficulty is
that filtered scores are dependent and learned recurrent filters
need not contract in their raw hidden state; we therefore analyse
contraction in \emph{each filter's own} observable predictive-law
quotient that identifies hidden states producing the same future
sequence of emitted Gaussian laws. Under a stable Bayes
Gaussian-projection filter, covariance bounds, and a
contraction-modulus identification condition, small excess Gaussian
negative log-likelihood is one sufficient route to contraction of the
learned filter on its own quotient. Combined with input sensitivity,
geometric data coupling, and a threshold-autocovariance envelope, a
dependent quantile-concentration bound yields an \emph{unconditional
approximate marginal coverage-gap bound} for filtered split-conformal
prediction under dependence (a Chebyshev form; a sharper Bernstein
form requires an additional geometric-mixing concentration
assumption). This is a marginal, not calibration-conditional,
guarantee. Under Gaussian oracle realisability we
also obtain a near-oracle log-volume comparison within the class of
conditionally valid Gaussian ellipsoid rules. We instantiate the
framework with a GCN-GRU filter with diagonal-plus-low-rank
covariance. On moderate-size graph-native traffic benchmarks
(\metrla-$20$ and \pemsbay-$50$), the learned filter gives
sharper at-target ellipsoids than static-covariance and non-filter
baselines; at full-graph scale and on non-graph-native datasets,
factor and copula baselines can be stronger.

\end{abstract}

\section{Introduction}
\label{sec:intro}

Deployed forecasters of correlated sensor streams output prediction
sets, not point forecasts. For $Y_t\in\R^N$ a multivariate sensor
vector at time~$t$, downstream users rely on the joint event
$\{Y_t\in\Cset_\alpha(x_t)\}$, not $N$ separate marginal statements.
Marginal conformal intervals therefore do not solve the deployed
problem: on correlated streams their Cartesian product can miss the
target joint event even when each coordinate is calibrated
\citep{xu2024multidim,stankeviciute2021conformal}, and our
diagnostics show marginal methods falling far below joint target
coverage (App.~\ref{app:leaderboard}). We instead conformalise a
scalar Mahalanobis score, producing an ellipsoid whose axes encode
cross-sensor dependence.

\begin{figure}[t]
  \centering
  \includegraphics[width=\linewidth]{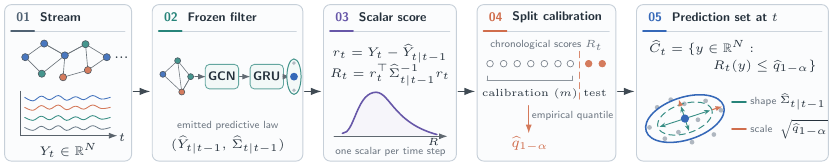}
  \caption{\textbf{Filtered conformal prediction.} A frozen filter
  emits a one-step mean and covariance; calibration scores are
  squared Mahalanobis residuals; the empirical $(1{-}\alpha)$-quantile
  defines a prediction ellipsoid. The filter chooses the ellipsoid
  shape; conformal calibration chooses the scalar radius. The theory
  analyses contraction of emitted predictive laws
  (Theorem~\ref{thm:obs-contract}), not raw recurrent states.}
  \label{fig:architecture}
\end{figure}

Our construction separates \emph{shape} from \emph{coverage}. A
frozen filter $\phi$ emits a one-step predictive Gaussian
\smash{$\mathcal N(\hat Y_{t|t-1},\hat\Sigma_{t|t-1})$} before $Y_t$ is
observed; the squared Mahalanobis score
\smash{$\score_t=(Y_t-\hat Y_{t|t-1})^{\top}\hat\Sigma_{t|t-1}^{-1}(Y_t-\hat Y_{t|t-1})$}
is calibrated on a chronological calibration block, and the test-time
set is \smash{$\{y:\score(y){\le}\hat q_{1-\alpha}\}$}.
(Algorithm~\ref{alg:fscp}, Figure~\ref{fig:architecture}). The
filter's \smash{$\hat\Sigma_{t|t-1}$} supplies only the ellipsoid's
orientation and per-axis scale; the conformal quantile of the
observed Mahalanobis residuals supplies the scalar radius. The
method therefore benefits from a learned predictive covariance while
retaining calibration through observed residuals rather than through
Gaussian tail probabilities. Throughout we use the display names
$\GNF$ (the learned graph filter), $\KalmanFCP$ (linear Kalman
baseline), $\StaticCGIF$ (static-covariance baseline), and
$\DiagGRU$ / $\GCNrankzero$ for the two diagonal-covariance ablations
(graph-removed-and-diagonal vs.\ graph-retained-but-rank-zero);
implementation details and the display-to-implementation name
mapping are in Apps.~\ref{app:filters}--\ref{app:baselines}. We
call a benchmark \emph{graph-native} when the data source provides a
meaningful sensor graph; the formal criterion appears in
App.~\ref{app:datasets}.

This setup creates two obstacles. The calibration scores come from one
chronological trajectory and are therefore dependent; moreover,
raw-state contraction is unnecessarily strong for recurrent filters
with output-invisible directions. We work instead in each filter's
parameter-specific observable quotient $\Qth{\theta}$, comparing only
scalar contraction moduli across filters, never individual points
across two different quotient spaces. Under a stable Bayes
Gaussian-projection oracle, uniform covariance spectral bounds, and a
contraction-modulus identification condition, small excess
Gaussian-NLL risk is \emph{one sufficient route} to contraction of the
learned filter on its own quotient (Theorem~\ref{thm:obs-contract});
input sensitivity and data coupling then yield dependent
quantile concentration and an \emph{unconditional approximate marginal
coverage-gap bound} for filtered split-conformal prediction
(Theorem~\ref{thm:learned-validity}). This guarantee is marginal, not
calibration-conditional, and the near-oracle log-volume result is
restricted to the class of conditionally valid Gaussian ellipsoid
rules (Theorem~\ref{thm:logvol}). Table~\ref{tab:ledger}
(§\ref{sec:theory}) records which assumptions each result uses and
what the experiments audit.

Empirically, $\GNF$ is the sharpest at-target method on \metrla-$20$
and \pemsbay-$50$, improving $\wproj$ by $23.6\%$ and $40.5\%$ over
$\StaticCGIF$ (Figure~\ref{fig:headline}, Table~\ref{tab:hero}); at
full-graph scale and on non-graph-native data, factor or copula
baselines can be stronger (Table~\ref{tab:scale},
App.~\ref{app:multi-backbone}). Same-mean and matched-budget controls
support the covariance-shape interpretation on the two headline cells
(§\ref{subsec:controlled}), while the audits probe finite-trajectory
consequences rather than population premises.

\paragraph{Contributions.}
\begin{enumerate}[leftmargin=*,topsep=1pt,itemsep=0pt]
\item \emph{Filtered conformal ellipsoids with learned predictive
  covariances.} A frozen filter supplies a time-varying ellipsoid
  shape and conformal calibration supplies the radius, so the
  construction does not rely on Gaussian tail validity.
\item \emph{Observable-law theory for dependent filtered scores.}
  Working in each filter's own parameter-specific observable quotient,
  small excess Gaussian-NLL is one sufficient route to contraction of
  the learned filter (Theorem~\ref{thm:obs-contract}); with input
  sensitivity and data coupling, a dependent quantile-concentration
  bound yields an unconditional approximate marginal coverage-gap
  bound (Theorem~\ref{thm:learned-validity}) and a near-oracle
  log-volume comparison within the conditionally valid
  Gaussian-ellipsoid class (Theorem~\ref{thm:logvol}).
\item \emph{Graph-native instantiation, controlled isolation, and
  empirical scope.} We instantiate the framework with a GCN-GRU
  diagonal-plus-low-rank covariance filter, show sharpness gains on
  moderate-$N$ graph-native traffic, isolate the covariance mechanism
  with a same-mean control and a matched-budget \DeepVAR\ control on
  the headline cells, and document where factor / copula baselines are
  stronger (full-graph and non-graph-native cells).
\end{enumerate}

\section{Related work}
\label{sec:related}

\paragraph{Conformal prediction under dependence.}
Split-conformal prediction \citep{vovk2005algorithmic,lei2018distribution}
requires exchangeability. Online adjustments such as \textsc{aci}
\citep{gibbs2021adaptive} and \textsc{AgAci} \citep{zaffran2022adaptive}
attain long-run marginal coverage under arbitrary drift; the
``beyond-exchangeability'' line \citep{barber2023conformal} refines
validity under explicit mixing.

\paragraph{Multivariate / joint conformal regions.}
\textsc{spci} \citep{xu2023sequential} and its multivariate extension
\textsc{MultiDimSPCI} \citep{xu2024multidim} produce ellipsoidal
regions via a learned-but-not-filtered covariance;
\textsc{CopulaCPTS} \citep{sun2022copula} builds joint regions via
empirical copulas. These are our published joint-CP comparators;
implementation details and additional baselines (factor, group,
spectral, EWMA, rolling, time-of-day, and a $k$-NN local ellipsoid)
are in App.~\ref{app:baselines}.

\paragraph{Topology-aware conformal prediction.}
Recent work injects graph topology into conformal prediction at
different stages: \textsc{staci} \citep{zhang2025topology} lets topology
and temporal dynamics enter the \emph{nonconformity score} for stream
networks, CF-GNN \citep{huang2023conformalized} through a graph-aware
\emph{output-correction} model, and \citet{dua2025graphcp,cini2025relational}
propagate calibration information over the graph. Our use of topology is
complementary and not claimed as first: it enters \emph{upstream},
through a learned time-varying predictive covariance, while
conformalisation remains a single scalar-radius step---the graph models
the law, not the calibration scores.

\paragraph{Learned predictive covariances and filter stability.}
Likelihood-trained Gaussian heads on recurrent backbones
\citep{salinas2020deepar,salinas2019high,rangapuram2018deep} predict
distributional parameters but are not calibrated finite-sample;
DeepVAR-style low-rank Gaussian RNNs \citep{salinas2019high} are the
non-graph analogue of our covariance mechanism, and we include a
matched-budget \DeepVAR\ control in §\ref{sec:experiments}. Graph
neural backbones \citep{kipf2017semi,li2018diffusion,wu2019graph,cho2014learning}
give point predictors that we re-use as mean backbones. Filter-stability and
observability tools are classical
\citep{anderson1979optimal,hairer2011yet,douc2018markov}; the
contribution of §\ref{sec:theory} is to transfer stability of a
learned filter from these abstract conditions to filtered
split-conformal validity and log-volume sharpness, in \emph{each
filter's own} observable predictive-law quotient rather than on raw
recurrent states.

\section{Preliminaries}
\label{sec:prelim}

\paragraph{Sequence model.}
We observe $\{Y_t\}_{t\ge 1}\subset\R^N$ together with covariates
$x_t$, which may include lagged endogenous observations, graph
features, and calendar/exogenous variables. A \emph{filter} is any
measurable map $\phi:(x_t,\state_{t-1})\mapsto(\hat Y_{t|t-1},\hat\Sigma_{t|t-1},\state_t)$;
we reserve $N$ for the observation dimension and $d$ for the latent
state. The residual $\res_t := Y_t-\hat Y_{t|t-1}$ gives the squared
Mahalanobis score
\begin{equation}\label{eq:score}
  \score_t \;=\; \res_t^{\top}\hat\Sigma_{t|t-1}^{-1}\res_t
\end{equation}
and its root form
$\rscore_t:=\score_t^{1/2}=\|\hat\Sigma_{t|t-1}^{-1/2}\res_t\|_2$. The
two scores induce the \emph{same} prediction region by monotonicity: if
$\hat q_{\rscore}$ is the empirical $(1{-}\alpha)$-quantile of
$\rscore_t$, then $\hat q_{1-\alpha}=\hat q_{\rscore}^{2}$ is the
corresponding quantile of $\score_t$ and
$\score_t\le\hat q_{1-\alpha}\Leftrightarrow\rscore_t\le\hat q_{\rscore}$.
We \emph{calibrate} $\score_t$ in Algorithm~\ref{alg:fscp} (smooth in
$\state_t$, no $1/\sqrt{\score}$ blow-up) and \emph{analyse} $\rscore_t$
in §\ref{sec:theory} (avoids the global-Lipschitz issue at unbounded
residuals; App.~\ref{app:Lscore}). At a calibration trajectory
$\{Y_t\}_{t=1}^{n}$, set $\hat
q_{1-\alpha}:=\score_{(\lceil(n+1)(1-\alpha)\rceil)}$ and emit
\begin{equation}\label{eq:setdef}
  \Cset_\alpha(x_{n+1}) \;=\; \bigl\{y\in\R^N :
    (y-\hat Y_{n+1\mid n})^{\top}\hat\Sigma_{n+1\mid n}^{-1}(y-\hat Y_{n+1\mid n})
    \le \hat q_{1-\alpha}\bigr\}.
\end{equation}
Joint coverage $\Prob\{Y_{n+1}\in\Cset_\alpha(x_{n+1})\}$ is the
deployment-relevant metric.

\paragraph{Set-size functional.}
The primary size functional we report is the \emph{mean
coordinate-projection length} $\wproj$: a common,
set-family-agnostic functional applied to each method's own prediction
set. Writing
$\ell_{t,i}:=\inf_{y\in\Cset_t}y_i$ and $u_{t,i}:=\sup_{y\in\Cset_t}y_i$
for the exact marginal extent of $\Cset_t$ in coordinate $i$,
\begin{equation}\label{eq:wproj}
  \wproj
  \;:=\;
  \frac{1}{T_{\mathrm{test}}\,N}
  \sum_{t=1}^{T_{\mathrm{test}}}\sum_{i=1}^{N}
  \bigl(u_{t,i}-\ell_{t,i}\bigr).
\end{equation}
For a single ellipsoid
$\Cset_t=\{y:(y-\hat\mu_t)^\top\hat\Sigma_t^{-1}(y-\hat\mu_t)\le\hat q_t\}$
the coordinate-projection length is available in closed form via the
support function, $u_{t,i}-\ell_{t,i}=2\sqrt{\hat q_t\,\hat\Sigma_{t,ii}}$
(the factor $2$ and the square root are both present). For
\textsc{CopulaCPTS} we use the coordinate endpoints of its native
jointly calibrated set; for factor, group, spectral, and Bonferroni
constructions we project the \emph{full composed set}, not one
constituent ellipsoid; joint coverage is always evaluated on the
original joint set $\Cset_t$, never its axis-aligned bounding box; and
all methods are compared in the same standardised coordinate system.
The exact projection computation used for each non-ellipsoidal or
composed set is given in App.~\ref{app:proj-rules}
(Table~\ref{tab:proj-rules}). Note that $\wproj$
is the average side length of the set's axis-aligned bounding box: it
is directly interpretable as an average per-sensor range, but it is
basis-dependent and does not characterise the complete geometry or
volume of a joint set.

As a complementary geometry-\emph{aware} size functional for genuine
single-ellipsoid rules we report the normalised log-volume
\begin{equation}\label{eq:logvol}
  \tfrac1N\log\mathrm{Vol}(\Cset_t)
  \;=\;
  \tfrac1N\log v_N
  \,+\,
  \log\hat s
  \,+\,
  \tfrac{1}{2N}\log\det\hat\Sigma_{t|t-1},
\end{equation}
where $v_N$ is the volume of the Euclidean unit ball in $\R^N$ and $\hat s=\hat q_{1-\alpha}^{1/2}$;
this is the object Theorem~\ref{thm:logvol} controls, and we report it
only for methods whose emitted set is a single full-dimensional
ellipsoid (so that $\log\det\hat\Sigma$ is the exact volume). Headline
cells use both: Table~\ref{tab:hero} reports $\wproj$ and
App.~\ref{app:logvol} (Table~\ref{tab:logvol}) reports the corresponding
single-ellipsoid log-volumes; the rankings agree (§\ref{subsec:hero}).

\paragraph{Why Mahalanobis.}
The Mahalanobis score produces an ellipsoidal region oriented by
$\hat\Sigma$ and invariant under affine reparameterisation of~$Y$;
alternative single-coordinate-norm scores ($\|r\|_\infty$, $\|r\|_2$,
quantile of $|r_i|/\sigma_i$) produce boxes or interval products that
are looser at matched coverage on correlated $Y$, since the
minimum-volume $(1-\alpha)$-content set of a Gaussian is a Mahalanobis
ellipsoid \citep{anderson1979optimal}.

\begin{theorem}[Split-conformal validity; \citealp{vovk2005algorithmic}]
\label{thm:A}
If $(\score_t)_{t=1}^{n+1}$ are exchangeable, the region
\eqref{eq:setdef} satisfies
$\Prob\{Y_{n+1}\in\Cset_\alpha(x_{n+1})\}\ge 1-\alpha$; if the scores
are a.s.\ distinct, coverage is at most $1-\alpha+1/(n+1)$.
\end{theorem}

\paragraph{Predictive-law metric and observable quotient.}
Write $P_{\theta,t}^h=\mathcal N(\mu_{\theta,t}^h,\Sigma_{\theta,t}^h)$
for the one-step predictive Gaussian law emitted by a frozen filter
parameterised by $\theta$ and initialised at quotient state $h$. We
endow Gaussian laws with the Bures-Wasserstein distance
\begin{equation}\label{eq:dG}
  \dG^{2}\bigl(\mathcal N(\mu,\Sigma),\mathcal N(\nu,\Lambda)\bigr)
  \;=\;
  \|\mu-\nu\|_{2}^{2}
  +
  \tr\!\bigl[\Sigma+\Lambda-2(\Lambda^{1/2}\Sigma\Lambda^{1/2})^{1/2}\bigr];
\end{equation}
on covariances spectrum-bounded between $\lambda_-$ and $\lambda_+$,
$\dG$ is bi-Lipschitz to KL and to Frobenius covariance distance
(App.~\ref{app:bures}). For \emph{each fixed filter parameter}
$\theta$ we define a \emph{parameter-specific} observable quotient
$\Qth{\theta}:=\state/\!\sim_\theta$, where $h\sim_\theta h'$ iff $h$
and $h'$ emit the same future predictive-law sequence under every
admissible common input continuation, with quotient metric
$\dth{\theta}$ (full construction in App.~\ref{app:obs-quotient}). All
contraction statements in §\ref{sec:theory} are made inside a single
filter's own quotient, not in raw recurrent-state norm and never by
identifying elements of $\Qth{\theta_\star}$ with those of
$\Qth{\widehat\theta}$: the analysis compares only the scalar
contraction moduli of the two filters. For the linear-Gaussian instance we use
$F$ for the transition, $C$ for the observation matrix, $K$ for the
steady-state Kalman gain, and $\Acl:=F(I-KC)$ for the closed-loop
transition (App.~\ref{app:graph-lgssm}); $A_{\mathrm{graph}}$ is
reserved for the adjacency. We treat both linear-Kalman (\GraphLGSSM)
and learned GCN--GRU (\GraphNeuralSSM) filters as instances.

\section{Filtered conformal prediction}
\label{sec:method}

Algorithm~\ref{alg:fscp} defines the protocol: fit a filter $\phi$ on
a held-out pretraining trajectory, freeze its parameters, simulate it
forward on the calibration trajectory to collect squared-Mahalanobis
scores, then emit \eqref{eq:setdef} at test time. The theory gives
\emph{unconditional approximate marginal validity under the stated
dependence conditions} when observable contraction
(Theorem~\ref{thm:obs-contract}) is combined with the
score-dependence assumptions of §\ref{sec:theory}
(Theorem~\ref{thm:learned-validity}); exact split-conformal validity
is recovered only in the oracle standardised-innovation case
(Proposition~\ref{prop:oracle}). Sharpness is governed by the
log-volume bound of Theorem~\ref{thm:logvol}.

\begin{algorithm}[t]
\caption{Filtered Split-Conformal Prediction}
\label{alg:fscp}
\begin{algorithmic}[1]
\Require sequence $\{Y_t\}_{t=1}^{n+T}$, filter $\phi$ (frozen), level $\alpha$
\State \emph{(calibration)} simulate $\phi$ on $Y_{1:n}$, compute
       $\score_t = \res_t^{\top}\hat\Sigma_{t|t-1}^{-1}\res_t$ for
       $t=t_0,\dots,n$ (discard warm-up $t_0$).
\State $\hat q_{1-\alpha}\gets \score_{(\lceil(n-t_0+2)(1-\alpha)\rceil)}$.
\State \emph{(test)} advance $\phi$ and emit $\Cset_\alpha(x_t)$ via \eqref{eq:setdef} for $t=n+1,\dots,n+T$.
\end{algorithmic}
\end{algorithm}

\paragraph{\GraphLGSSM\ (linear Kalman baseline).}
A linear-Gaussian SSM with graph transition
$F=\rho_{\mathrm{scale}}\,\widetilde A_{\mathrm{sym}}/\lambda_{\max}(\widetilde A_{\mathrm{sym}})$
(self-loop-augmented normalised adjacency), isotropic $Q=\sigma_Q^2 I$, $R=\sigma_R^2 I$,
$C=I_N$, $\rho_{\mathrm{scale}}=0.8$. This baseline serves both as a
sanity check and as the linear specialisation of the
observable-contraction theorem; the analytic closed-loop rate
$\sigma_1(\Acl)$, the initialisation-weighted RMS rate
$\rhoMzero(\Acl,T)$ from Theorem~\ref{thm:weighted-rms}
(App.~\ref{app:proofs-L}), and the SNR proxy $\rhosnr$ are reported
per dataset (App.~\ref{app:graph-lgssm}).

\paragraph{\GraphNeuralSSM\ (learned GCN-GRU; our ``learned graph filter'').}
A graph-convolutional GRU cell \citep{li2018diffusion,cho2014learning}
with per-node hidden states and node-shared weights: at time~$t$ the
filter consumes the observed $Y_{t-1}$, advances
$H_t=\mathrm{GCNGRU}(Y_{t-1},H_{t-1};\widetilde A_{\mathrm{sym}})$,
and emits $(\hat Y_{t|t-1},\hat\Sigma_{t|t-1})=g_\theta(H_t)$ with
$\hat\Sigma_{t|t-1}=\diag(d_t)+L_tL_t^{\top}$, $L_t\in\R^{N\times r}$,
\emph{before} $Y_t$ is observed (no calibration or test score uses
$Y_t$ to predict itself; App.~\ref{app:tf}). The diagonal head
$d_t=\exp(\cdot)$ clamped at a floor $\varepsilon_d=10^{-4}$ keeps
$\hat\Sigma_{t|t-1}$ positive definite; Woodbury and the matrix-determinant lemma give
$\bigO(Nr^{2}+r^{3})$ inversion / log-determinant cost (linear in $N$
for fixed $r=4$). The rank-$r$ factor carries the cross-sensor
correlation. Rank $r=4$ was fixed globally from the synthetic pilot
and used unchanged across all real datasets; the \metrla\ rank sweep
(App.~\ref{app:rank}) is a retrospective sensitivity analysis and was
not used for per-dataset model selection. Training: maximum likelihood
under the induced Gaussian predictive, Adam, teacher-forcing,
truncated BPTT (window 64); details in App.~\ref{app:filters}.
A sweep over $\varepsilon_d\in\{10^{-6},\dots,10^{-2}\}$ shows the
default floor is nonbinding on the audited full-covariance trajectories
(no audited step has $\lambda_{\min}(\hat\Sigma_t)\le 2\varepsilon_d$)
and does not explain the residual high-condition-number tail
(App.~\ref{app:floor-sweep}).

Gaussian NLL plays two roles. Empirically it trains the mean and
covariance head; theoretically, under the finite-horizon Fisher
condition of Theorem~\ref{thm:obs-contract}, excess NLL controls
perturbations of the observable transition map.

\paragraph{Ablations.}
We use two diagonal-covariance ablations. \emph{$\DiagGRU$
(\textsc{NeuralSSM}: graph removed and diagonal $\Sigma_t$)} replaces
the GCN by an identity map \emph{and} sets $r=0$; this is the row
reported in Table~\ref{tab:hero} ($\wproj=4.71$ at joint $0.932$ on
\metrla-$20$). \emph{$\GCNrankzero$ (graph retained, rank $0$)}
keeps the GCN and only sets $r=0$; this is the rank-sweep point
($\wproj=5.32$ on \metrla-$20$, App.~\ref{app:rank}) used to isolate
the covariance head at fixed graph mixing. The gap from
$\GCNrankzero$ to $r{=}4$ shows that the diagonal-plus-low-rank head
accounts for most of the sharpness gain, while the $\DiagGRU$
comparison jointly captures graph-mixing and covariance-rank effects.

\paragraph{Naming.}
We use the display names defined in §\ref{sec:intro} consistently in
prose, figures, and results tables; Table~\ref{tab:method-map} in
App.~\ref{app:baselines} maps display names to their implementation
identifiers, which we retain only in code-path and reproducibility
discussions.

\section{Theory}
\label{sec:theory}

All contraction statements are made inside each filter's own
parameter-specific observable quotient (§\ref{sec:prelim}), and
cross-parameter comparisons are made only between scalar contraction
moduli. The guarantees form a conditional chain, not a
distribution-free theorem for arbitrary learned filters;
Table~\ref{tab:ledger} lists the assumptions, the diagnostics that
probe each conclusion, and what they do \emph{not} certify.

\paragraph{Theory map.}
Excess Gaussian NLL, a stable oracle, and contraction-modulus
identification control the learned filter's contraction rate on its own
quotient (Theorem~\ref{thm:obs-contract}). Input sensitivity and
geometric data coupling then force forgetting of the augmented
filter--data process, giving score-CDF forgetting and a summable
threshold-indicator covariance envelope
(Corollary~\ref{cor:thresh-mix}); dependent empirical-CDF concentration
turns this into a high-probability calibration-quantile error and hence
an unconditional marginal coverage-gap bound
(Theorem~\ref{thm:learned-validity}). Under Gaussian oracle
realisability the same quantile event controls log-volume against
conditionally valid Gaussian ellipsoid rules
(Theorem~\ref{thm:logvol}).

\begin{table}[t]
\caption{\textbf{Assumption summary.} Each main result is conditional
on the listed assumptions; the audits estimate operative conclusions
or finite-trajectory proxies directly but do not certify the
population assumptions.}
\label{tab:ledger}
\centering\footnotesize
\setlength{\tabcolsep}{4pt}
{
  \renewcommand{\arraystretch}{1.1}
\begin{tabular}{p{1.1cm} p{3.9cm} p{3.9cm} p{3.9cm}}
\toprule
Result & Main additional assumptions & Audit checks (operative / proxy) & Not certified (population) \\
\midrule
Theorem~\ref{thm:obs-contract} & stable Bayes Gaussian-projection oracle; covariance bounds; contraction-modulus identification & $\rhodG$ emitted-law contraction (App.~\ref{app:dG-audit}); finite-horizon output-sensitivity proxy (App.~\ref{app:obs-audit}) & population $\eNLL$; realisability; population observability inequality \\
\addlinespace[1pt]
Cor.~\ref{cor:thresh-mix} & input sensitivity + geometric data coupling (Ass.~\ref{ass:input}) & indicator-autocov decay $\rhoind$ (Audit~3) & distribution-free exchangeability \\
\addlinespace[1pt]
Thm.~\ref{thm:learned-validity} (Cheb.) & threshold-autocovariance envelope (Cor.~\ref{cor:thresh-mix}) & $\rhoind$, $\tauint$ (App.~\ref{app:tau-int}) & calibration-conditional coverage \\
\addlinespace[1pt]
Thm.~\ref{thm:learned-validity} (Bern.) & geometric-mixing Bernstein concentration (Ass.~\ref{ass:bernstein}) & integrated autocorrelation $\tauint$ (App.~\ref{app:tau-int}) & the concentration inequality itself \\
\addlinespace[1pt]
Theorem~\ref{thm:logvol} & Gaussian oracle realisability; conditionally valid Gaussian comparator & log-volume agrees with $\wproj$ (App.~\ref{app:logvol}) & optimality vs.\ arbitrary marginally valid sets \\
\bottomrule
\end{tabular}
}
\end{table}

\subsection{NLL excess as one sufficient route to observable contraction}
\label{subsec:obs-contract}

Theorem~\ref{thm:obs-contract} is a transfer statement, not a proof
that an arbitrary trained GRU is stable: \emph{if} the oracle Bayes
Gaussian-projection filter contracts on its quotient and NLL risk
identifies the learned filter's contraction modulus, \emph{then} the
learned emitted predictive laws inherit contraction. The ``no-ghost''
part of Assumption~\ref{ass:obs} rules out observable directions that
move future laws yet are invisible to Gaussian NLL; directions that
never affect emitted laws are allowed---exactly what the quotient
$\Qth{\widehat\theta}$ removes. Define the \emph{scalar quotient
contraction modulus}
$\kmod{\theta}:=\sup_{y}\sup_{q\ne q'}
\dth{\theta}(T_\theta(y,q),T_\theta(y,q'))/\dth{\theta}(q,q')$
($T_\theta$, $O_\theta$: transition and emission maps induced on
$\Qth{\theta}$) and the contraction-modulus gap
$\Dctr{\theta}:=[\kmod{\theta}-\kmod{\theta_\star}]_+$.

\begin{assumption}[Observable contraction setting]\label{ass:obs}
\emph{(O1)~Stable oracle.} An oracle $\theta_\star$ in the
learned-filter family emits the Bayes Gaussian projection of the true
one-step law and contracts on its own quotient,
$\kmod{\theta_\star}\le\rhostar<1$.
\emph{(O2)~Uniform covariance spectrum.} $\lambda_{-}I\preceq\hat\Sigma_{\theta,t}\preceq\lambda_{+}I$.
\emph{(O3)~Emitted-law Lipschitz + finite-horizon observability.}
$\dG(O_\theta(q),O_\theta(q'))\le\Lemit\,\dth{\theta}(q,q')$ and
$\cobs\,\dth{\theta}(q,q')\le\Dobs{\theta}(q,q')$
(App.~\ref{app:obs-quotient}), where $\Dobs{\theta}$ is the
finite-horizon observable discrepancy under a fixed $\theta$-independent
probing law with lag weights $\beta_\ell$.
\emph{(O4)~Contraction-modulus identification / no-ghost.}
$\Dctr{\theta}^{2}\le\Cid\,[\Robs(\theta)-\Robs(\theta_\star)]$
(App.~\ref{app:fisher}).
\end{assumption}

\begin{theorem}[Observable contraction from excess Gaussian NLL]\label{thm:obs-contract}
Under Assumption~\ref{ass:obs}, let
$\eNLL:=\Robs(\widehat\theta)-\Robs(\theta_\star)$ and suppose
$\rhoobs:=\rhostar+\sqrt{\Cid\,\eNLL}<1$. Then the learned filter
$\widehat\theta$ contracts on its own quotient: on every common input
path,
\begin{equation}\label{eq:obs-contract-quotient}
  \dth{\widehat\theta}\bigl(Q^{q}_{\widehat\theta,t},Q^{q'}_{\widehat\theta,t}\bigr)
  \;\le\;\rhoobs^{\,t}\,\dth{\widehat\theta}(q,q').
\end{equation}
Consequently, translating to emitted laws,
\begin{equation}\label{eq:obs-contract}
  \dG\!\bigl(\Plawth{\widehat\theta}{t}^{q},\Plawth{\widehat\theta}{t}^{q'}\bigr)
  \;\le\;
  \frac{\Lemit}{\cobs}\,\rhoobs^{\,t}\,
    \Dobs{\widehat\theta}(q,q').
\end{equation}
\end{theorem}

When $\Hobs=0$ and $\beta_0=1$,
$\Dobs{\widehat\theta}(q,q')=\dG(\Plawth{\widehat\theta}{0}^{q},\Plawth{\widehat\theta}{0}^{q'})$,
recovering the initial-emitted-law form as the
instantaneous-observability special case. Since $\eNLL$ is not
observable (the oracle filter is unknown), we audit the theorem's
\emph{conclusion}---emitted-law contraction
(App.~\ref{app:dG-audit})---not its premises; a sufficient condition
for (O4) via a risk margin and local Lipschitzness of $\kmod{\cdot}$
on the GCN-GRU family is in App.~\ref{app:fisher-gcn}.

\subsection{From frozen-filter contraction to score dependence}
\label{subsec:score-dep}

The dependence and quantile results below start from
\emph{frozen-filter contraction}, input sensitivity, geometric data
coupling, and local score-law regularity; Theorem~\ref{thm:obs-contract}
is one sufficient mechanism for the contraction premise when the
filter is likelihood-trained, but the chain from
Corollary~\ref{cor:thresh-mix} onward does not require Gaussian oracle
realisability or a verified Fisher margin. Same-input filter
contraction does not by itself say what happens when two coupled data
paths differ; we therefore make the input-sensitivity coupling
explicit.

\begin{assumption}[Input sensitivity and geometric data coupling]\label{ass:input}
The frozen learned filter obeys, for a constant $\Lin$,
$\dth{\widehat\theta}(T_{\widehat\theta}(y,q),T_{\widehat\theta}(y',q'))
\le\rhoobs\,\dth{\widehat\theta}(q,q')+\Lin\,d_Y(y,y')$,
and the data/input process couples geometrically,
$\E\,d_Y(Y_t,Y_t')\le C_Y\rho_{\mathrm{data}}^{\,t}\,d_Y(Y_0,Y_0')$
with $\rho_{\mathrm{data}}<1$.
\end{assumption}

\begin{corollary}[Threshold mixing]\label{cor:thresh-mix}
Under Assumption~\ref{ass:input} and a local Lipschitz condition on
the conditional law of the root score given the observable past
(App.~\ref{app:thresh-mix}), the augmented filter--data process
couples geometrically at a rate $\rhoaug<1$ that can be taken
arbitrarily close to $\max\{\rho_{\mathrm{data}},\rhoobs\}$, and the
threshold indicators $I_i(u):=\mathbf 1\{\rscore_{\widehat\theta,i}\le u\}$
obey, for $u$ in a neighbourhood of the conformal threshold,
\begin{equation}\label{eq:thresh-mix}
  \sup_i|\Cov(I_i(u),I_{i+k}(u))|\le C_\Gamma\,\rhoaug^{\,k}.
\end{equation}
The envelope is summable, and the same coupling yields a warm-start
conditional-CDF bias
$\sup_{u\in I}|\Prob(\rscore_{\widehat\theta,t}\le u)-\Gst(u)|\le
C_{\mathrm{cdf}}\,\rhoaug^{\,t}$ (App.~\ref{app:thresh-mix}).
\end{corollary}

We carry $\rhoaug$ as a separate rate throughout: when
$\rho_{\mathrm{data}}\ne\rhoobs$ it is their maximum (adjusted
constant), and when they coincide the $t\,\rho^{t}$ convolution is
bounded at any $\rhoaug\in(\rho,1)$ (App.~\ref{app:thresh-mix}).

\subsection{Dependent quantile concentration and marginal validity}
\label{subsec:learned-validity}

\begin{assumption}[Geometric-mixing concentration]\label{ass:bernstein}
For $u$ in a neighbourhood of $\sstar$, the centred threshold
indicators satisfy
$\Prob(|\hat F_m(u)-\E\hat F_m(u)|>\varepsilon)\le 2\exp(-c\,\meff\,\varepsilon^{2})$
for $0\le\varepsilon\le\varepsilon_{\max}$, with
$\meff\asymp m(1-\rhoaug)/(1+\rhoaug)$ uniformly in $u$.
\end{assumption}

\begin{theorem}[Dependent quantile concentration and unconditional marginal validity]
\label{thm:learned-validity}
Let $m$ be the post-warmup calibration size, $\hat s_m$ the empirical
$(1-\alpha)$-quantile of the learned root scores, and $g_\star$ the
stationary score density (bounded above by $\bar g$ and below near
$\sstar$).

\emph{Part I (high-probability quantile estimation).} For the
appropriate radius $r_m(\delta)$,
$\Prob\{|\hat s_m-s_\star|\le r_m(\delta)\}\ge 1-\delta$, with
\begin{equation}\label{eq:qconc-cheb}
  r_m^{\mathrm{Cheb}}(\delta)
  \;\lesssim\;
  \rhoaug^{\,t_0}+(\delta\,\meffcov)^{-1/2}+m^{-1}
  \quad\text{under Corollary~\ref{cor:thresh-mix},}
\end{equation}
and, under the stronger Assumption~\ref{ass:bernstein},
\begin{equation}\label{eq:qconc-bern}
  r_m^{\mathrm{Bern}}(\delta)
  \;\lesssim\;
  \rhoaug^{\,t_0}+\sqrt{\tfrac{\log(1/\delta)}{\meff}}+\tfrac{\log(1/\delta)}{\meff}+m^{-1},
  \qquad
  \meffcov\asymp\tfrac{m}{1+2\sum_{k\ge 1}\Gamma_k}.
\end{equation}

\emph{Part II (unconditional marginal coverage).} Let
$b_t:=\sup_{u\in I}|\Prob\{S_t\le u\}-G_\star(u)|$ be the test-score
CDF discrepancy. Then the (deterministic) unconditional coverage
probability satisfies
\begin{equation}\label{eq:uncond-cov}
  \bigl|\Prob\{Y_t\in\Cset_t\}-(1-\alpha)\bigr|
  \;\le\;
  b_t+\bar g\,r_m(\delta)+\delta.
\end{equation}
This is an unconditional marginal-coverage statement for the
contiguous calibration/test trajectory. It is \emph{not} a
calibration-conditional coverage guarantee.
\end{theorem}

\begin{proposition}[Oracle exact validity]\label{prop:oracle}
If the frozen filter admits standardised innovations
$Y_t = \hat Y_{t|t-1} + \hat\Sigma_{t|t-1}^{1/2} Z_t$ with i.i.d.\
spherically symmetric $(Z_t)$ independent of the past, then
$\score_t=Z_t^{\top}Z_t$ are i.i.d.\ and Algorithm~\ref{alg:fscp} is
exactly split-conformally valid via Theorem~\ref{thm:A}.
\end{proposition}

\paragraph{Plug-in baselines.}
$\StaticCGIF$, \FactorCGIF, \GroupCGIF\ re-use the calibration block
to fit $\hat\Sigma$ and are \emph{plug-in} rather than exact
split-conformal: the coverage gap incurs a leave-one-out stability
term controlled by Theorem~\ref{thm:plugin}
(App.~\ref{app:proofs-A9-plugin}), with $\varepsilon_m=O(m^{-1})$ for
ridge empirical covariance.

\subsection{Log-volume sharpness}
\label{subsec:logvol}

\begin{theorem}[Near-oracle log-volume]\label{thm:logvol}
Assume Gaussian oracle realisability ($P_t^\star=\mathcal N(\mu_t^\star,\Sigma_t^\star)$
in the filter family at $\theta_\star$); the comparator class
$\Fcond$ consists of Gaussian ellipsoid rules whose population
conformal ellipsoid has conditional coverage $\ge 1-\alpha$. Then, on
the quantile-concentration event of Theorem~\ref{thm:learned-validity}
(Part~I),
\begin{equation}\label{eq:logvol-bound}
  \Vhat_m(\hat\theta)
  \le\inf_{\theta\in\Fcond}\Vfun(\theta)+C_1\sqrt{\eNLL}+C_2\,r_m(\delta).
\end{equation}
The comparison is against \emph{conditionally valid Gaussian
ellipsoid rules} only; arbitrary marginally valid conformal sets can
trade coverage across regimes and are not bounded by~\eqref{eq:logvol-bound}.
\end{theorem}

The constants in
\eqref{eq:qconc-cheb}--\eqref{eq:logvol-bound} are not tight
finite-sample certificates; they identify the dependence quantities to
control, which the audits check directly. The linear-Kalman case
recovers the closed-loop contraction rate exactly as a sanity check
(App.~\ref{app:proofs-L}, Theorem~\ref{thm:weighted-rms}).

\section{Experiments}
\label{sec:experiments}

We evaluate three questions. \textbf{Q1.} Does a learned time-varying
covariance shape improve at-target joint ellipsoid sharpness on
graph-native traffic? \textbf{Q2.} Are the contraction and
score-dependence quantities used by the theory small in the deployed
filters? \textbf{Q3.} Where does the learned filter stop being the
best choice?

\paragraph{Protocol.}
Two graph-native traffic benchmarks at moderate~$N$ ($\metrla$
busiest-$20$, $\pemsbay$ busiest-$50$) for the headline comparison;
six additional correlated-sensor benchmarks plus the full-graph
$\metrla$ ($N{=}207$) and $\pemsbay$ ($N{=}325$) for scope and
contraction validation, $10$ real-data evaluation cells in total.
Every cell uses the same chronological 40/40/20
train/calib/test split, $\alpha{=}0.1$, and 10 seeds
(App.~\ref{app:experiments}). \emph{At-target} means joint coverage
$\ge 0.895$ ($0.005$ slack for 10-seed Monte Carlo variability at
$T_{\mathrm{test}}\!\approx\!4000$). Display names used throughout
(prose, figures, tables): $\GNF$ (learned graph filter), $\KalmanFCP$
(linear Kalman), $\StaticCGIF$ (static-covariance baseline),
$\DiagGRU$ (graph-and-rank-removed), $\GCNrankzero$
(graph-retained-but-rank-zero).

\subsection{Moderate-$N$ graph-native traffic: sharpness gains (Q1)}
\label{subsec:hero}

\begin{figure}[t]
  \centering
  \includegraphics[width=\linewidth]{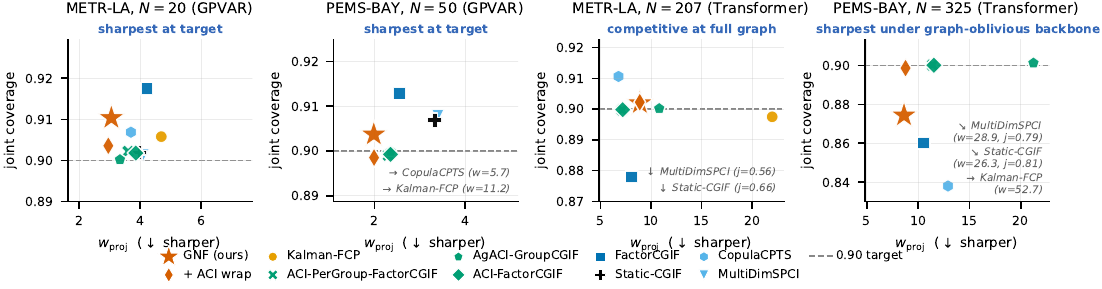}
  \caption{\textbf{Mean coordinate-projection length $\wproj$ vs.\
  joint coverage on graph-native traffic.}
  Each panel plots $\wproj$ (lower is sharper) against joint
  coverage at $\alpha{=}0.1$; the dashed line marks the $0.90$
  target. The first two panels correspond to Table~\ref{tab:hero}
  (moderate-$N$, \textsc{gpvar} backbone) and the last two panels
  are the full-graph cells under a graph-oblivious Transformer
  backbone (Table~\ref{tab:multibackbone-full}); plotted points are
  10-seed means (uncertainty in Tables~\ref{tab:hero}
  and~\ref{tab:multibackbone-full}). Table~\ref{tab:hero} reports the exact deltas:
  unwrapped-$\GNF$ is $-8.4\%$ vs.\ the strongest at-target
  non-filter on \metrla-$20$ and $-13.9\%$ on \pemsbay-$50$.
  Table~\ref{tab:scale} shows the complementary graph-aware
  \textsc{gpvar} regime where factor baselines lead.}
  \label{fig:headline}
\end{figure}

\begin{table}[t]
  \caption{\textbf{Moderate-$N$ graph-native regime: $\GNF$ is the
  sharpest at-target method evaluated.} Mean coordinate-projection
  length $\wproj$ and joint coverage at $\alpha{=}0.1$, both 10-seed
  mean$\pm$std; relative $\wproj$ vs.\ $\StaticCGIF$. \best{Bold}
  indicates the best at-target method ($\ge 0.895$); ``$\uparrow$'' =
  over-cover. Filter rows use their native one-step mean; non-filter
  rows use the shared \textsc{gpvar} mean. $\DiagGRU$ removes both
  graph mixing and the low-rank covariance head; $\GCNrankzero$ in
  App.~\ref{app:rank} isolates the covariance rank at fixed graph
  mixing. Per-seed values are in App.~\ref{app:seeds}
  (Table~\ref{tab:seeds}). Conditional-$\Sigma_t$, group, spectral,
  $\KalmanFCP$, \textsc{spci}, and \textsc{HopCPT} rows are in
  Table~\ref{tab:hero-full}.}
  \label{tab:hero}
  \centering\footnotesize
  \setlength{\tabcolsep}{3.4pt}
  \renewcommand{\arraystretch}{0.95}
  \begin{tabular}{l rc rc rr}
    \toprule
    & \multicolumn{2}{c}{\metrla\ ($N{=}20$)} & \multicolumn{2}{c}{\pemsbay\ ($N{=}50$)} & \multicolumn{2}{c}{$\Delta\wproj$ vs $\StaticCGIF$}\\
    \cmidrule(lr){2-3}\cmidrule(lr){4-5}\cmidrule(lr){6-7}
    Method (role) & $\wproj$ & joint & $\wproj$ & joint & \metrla & \pemsbay \\
    \heromidrule
    \rowcolor{heroRow}
    \heroname{$\GNF$ \textbf{(ours)}}
      & \widthstd{3.05}{0.13} & \widthstd{0.910}{0.004}
      & \best{\widthstd{1.98}{0.06}} & \widthstd{0.904}{0.005}
      & $-23.6\%$ & \best{$-40.5\%$} \\
    \rowcolor{heroRow}
    \heroname{\quad + $\ACI$ wrap}
      & \best{\widthstd{2.95}{0.17}} & \widthstd{0.904}{0.004}
      & \widthstd{1.99}{0.07} & \widthstd{0.899}{0.004}
      & \best{$-26.1\%$} & $-40.2\%$ \\
    \heroend
    \addlinespace[1pt]
    Best non-filter at-target\textsuperscript{\dag}
      & \widthstd{3.33}{0.23} & \widthstd{0.900}{0.003}
      & \widthstd{2.30}{0.09} & \widthstd{0.899}{0.002}
      & $-16.5\%$ & $-30.9\%$ \\
    $\StaticCGIF$ (reference)
      & \widthstd{3.99}{0.17} & \widthstd{0.902}{0.004}
      & \widthstd{3.33}{0.11} & \widthstd{0.907}{0.004}
      & ---       & --- \\
    $\DiagGRU$ (no graph, diag $\Sigma_t$)
      & \widthstd{4.71}{1.16} & \widthstd{0.932}{0.010}$\uparrow$ \!\!\!\!\!\!
      & \widthstd{2.88}{0.10} & \widthstd{0.921}{0.008}$\uparrow$ \!\!\!\!\!\!
      & $+18.0\%$ & $-13.5\%$ \\
    \textsc{CopulaCPTS} \citep{sun2022copula}
      & \widthstd{3.70}{0.24} & \widthstd{0.907}{0.004}
      & \widthstd{5.69}{0.74} & \widthstd{0.908}{0.009}
      & $-7.3\%$  & $+70.9\%$ \\
    \textsc{MultiDimSPCI} \citep{xu2024multidim}
      & \widthstd{4.10}{0.17} & \widthstd{0.901}{0.004}
      & \widthstd{3.40}{0.11} & \widthstd{0.908}{0.004}
      & $+2.8\%$  & $+2.1\%$ \\
    \bottomrule
  \end{tabular}
  \\[2pt]
  {\footnotesize\textsuperscript{\dag}\AgACIGroupCGIF\ on \metrla;
  \ACIPerGroupFactorCGIF\ on \pemsbay.}
  \vspace{-0.3em}
\end{table}

In the moderate-$N$ graph-native cells, the learned filter $\GNF$
gives sharper at-target ellipsoids than every baseline we evaluate
(Table~\ref{tab:hero}): $\wproj=3.05$ and $1.98$ at joint coverage
$0.910$ and $0.904$, i.e.\ $-23.6\%/-40.5\%$ vs.\ $\StaticCGIF$ and
$-8.4\%/-13.9\%$ vs.\ the strongest at-target non-filter competitor.
On the moderate-$N$ cells the $\ACI$ wrap preserves at-target coverage
and slightly tightens \metrla\ ($3.05\!\to\!2.95$); at full-graph scale
it empirically restores at-target coverage where the unwrapped filter
under-covers (\pemsbay-$325$: $0.878\!\to\!0.899$; §\ref{subsec:scale}).

\paragraph{Mechanism.}
The gain is primarily a covariance-shape gain: $\StaticCGIF$ uses one
static covariance, $\DiagGRU$ learns time variation but no cross-sensor
axes, and conditional-$\Sigma_t$ updates do not close the gap (best
at-target EWMA/rolling/local $\wproj=4.31$/$3.43$ vs.\ $\GNF$'s
$3.05$/$1.98$; App.~\ref{app:leaderboard}). The rank-only sweep drops
$\wproj$ from $5.32$ at $r{=}0$ ($\GCNrankzero$) to $3.04$ at $r{=}4$
on \metrla-$20$, so the diagonal-plus-low-rank head is the clearest
isolated mechanism, and the graph-free \LRGRUcov\ control
(App.~\ref{app:deepvar}) shows the graph
is a modest shape refinement once a learned per-node low-rank covariance
is in place. The Theorem~\ref{thm:logvol} log-volume, comparable only
across genuine single-ellipsoid rules, agrees with the $\wproj$
ranking---$\GNF$ improves $\Vhat_m$ by $0.27$/$0.51$ nats/coordinate
over $\StaticCGIF$ (full table in App.~\ref{app:logvol}).

\subsection{Controlled comparisons (Q1)}
\label{subsec:controlled}

Two appendix controls address mean and model-budget confounding on
\metrla-$20$ and \pemsbay-$50$ ($10$ paired seeds; each under its own
protocol, not mixed with the Table~\ref{tab:hero} runs;
Apps.~\ref{app:same-mean}--\ref{app:deepvar}). With a common frozen
mean and covariance heads fitted only to identical training residuals,
$\GNFcov$ is $24$--$34\%$ sharper in $\wproj$ than static full, static
factor, EWMA, and rolling covariance on \metrla, and $35$--$48\%$
sharper on \pemsbay; all paired $95\%$ CIs exclude zero under both the
\textsc{gpvar} and $\GNF$ means. A rank-$4$ DeepVAR-style control with
the matched optimisation protocol and identical Mahalanobis
conformalisation undercovers and is wider on both cells; the paired
$\DeepVAR{-}\GNF$ differences in $\wproj$ are $1.090\pm0.228$ and
$3.160\pm0.181$, where ``$\pm$'' denotes a $95\%$ CI half-width. These
controls support the covariance-shape interpretation on the two
headline cells; they do not extend the claim beyond them.

\subsection{Theory diagnostics (Q2)}
\label{subsec:diag}

The audited real-data rates diagnose the conclusions used by the
theory but do not verify the population assumptions; they measure
distinct objects and should not be conflated. $\rhodG\!\in\![0.45,0.46]$
is the empirical emitted-law analogue of the
Theorem~\ref{thm:obs-contract} contraction object (a diagnostic of the
conclusion, not a realisability proof). $\rhoDL$ is a finite-horizon
output-sensitivity proxy relating to Assumption~\ref{ass:obs}(O3)
($0.48$--$0.50$ in the cheap audit configuration, $0.86$--$0.89$ on the
deployed model): it diagnoses weakly observable directions on the
deployed trajectory and does \emph{not} certify the population quotient
inequality for $\cobs$ (App.~\ref{app:obs-audit}). $\rhoscore\!\in\!
[0.140,0.153]$ is the downstream score-CDF forgetting rate for
Theorem~\ref{thm:C}, and $\rhoind\!\approx\!0.85$ with integrated
autocorrelation time $\tauint\!\in\![10,13]$ probes the
threshold-dependence regime of the Chebyshev form of
Theorem~\ref{thm:learned-validity} and Assumption~\ref{ass:bernstein}.
The audited $\tauint$ makes the effective calibration size roughly an
order of magnitude below the raw size, consistent with dependence-aware
rather than exchangeability-based coverage; the linear-Kalman check
confirms $\rhodG\!\approx\!\sigma_1$ (App.~\ref{app:dG-audit}), the
warm-up sweep renders the burn-in bias negligible at $t_0\!\ge\!50$
(App.~\ref{app:warmup}), and per-cell values with protocols are in
App.~\ref{app:rhohat-protocol}--\ref{app:tau-int}
(Table~\ref{tab:rho-full}).

\subsection{Full graphs and scope (Q3)}
\label{subsec:scale}\label{subsec:ablation}

\begin{table}[t]
  \caption{\textbf{\textsc{gpvar} scale sweep --- moderate-$N$ win,
  full-graph crossover.} Mean coordinate-projection length $\wproj$
  (10-seed mean$\pm$std) and joint coverage (10-seed mean) at
  $\alpha{=}0.1$; joint-coverage standard deviations for the
  moderate-$N$ cells appear in Table~\ref{tab:hero} and per-seed joint
  coverage in App.~\ref{app:seeds} (Table~\ref{tab:seeds}).
  \best{bold}/\secondbest{underlined} =
  best/second at-target; ``$\downarrow$'' flags joint $<\!0.895$.
  At full graph all seeds use the same sensor set and the
  \textsc{gpvar} fit is deterministic, so methods with deterministic
  fitting are seed-invariant ($\pm0.00$).
  This table fixes a graph-aware \textsc{gpvar} mean backbone; under
  this backbone factor methods become sharper at full graph scale.
  Figure~\ref{fig:headline} and Table~\ref{tab:multibackbone-full}
  show the complementary graph-oblivious Transformer/GRU regime.}
  \label{tab:scale}
  \centering\footnotesize
  \setlength{\tabcolsep}{4.0pt}
  \renewcommand{\arraystretch}{0.85}
  \begin{tabular}{l rc rc rc rc}
    \toprule
    & \multicolumn{4}{c}{\metrla} & \multicolumn{4}{c}{\pemsbay}\\
    \cmidrule(lr){2-5}\cmidrule(lr){6-9}
    & \multicolumn{2}{c}{$N{=}20$} & \multicolumn{2}{c}{$N{=}207$}
    & \multicolumn{2}{c}{$N{=}50$} & \multicolumn{2}{c}{$N{=}325$}\\
    \cmidrule(lr){2-3}\cmidrule(lr){4-5}\cmidrule(lr){6-7}\cmidrule(lr){8-9}
    Method & $\wproj$ & joint & $\wproj$ & joint & $\wproj$ & joint & $\wproj$ & joint\\
    \heromidrule
    \rowcolor{heroRow}
    \heroname{$\GNF$ \textbf{(ours)}}
      & \secondbest{\widthstd{3.05}{0.13}} & $0.910$
      & \widthstd{9.14}{0.24} & $0.902$
      & \best{\widthstd{1.98}{0.06}} & $0.904$
      & \widthstd{8.58}{0.28} & $0.878\downarrow$\!\!\!\!\!\\
    \rowcolor{heroRow}
    \heroname{\quad + $\ACI$ wrap}
      & \best{\widthstd{2.95}{0.17}} & $0.904$
      & \widthstd{9.00}{0.60} & $0.902$
      & \secondbest{\widthstd{1.99}{0.07}} & $0.899$
      & \secondbest{\widthstd{9.01}{0.38}} & $0.899$\\
    \heroend
    \ACIPerGroupFactorCGIF
      & \widthstd{3.60}{0.28} & $0.902$
      & \secondbest{\widthstd{7.21}{0.00}} & $0.902$
      & \widthstd{2.30}{0.09} & $0.899$
      & \best{\widthstd{5.64}{0.00}} & $0.901$\\
    \textsc{CopulaCPTS} \citep{sun2022copula}
      & \widthstd{3.70}{0.24} & $0.907$
      & \best{\widthstd{6.73}{0.04}} & $0.914$
      & \widthstd{5.69}{0.74} & $0.908$
      & \widthstd{14.4}{1.41} & $0.894\downarrow$\!\!\!\!\!\\
    \textsc{MultiDimSPCI} \citep{xu2024multidim}
      & \widthstd{4.10}{0.17} & $0.901$
      & \widthstd{12.6}{0.00} & $0.877\downarrow$\!\!\!\!\!
      & \widthstd{3.40}{0.11} & $0.908$
      & \widthstd{12.7}{0.00} & $0.886\downarrow$\!\!\!\!\!\\
    $\StaticCGIF$ (static $\hat\Sigma$)
      & \widthstd{3.99}{0.17} & $0.902$
      & \widthstd{11.8}{0.00} & $0.880\downarrow$\!\!\!\!\!
      & \widthstd{3.33}{0.11} & $0.907$
      & \widthstd{12.0}{0.00} & $0.887\downarrow$\!\!\!\!\!\\
    $\KalmanFCP$ (tuned Kalman)
      & \widthstd{4.69}{0.28} & $0.906$
      & \widthstd{21.2}{0.00} & $0.900$
      & \widthstd{11.2}{1.03} & $0.897$
      & \widthstd{51.4}{0.00} & $0.852\downarrow$\!\!\!\!\!\\
    \bottomrule
  \end{tabular}
  \vspace{-0.3em}
\end{table}

Full-graph leadership is backbone-conditional. Under graph-aware
backbones factor baselines are sharper at $N{=}207,325$ ($\wproj=5.64$
vs.\ $9.01$ on \pemsbay-$325$); under graph-oblivious GRU/Transformer
the best factor baseline inflates to $\wproj=11.4$--$13.5$ while
$\GNF{+}\ACI$ stays in $8.8$--$9.0$ (sharper by up to $-34\%$), and the
$\ACI$ wrap restores at-target coverage where the unwrapped filter
under-covers (App.~\ref{app:multi-backbone},
Table~\ref{tab:multibackbone-full}). $\KalmanFCP$ blows up at full scale
($51.4$ on \pemsbay-$325$, below target); on non-graph-native data
\textsc{CopulaCPTS} (\solar, \loopseattle) or CGIF-family baselines
(\ett, \jena) are sharpest at-target
(Table~\ref{tab:multibackbone-nontraffic}).

\paragraph{What a high-dimensional joint set buys.}
The joint region controls one deployment-level event
$\{Y_t\in\Cset_t\}$, relevant when a decision depends on the complete
network state (capacity, routing, aggregate safety, anomaly detection),
not a replacement for a marginal interval when only one sensor matters.
Any operational linear functional is available in closed form,
$a^\top Y_t\in a^\top\hat\mu_t\pm\sqrt{\hat q_t\,a^\top\hat\Sigma_t a}$
(§\ref{sec:prelim}); such projections inherit the simultaneous
guarantee but need not be shortest, and simultaneous coverage grows
conservative as $N$ grows---why our strongest claim is at moderate $N$
and we report the crossover explicitly.
\FloatBarrier

\section{Discussion and limitations}
\label{sec:discussion}

\paragraph{Main takeaway and scope.}
Learned conditional covariance shapes are most useful in the
moderate-dimensional graph-native traffic regime. At full-graph scale,
factor methods can be sharper when graph structure is already absorbed
by the mean backbone. The same-mean and matched-budget controls
(§\ref{subsec:controlled}) support the covariance-shape interpretation
on the two headline cells, but do not establish universal superiority.

\paragraph{Theory and validity.}
Theorem~\ref{thm:obs-contract} is a conditional NLL-to-contraction
transfer on a parameter-specific observable quotient (comparing scalar
moduli, never points across two quotients), not a guarantee for
arbitrary trained GRUs. Together with input sensitivity and score
dependence it gives unconditional approximate marginal validity, not
calibration-conditional coverage; exact validity holds only in the
standardised-innovation case (Proposition~\ref{prop:oracle}). The
audits probe finite-trajectory conclusions and proxies
(§\ref{subsec:diag}) rather than population realisability or
observability, and Theorem~\ref{thm:logvol} is restricted to
conditionally valid Gaussian ellipsoid rules.

\paragraph{Architecture and numerical limitations.}
The node-shared per-node recurrence keeps $\GNF$ small (parameter count
independent of $N$) but its fixed shared capacity and single graph
shift may lose node-specific information at large $N$, consistent with
the crossover. The default covariance floor is nonbinding in the
audited runs, although a high-condition-number tail remains
(App.~\ref{app:floor-sweep}). $\GNF$ uses no per-dataset
hyperparameter search (the only per-cell-tuned row is a filter
\emph{baseline}, Table~\ref{tab:tuning-ledger}), and the reported
full-graph and non-graph-native crossovers delimit the intended
empirical scope.


\clearpage
\begin{ack}
The author thanks the anonymous reviewers and area chairs for their
careful reading and constructive feedback, which sharpened the
statement of the observable-quotient theory and improved the design and
reporting of the empirical controls.

\paragraph{Funding.}
This work was carried out without external funding. The compute used
for training, calibration, and the auditing experiments was drawn from
the author's ordinary research resources; its hardware and wall-clock
footprint is documented in App.~\ref{app:footprint}.

\paragraph{Competing interests.}
The author is employed by DRW. The views expressed
here are those of the author and do not necessarily represent those of
DRW.

\paragraph{Code and reproducibility.}
The full code base, including a self-contained reproducibility suite,
is publicly available under an MIT licence
(App.~\ref{app:repr-suite}). Seeds, run counts, and dataset provenance
are described in App.~\ref{app:reproducibility}; the exact commands
used to regenerate every reported number are documented in the
repository.

\end{ack}

\bibliographystyle{plainnat}
\bibliography{references.bib}

\appendix
\section{Observable-contraction theory}
\label{app:obs-theory}

This section collects the proofs of the main theorems
(Theorem~\ref{thm:obs-contract}, Corollary~\ref{cor:thresh-mix},
Theorem~\ref{thm:learned-validity}, Theorem~\ref{thm:logvol}) and
defines the metric and identification objects they rely on.

\subsection{Bures-Wasserstein metric on Gaussian predictive laws}
\label{app:bures}

Recall \eqref{eq:dG}. For Gaussian laws on $\R^N$ with covariances
spectrum-bounded between $0<\lambda_-\le\lambda_+<\infty$, the
following relations hold on a compact eigenvalue interval determined
only by $\lambda_-/\lambda_+$:
\begin{equation}\label{eq:bures-equiv}
\begin{aligned}
  c_{\lambda}\,\dG^{2}(P,P')
  &\;\le\;\mathrm{KL}(P\,\|\,P')\;\le\;C_{\lambda}\,\dG^{2}(P,P'),\\[-1pt]
  c_{\lambda}\,\|\Sigma-\Sigma'\|_{F}^{2}
  &\;\le\;
  \tr\!\bigl[\Sigma+\Sigma'-2(\Sigma'^{1/2}\Sigma\Sigma'^{1/2})^{1/2}\bigr]
  \;\le\; C_\lambda\,\|\Sigma-\Sigma'\|_F^{2}.
\end{aligned}
\end{equation}
The inequalities follow from the bound
$x-\log x-1\ge c_\lambda(x-1)^{2}$ on the eigenvalue interval of
$\Sigma'^{-1/2}\Sigma\Sigma'^{-1/2}$, combined with the closed-form
Gaussian KL. Hence \emph{a bound in any one of the three metrics
implies a bound in the other two, with constants depending only on
$\lambda_-,\lambda_+$.} We use $\dG$ in the theorem statements
because Pinsker $\mathrm{TV}\le\sqrt{\mathrm{KL}/2}$ then transports
KL bounds to $\mathrm{TV}$ bounds at no constant cost.

\subsection{Parameter-specific observable quotient and metric}
\label{app:obs-quotient}

Fix a filter parameter $\theta$ and let the filter have a
(non-minimal) hidden state $h$ emitting
$O_\theta(h)=\mathcal N(\mu_\theta(h),\Sigma_\theta(h))$. For $\theta$
fixed define the equivalence relation
\begin{equation}\label{eq:simtheta}
  h\sim_\theta h'
  \;\Longleftrightarrow\;
  O_\theta\bigl(H^{h,\mathbf u}_{\theta,\ell}\bigr)
  =O_\theta\bigl(H^{h',\mathbf u}_{\theta,\ell}\bigr)
  \quad\forall\,\ell\ge 0,\ \forall\,\mathbf u,
\end{equation}
where $\mathbf u$ ranges over admissible common input continuations and
$H^{h,\mathbf u}_{\theta,\ell}$ (defined below) is the step-$\ell$
hidden state; equivalently, $h$ and $h'$ emit the same future
predictive-law sequence under every such continuation. We set
$\Qth{\theta}:=\state/\!\sim_\theta$ and do \emph{not} identify
elements of $\Qth{\theta_\star}$ with elements of
$\Qth{\widehat\theta}$; each parameter has its own quotient. For the
GCN-GRU \GraphNeuralSSM, $\Qth{\theta}$ removes exactly the hidden
directions invisible to $g_\theta$ along admissible inputs.

\paragraph{Metric.} Let $H^{h,\mathbf u}_{\theta,\ell}$ denote the
hidden state reached at step $\ell$ from $h$ under input continuation
$\mathbf u$, write $\overline d_{\mathcal G}(P,P'):=\min\{1,\dG(P,P')\}$,
and fix summable weights $w_\ell>0$, $\sum_\ell w_\ell<\infty$. Define
\begin{equation}\label{eq:dtheta}
  \dth{\theta}\bigl([h]_\theta,[h']_\theta\bigr)
  :=
  \sup_{\mathbf u\in\mathcal U^\infty}
  \Bigl[\,\sum_{\ell\ge 0} w_\ell\,
    \overline d_{\mathcal G}^{\,2}\bigl(
      O_\theta(H^{h,\mathbf u}_{\theta,\ell}),
      O_\theta(H^{h',\mathbf u}_{\theta,\ell})\bigr)\Bigr]^{1/2}.
\end{equation}
Truncating $\dG$ at $1$ before squaring makes both finiteness and the
triangle inequality transparent. We record the four required
properties.
\emph{(1)~Representative-independence.} If $h_1\sim_\theta h_2$ and
$h_1'\sim_\theta h_2'$ then, by~\eqref{eq:simtheta}, the two emitted
sequences agree for all $\mathbf u$, so the summand
in~\eqref{eq:dtheta} is unchanged; $\dth{\theta}$ is thus well-defined
on classes.
\emph{(2)~Vanishing exactly on $\sim_\theta$.}
$\dth{\theta}([h]_\theta,[h']_\theta)=0$ iff
$\overline d_{\mathcal G}(O_\theta(H^{h,\mathbf u}_{\theta,\ell}),
O_\theta(H^{h',\mathbf u}_{\theta,\ell}))=0$ for all $\ell,\mathbf u$
(all $w_\ell>0$), i.e.\ iff $h\sim_\theta h'$.
\emph{(3)~Finiteness and triangle inequality.} Each summand is
$\le w_\ell$, so $\dth{\theta}\le(\sum_\ell w_\ell)^{1/2}<\infty$;
$\overline d_{\mathcal G}$ is a metric on Gaussian laws, so for each
$(\ell,\mathbf u)$ the weighted-$\ell^2$ triangle inequality holds,
and taking the supremum over $\mathbf u$ preserves it.
\emph{(4)~Transition well-defined on $\Qth{\theta}$.} If
$h\sim_\theta h'$ then applying the same first input $y$ yields states
that still agree under every future continuation
(by~\eqref{eq:simtheta} shifted by one step), so
$T_\theta(y,h)\sim_\theta T_\theta(y,h')$ and the induced map
$T_\theta(y,\cdot)$ on $\Qth{\theta}$ is well-defined.

\paragraph{Finite-horizon observable discrepancy.} Distinguished from
the infinite-horizon quotient metric, we also use, for horizon
$\Hobs\ge 0$ and \emph{one fixed} probing distribution $\nu$ on
admissible continuations that is \emph{independent of $\theta$},
\begin{equation}\label{eq:Doh}
  \Dobs{\theta}^{2}(q,q')
  \;:=\;
  \sum_{\ell=0}^{\Hobs}\beta_\ell\,
  \E_{\nu}
  \bigl[
    \dG^{2}\bigl(O_{\theta}(Q^{q}_{\theta,\ell}),O_{\theta}(Q^{q'}_{\theta,\ell})\bigr)
  \bigr],
  \qquad \beta_\ell>0,
\end{equation}
where both trajectories use the same future input continuation.
Fixing $\nu$ and the $\beta_\ell$ independently of $\theta$ is what
makes $\Dobs{\theta}$ a genuine per-parameter discrepancy rather than
a quantity that moves through both the filter and the probe.
Assumption~\ref{ass:obs}(O3) states the local observability relation
$\cobs\,\dth{\theta}(q,q')\le\Dobs{\theta}(q,q')$ together with the
emitted-law Lipschitz relation
$\dG(O_\theta(q),O_\theta(q'))\le\Lemit\,\dth{\theta}(q,q')$.
Instantaneous observability is the case $\Hobs=0$, $\beta_0=1$, where
$\Dobs{\theta}(q,q')=\dG(O_\theta(q),O_\theta(q'))$.

\subsection{Risk identifies the contraction modulus: Fisher margin}
\label{app:fisher}

Assumption~\ref{ass:obs}(O4) says NLL excess risk controls the
\emph{scalar} contraction-modulus gap
$\Dctr{\theta}=[\kmod{\theta}-\kmod{\theta_\star}]_+$, where
$\kmod{\theta}=\sup_y\sup_{q\ne q'}
\dth{\theta}(T_\theta(y,q),T_\theta(y,q'))/\dth{\theta}(q,q')$. This
removes any need to compare points or transition maps across two
different quotient spaces. We give a sufficient condition on that one
scalar.

\begin{lemma}[Fisher margin, scalar-modulus form]\label{lem:fisher}
Let $\Theta_0$ be an identifiable local chart around $\theta_\star$
(parameter symmetries such as low-rank rotations $L\mapsto LQ$
quotiented out). Suppose the Gaussian-NLL population risk has a local
quadratic margin
$\Robs(\theta)-\Robs(\theta_\star)\ge\lambda_R\|\theta-\theta_\star\|^{2}$
on $\Theta_0$, and the contraction modulus is locally Lipschitz,
$[\kmod{\theta}-\kmod{\theta_\star}]_+\le L_\kappa\|\theta-\theta_\star\|$
on $\Theta_0$. Then
\[
  \Dctr{\theta}^{2}
  \;\le\;
  \frac{L_\kappa^{2}}{\lambda_R}\,[\Robs(\theta)-\Robs(\theta_\star)],
\]
i.e.\ \emph{(O4)} holds with $\Cid=L_\kappa^{2}/\lambda_R$.
\end{lemma}

\begin{proof}
$\Dctr{\theta}^{2}\le L_\kappa^{2}\|\theta-\theta_\star\|^{2}
\le(L_\kappa^{2}/\lambda_R)[\Robs(\theta)-\Robs(\theta_\star)]$.
\end{proof}

We present local Lipschitzness of the scalar modulus as the
sufficient regularity condition; we do \emph{not} claim that a generic
transition-smoothness argument automatically proves local
Lipschitzness of $\kmod{\cdot}$. The Fisher-margin calculation below
supplies the risk side ($\lambda_R$).

\paragraph{Sufficient condition for the quadratic risk margin (GCN-GRU).}
\label{app:fisher-gcn}
Write the conditional Gaussian-NLL Hessian at the emitted parameter
$\eta=(\mu,\Sigma)$ as $\nabla_\eta^{2}\E[\ell\mid\Fcal_{t-1}]\succeq c_G I$
on $\mathcal G_\lambda$ (uniformly positive on the
spectrum-bounded class; $c_G$ depends on $\lambda_-,\lambda_+$). Let
$J_t=\partial\eta_{\theta,t}/\partial\theta|_{\theta_\star}$ and
suppose the observable Fisher Gramian
$\E[J_t^{\top}J_t]\succeq\gamma I$ after quotienting parameter
symmetries. Then
$\nabla_\theta^{2}\Robs(\theta_\star)=\E[J_t^{\top}\nabla_\eta^{2}\ell_t J_t]\succeq c_G\gamma I$,
and continuity of the Hessian gives the quadratic margin
$\Robs(\theta)-\Robs(\theta_\star)\ge(c_G\gamma/4)\|\theta-\theta_\star\|^{2}=\lambda_R\|\theta-\theta_\star\|^2$
on a neighbourhood of $\theta_\star$. The condition
$\E[J_t^{\top}J_t]\succeq\gamma I$ is the formal version of
\emph{finite-horizon Fisher observability}: likelihood training
identifies the parts of the parameter that move future emitted
predictive laws. Paired with local Lipschitzness of $\kmod{\cdot}$ it
yields (O4).

\subsection{Proof of Theorem~\ref{thm:obs-contract}}
\label{app:proofs-obs}

The proof proceeds entirely within the learned filter's own quotient
$\Qth{\widehat\theta}$; there is no comparison of $T_{\widehat\theta}$
and $T_{\theta_\star}$ in a common metric.

\emph{Step 1 (modulus identification).} By Assumption~\ref{ass:obs}(O4),
$\Dctr{\widehat\theta}^{2}\le\Cid\eNLL$, i.e.\
$[\kmod{\widehat\theta}-\kmod{\theta_\star}]_+\le\sqrt{\Cid\eNLL}$.
With (O1) ($\kmod{\theta_\star}\le\rhostar$) this gives
\[
  \kmod{\widehat\theta}\;\le\;\kmod{\theta_\star}+\sqrt{\Cid\eNLL}
  \;\le\;\rhostar+\sqrt{\Cid\eNLL}\;=\;\rhoobs\;<\;1.
\]

\emph{Step 2 (quotient contraction).} By definition of the modulus,
$\dth{\widehat\theta}(T_{\widehat\theta}(y,q),T_{\widehat\theta}(y,q'))
\le\kmod{\widehat\theta}\,\dth{\widehat\theta}(q,q')
\le\rhoobs\,\dth{\widehat\theta}(q,q')$ for every admissible $y$.
Iterating $t$ times along the same input path (the transition is
well-defined on $\Qth{\widehat\theta}$, App.~\ref{app:obs-quotient})
gives~\eqref{eq:obs-contract-quotient},
$\dth{\widehat\theta}(Q^{q}_{\widehat\theta,t},Q^{q'}_{\widehat\theta,t})\le\rhoobs^{t}\dth{\widehat\theta}(q,q')$.

\emph{Step 3 (emitted-law Lipschitzness).} By the (O3) Lipschitz
relation,
$\dG(\Plawth{\widehat\theta}{t}^{q},\Plawth{\widehat\theta}{t}^{q'})
=\dG(O_{\widehat\theta}(Q^{q}_{\widehat\theta,t}),O_{\widehat\theta}(Q^{q'}_{\widehat\theta,t}))
\le\Lemit\,\dth{\widehat\theta}(Q^{q}_{\widehat\theta,t},Q^{q'}_{\widehat\theta,t})
\le\Lemit\,\rhoobs^{t}\,\dth{\widehat\theta}(q,q')$.

\emph{Step 4 (finite-horizon observability at the initial states).}
The (O3) lower bound at the initial learned-filter quotient states
gives $\dth{\widehat\theta}(q,q')\le\Dobs{\widehat\theta}(q,q')/\cobs$.

\emph{Step 5 (combine).} Substituting Step~4 into Step~3,
$\dG(\Plawth{\widehat\theta}{t}^{q},\Plawth{\widehat\theta}{t}^{q'})\le(\Lemit/\cobs)\rhoobs^{t}\Dobs{\widehat\theta}(q,q')$,
which is~\eqref{eq:obs-contract}. Under instantaneous observability
($\Hobs=0$, $\beta_0=1$), $\Dobs{\widehat\theta}(q,q')$ collapses to
$\dG(\Plawth{\widehat\theta}{0}^{q},\Plawth{\widehat\theta}{0}^{q'})$,
recovering the initial-emitted-law form.
\hfill\qedhere

\subsection{Proof of Corollary~\ref{cor:thresh-mix}}
\label{app:thresh-mix}

Let $Z_i=(\state_{\widehat\theta,i},Y_i)$ be the augmented filter+data
state and $a_u(z):=\Prob(\rscore_{\widehat\theta,i+1}\le u\mid Z_i=z)$.

\emph{Augmented coupling.} Couple two copies $(Z_i),(Z_i')$ driven by
the geometrically coupled inputs of Assumption~\ref{ass:input}. The
input-sensitivity recursion
$\dth{\widehat\theta}(T_{\widehat\theta}(y,q),T_{\widehat\theta}(y',q'))
\le\rhoobs\,\dth{\widehat\theta}(q,q')+\Lin\,d_Y(y,y')$
unrolled against $\E\,d_Y(Y_t,Y_t')\le C_Y\rho_{\mathrm{data}}^{t}d_Y(Y_0,Y_0')$
gives, in expectation, a filter-gap bound
$\E\,\dth{\widehat\theta}(Q_t,Q_t')\le
C\sum_{s\le t}\rhoobs^{\,t-s}\rho_{\mathrm{data}}^{\,s}\,d_Y(Y_0,Y_0')
+\rhoobs^{t}\dth{\widehat\theta}(q,q')$. Two cases.
(i)~If $\rho_{\mathrm{data}}\ne\rhoobs$ the geometric sum is bounded by
$C'\max\{\rho_{\mathrm{data}},\rhoobs\}^{\,t}$, so the augmented state
couples at any $\rhoaug=\max\{\rho_{\mathrm{data}},\rhoobs\}$ with an
adjusted constant. (ii)~If $\rho_{\mathrm{data}}=\rhoobs=:\rho$ the
convolution yields $t\,\rho^{t}$, which for any target
$\rhoaug\in(\rho,1)$ satisfies $t\rho^{t}\le C_{\rhoaug}\rhoaug^{t}$;
hence the augmented state couples at any $\rhoaug$ strictly above
$\rho$. In both cases
$\E[d_Z(Z_{i+k}^{Z_i},Z_{i+k}^{Z_i'})\mid Z_i,Z_i']\le C_Z\rhoaug^{k}d_Z(Z_i,Z_i')$
with $\rhoaug<1$; we do \emph{not} assert coupling at exactly
$\max\{\rho_{\mathrm{data}},\rhoobs\}$ when the rates coincide.

\emph{Envelope.} Local Lipschitzness of $a_u$ gives
$|\E[I_{i+k}(u)\mid Z_i]-\E I_{i+k}(u)|\le L_a C_Z\rhoaug^{k}\E[d_Z(Z_i,Z_i')\mid Z_i]$,
and the tower-property identity for covariance plus
$|I_i(u)-\E I_i(u)|\le 1$ yields
$|\Cov(I_i(u),I_{i+k}(u))|\le L_a C_Z\E[d_Z(Z_i,Z_i')]\,\rhoaug^{k}=:C_\Gamma\rhoaug^{k}$.
The frozen-filter contraction premise $\rhoobs<1$ is supplied by
Theorem~\ref{thm:obs-contract} for a likelihood-trained filter, but
the argument uses only that premise together with
Assumption~\ref{ass:input}, not Gaussian oracle realisability.
\hfill\qedhere

\subsection{Proof of Theorem~\ref{thm:learned-validity}}
\label{app:proofs-validity}

We prove the two logically separate conclusions in turn: Part~I is a
high-probability bound on the calibration-quantile error, and Part~II
is a deterministic unconditional-coverage statement obtained from the
quantile event by pathwise inclusion (no calibration/test
independence is used).

\paragraph{Part I: quantile-error radius $r_m(\delta)$.}
Write $p=1-\alpha$, $\sstar=(\Gst)^{-1}(p)$, $k=\lceil(m+1)p\rceil$,
$\hat s_m=S_{(k)}$, and
$\eta_m=\max\{|k/m-p|,|(k-1)/m-p|\}=O(m^{-1})$ for the conformal rank
offset. Let $b_m=\sup_{u\in I}|\E\hat F_m(u)-\Gst(u)|$; the augmented-process
coupling and local conditional-CDF regularity used in
Corollary~\ref{cor:thresh-mix} (App.~\ref{app:proofs-C}) give the
warm-start bias $b_m\lesssim\rhoaug^{\,t_0}$. For $\varepsilon$ with
$\sstar\pm\varepsilon\in I$, the density lower bound gives
$\Gst(\sstar+\varepsilon)\ge p+\underline g\,\varepsilon$ and
$\Gst(\sstar-\varepsilon)\le p-\underline g\,\varepsilon$. If
$\hat s_m>\sstar+\varepsilon$ then $\hat F_m(\sstar+\varepsilon)\le(k-1)/m\le p+\eta_m$, so
$\hat F_m(\sstar+\varepsilon)-\E\hat F_m(\sstar+\varepsilon)\le-(\underline g\,\varepsilon-b_m-\eta_m)$;
symmetrically at $\sstar-\varepsilon$. Hence whenever
$\underline g\,\varepsilon>b_m+\eta_m$,
\begin{equation}\label{eq:qinv}
\Prob(|\hat s_m-\sstar|>\varepsilon)
\le\!\!\sum_{\sigma\in\{\pm 1\}}\!\!
\Prob\bigl(|\hat F_m(\sstar+\sigma\varepsilon)-\E\hat F_m(\sstar+\sigma\varepsilon)|>\underline g\,\varepsilon-b_m-\eta_m\bigr).
\end{equation}
\emph{Chebyshev radius.} By Corollary~\ref{cor:thresh-mix},
$\sum_{k\ge 1}\Gamma_k\le C_\Gamma\rhoaug/(1-\rhoaug)$, so
$\Var(\hat F_m(u))\le 1/(4\meffcov)$ with
$\meffcov\asymp m/(1+8C_\Gamma\rhoaug/(1-\rhoaug))$. Chebyshev at
level $\delta$ and~\eqref{eq:qinv} give
$|\hat s_m-\sstar|\le r_m^{\mathrm{Cheb}}(\delta)$ with probability
$\ge 1-\delta$, where
$\underline g\,r_m^{\mathrm{Cheb}}\lesssim\rhoaug^{t_0}+(\delta\meffcov)^{-1/2}+m^{-1}$,
i.e.\ \eqref{eq:qconc-cheb}.
\emph{Bernstein radius.} Under Assumption~\ref{ass:bernstein},
$\Prob(|\hat F_m(u)-\E\hat F_m(u)|>\varepsilon)\le 2\exp(-c\meff\varepsilon^{2})$
with $\meff\asymp m(1-\rhoaug)/(1+\rhoaug)$; inverting and
using~\eqref{eq:qinv} gives
$|\hat s_m-\sstar|\le r_m^{\mathrm{Bern}}(\delta)$ with
$\underline g\,r_m^{\mathrm{Bern}}\lesssim\rhoaug^{t_0}+\sqrt{\log(1/\delta)/\meff}+\log(1/\delta)/\meff+m^{-1}$,
i.e.\ \eqref{eq:qconc-bern}. In both cases the exact constants, the
local-density factors $\underline g,\bar g$, the rank offset
$\eta_m$, and the calibration-bias term are retained as displayed.
Corollary~\ref{cor:thresh-mix} alone supports the Chebyshev radius
(its covariance summability controls only $\Var(\hat F_m)$, giving a
$(\delta\meffcov)^{-1/2}$ deviation); the sub-Gaussian
$\sqrt{\log(1/\delta)/\meff}$ radius requires the strictly stronger
Assumption~\ref{ass:bernstein}, whose sufficient conditions include
geometric $\beta$-mixing of the augmented filter+data state
\citep{douc2018markov}, a Markov-coupling spectral gap, or
exponential $\phi$-mixing.

\paragraph{Part II: unconditional coverage via the quantile event.}
Let $E_\delta:=\{|\hat s_m-s_\star|\le r_m(\delta)\}$, so
$\Prob(E_\delta^c)\le\delta$ by Part~I. On $E_\delta$, pathwise,
\begin{equation}\label{eq:pathwise-incl}
  \{S_t\le s_\star-r_m(\delta)\}
  \;\subseteq\;
  \{S_t\le\hat s_m\}
  \;\subseteq\;
  \{S_t\le s_\star+r_m(\delta)\}.
\end{equation}
This inclusion is deterministic and uses \emph{no} independence
between calibration and test. Taking probabilities and adding
$\Prob(E_\delta^c)\le\delta$ to both sides,
\[
  \Prob\{S_t\le s_\star-r_m(\delta)\}-\delta
  \;\le\;\Prob\{Y_t\in\Cset_t\}\;\le\;
  \Prob\{S_t\le s_\star+r_m(\delta)\}+\delta .
\]
Write the test-score CDF discrepancy
$b_t:=\sup_{u\in I}|\Prob\{S_t\le u\}-G_\star(u)|$ and use the density
upper bound $\bar g$: $\Prob\{S_t\le s_\star\pm r_m\}\le
G_\star(s_\star\pm r_m)+b_t\le(1-\alpha)\pm\bar g\,r_m(\delta)+b_t$
(and the matching lower bound). Hence
$|\Prob\{Y_t\in\Cset_t\}-(1-\alpha)|\le b_t+\bar g\,r_m(\delta)+\delta$,
which is~\eqref{eq:uncond-cov}. Because
$\rscore_t(y)\le\hat s_m\Leftrightarrow\score_t(y)\le\hat s_m^{2}=\hat q_{1-\alpha}$,
this is the coverage gap of $\Cset_t$. The statement is an
unconditional marginal-coverage bound for the contiguous
calibration/test trajectory; it aligns Theorem~\ref{thm:learned-validity}
with the deterministic form of Proposition~\ref{prop:approx}, and is
\emph{not} a calibration-conditional guarantee.
\hfill\qedhere

\subsection{Proof of Theorem~\ref{thm:logvol}}
\label{app:proofs-logvol}

\paragraph{Step 1: volume formula.}
For $\theta$ in the local class, the ellipsoid $\Cset_{\theta,t}(s)=\{y:\rscore_{\theta,t}(y)\le s\}$
is the affine image of the Euclidean ball of radius $s$ under
$\hat\Sigma_{\theta,t}^{1/2}$ centred at $\mu_{\theta,t}$, so
$\mathrm{Vol}(\Cset_{\theta,t}(s))=v_N s^{N}\det(\hat\Sigma_{\theta,t})^{1/2}$
and \eqref{eq:logvol} holds.

\paragraph{Step 2: oracle minimisation among $\Fcond$.}
Under Gaussian oracle realisability, the conditional law $P_t^\star=\mathcal N(\mu_t^\star,\Sigma_t^\star)$
has Gaussian density $f_t$, and the Bayes-projection ellipsoid
$\Cset_{\star,t}=\{y:f_t(y)\ge c_t\}$ at the threshold making
$P_t^\star(\Cset_{\star,t})=1-\alpha$ has Lebesgue volume \emph{minimal}
among Borel sets of conditional coverage $\ge 1-\alpha$, by a
likelihood-ratio level-set argument: for any other set $A_t$ with
$P_t^\star(A_t)\ge 1-\alpha$,
$0\le\int_{A_t\setminus\Cset_{\star,t}}f_t-\int_{\Cset_{\star,t}\setminus A_t}f_t\le c_t[\mathrm{Leb}(A_t)-\mathrm{Leb}(\Cset_{\star,t})]$.
Hence $\inf_{\theta\in\Fcond}\Vfun(\theta)=\Vfun(\theta_\star)$.

\paragraph{Step 3: NLL excess controls the determinant term.}
For Gaussian laws on $\mathcal G_\lambda$, the closed-form KL
$K_t=\mathrm{KL}(P_t^\star\|Q_{\hat\theta,t})$ contains the term
$\tfrac12\sum_i\phi(a_{t,i})$ with $\phi(x)=x-\log x-1$ and
$a_{t,i}$ the eigenvalues of
$\hat\Sigma_{\hat\theta,t}^{-1/2}\Sigma_t^\star\hat\Sigma_{\hat\theta,t}^{-1/2}$.
On the compact eigenvalue interval $[\lambda_-/\lambda_+,\lambda_+/\lambda_-]$,
$(\log x)^{2}\le C_\lambda\phi(x)$, so by Cauchy-Schwarz
$|\log\det\hat\Sigma_{\hat\theta,t}-\log\det\Sigma_t^\star|=|\sum_i-\log a_{t,i}|\le\sqrt{N\sum_i(\log a_{t,i})^{2}}\le\sqrt{2NC_\lambda K_t}$.
Taking expectation and Jensen,
$\tfrac{1}{2N}\E_t|\log\det\hat\Sigma_{\hat\theta,t}-\log\det\Sigma_t^\star|\le C_{\det}\sqrt{\eNLL}$.

\paragraph{Step 4: NLL excess controls the population radius.}
Pinsker gives $\mathrm{TV}(P_t^\star,Q_{\hat\theta,t})\le\sqrt{K_t/2}$,
so $\sup_u|G_{\hat\theta}(u)-G_{\chi_N}(u)|\le\sqrt{\eNLL/2}=:\tau_n$ in
expectation, where $G_{\chi_N}$ is the $\chi_N$ CDF and
$G_{\hat\theta}$ the population score CDF under $P_t^\star$. Local
density $\dens\ge\underline g_0$ at $s_\star=G_{\chi_N}^{-1}(1-\alpha)$
yields $|s_{\hat\theta}-s_\star|\le 2\tau_n/\underline g_0\le(\sqrt{2}/\underline g_0)\sqrt{\eNLL}$
and $|\log s_{\hat\theta}-\log s_\star|\le(2\sqrt 2/(s_\star\underline g_0))\sqrt{\eNLL}$.

\paragraph{Step 5: empirical calibration radius adds $r_m(\delta)$.}
On the quantile-concentration event $E_\delta$ of
Theorem~\ref{thm:learned-validity} (Part~I),
$|\hat s_m-s_{\hat\theta}|\le r_m(\delta)$ with $r_m(\delta)$ the
Chebyshev/Bernstein radius \eqref{eq:qconc-cheb}--\eqref{eq:qconc-bern},
hence $\log\hat s_m\le\log s_\star+(2\sqrt 2/(s_\star\underline g_0))\sqrt{\eNLL}+(2/s_\star)r_m(\delta)$.

\paragraph{Step 6: combine.}
Using \eqref{eq:logvol} and Steps~3--5,
$\Vhat_m(\hat\theta)\le\Vfun(\theta_\star)+C_1\sqrt{\eNLL}+C_2\,r_m(\delta)$
with $C_1=C_{\det}+2\sqrt 2/(s_\star\underline g_0)$ and $C_2=2/s_\star$.
Step~2 identifies $\Vfun(\theta_\star)=\inf_{\theta\in\Fcond}\Vfun(\theta)$,
giving \eqref{eq:logvol-bound}.
\hfill\qedhere

\section{Full proofs and derivations}
\label{app:proofs}

This section collects the auxiliary statements used by §\ref{sec:theory}
and the linear-Kalman specialisation.

\subsection{Why the Mahalanobis score}
\label{app:why-mahalanobis}

The score choice in~\eqref{eq:score} is non-trivial: changing the
score changes the geometry of the emitted region. We collect here the
three properties that single out the squared Mahalanobis score among
alternatives on correlated $Y$ --- briefly stated in
§\ref{sec:prelim} and elaborated here for completeness.

\paragraph{Affine invariance.}
Let $\psi: Y\mapsto AY+b$ be any invertible affine reparameterisation,
and write $(\hat Y^{\psi}, \hat\Sigma^{\psi})=(A\hat Y+b, A\hat\Sigma A^{\top})$
for the transformed predictive moments. Then
$(\psi(Y)-\hat Y^{\psi})^{\top}(\hat\Sigma^{\psi})^{-1}(\psi(Y)-\hat Y^{\psi})
=(Y-\hat Y)^{\top}\hat\Sigma^{-1}(Y-\hat Y)$, so the score is
$\psi$-invariant and the emitted region transforms as $\psi(\Cset_\alpha)$.
Competing scores $\|r\|_\infty$, $\|r\|_2$, and the per-coordinate
studentised $\max_i|r_i|/\sigma_i$ are \emph{not} affine-invariant;
each picks a preferred basis and pays for it on correlated $Y$.

\paragraph{Minimum-volume property at Gaussian targets.}
Among all Borel sets $C\subseteq\R^N$ of Gaussian-measure $1-\alpha$
under $\Normal(\hat Y,\hat\Sigma)$, the Lebesgue-minimum-volume set is
the Mahalanobis ellipsoid
$\{y:(y-\hat Y)^{\top}\hat\Sigma^{-1}(y-\hat Y)\le \chi^2_{N,1-\alpha}\}$
\citep[Ch.~1]{anderson1979optimal}. Conformalising the squared
Mahalanobis score replaces the $\chi^2$ threshold by the empirical
quantile $\hat q_{1-\alpha}$, so the emitted region is always a
Mahalanobis ellipsoid; validity is finite-sample (not Gaussian) but
the shape is optimal whenever the predictive law is approximately
Gaussian.

\paragraph{Reduction to a scalar conformal problem.}
Split-conformal operates on a scalar score. Alternative ellipsoidal
constructions must either calibrate a $\chi^2$-like statistic with
known tail (forfeiting finite-sample guarantees on dependent data) or
post-hoc inflate per-coordinate intervals to cover the joint event
(producing a Cartesian product that ignores cross-sensor structure).
The squared Mahalanobis score recovers the ellipsoid geometry while
staying in the scalar-calibration regime Theorem~\ref{thm:A} needs.

\paragraph{Choice of squared vs.\ square-root form.}
The squared score $\score_t$ and its root $\rscore_t=\sqrt{\score_t}$
produce the same prediction region: split-conformal is invariant under
monotone transforms (the order statistic commutes with monotone maps),
and $\hat q_{1-\alpha}=\hat s_{m,t_0}^2$. The implementation uses the
squared form because $h\mapsto\score(h,\xi)$ is smooth wherever
$(\mu,\Sigma)$ are, avoiding the $1/\sqrt{\score}$ blow-up at
$\score=0$; the analysis uses the root form because $\score$ is not
globally Lipschitz in $h$ on an unbounded state space (its gradient
grows linearly with the residual), whereas $\rscore$ admits the
finite-constant root-score stability (C2$'$) of
Lemma~\ref{lem:Lroot}. Under Proposition~\ref{prop:oracle} the squared
scores are $\chi^2_N$-distributed (root scores are $\chi_N$); the two
forms calibrate the same ellipsoid.

\subsection{Proof of Theorem~\ref{thm:A} (split-conformal validity)}

By exchangeability, for any permutation $\pi$ of $\{1,\dots,n+1\}$ the joint
law of $(\score_{\pi(1)},\dots,\score_{\pi(n+1)})$ equals that of
$(\score_1,\dots,\score_{n+1})$. Hence the rank of $\score_{n+1}$ in the pooled
sample is uniform on $\{1,\dots,n+1\}$. Setting
$k:=\lceil(n+1)(1-\alpha)\rceil$ and $\hat q_{1-\alpha}:=\score_{(k)}$,
$\Prob\{\score_{n+1}\le \hat q_{1-\alpha}\}=\Prob\{\mathrm{rank}(\score_{n+1})\le k\}=k/(n+1)\ge 1-\alpha$.
When the scores are a.s.\ distinct the inequality can be sharpened to the
two-sided bound $1-\alpha\le\Prob\{\cdot\}\le 1-\alpha+1/(n+1)$.
\hfill$\square$

\subsection{Warm-start dependent-quantile concentration}
\label{app:proofs-qhat-warm}

We isolate the warm-start dependent-quantile theorem used as a proof
ingredient in Theorem~\ref{thm:learned-validity}. This auxiliary
result is stated on a generic metric state space. In the
learned-filter application, Theorem~\ref{thm:obs-contract} supplies
same-input contraction on $\Qth{\widehat\theta}$,
Assumption~\ref{ass:input} lifts this to coupling of the augmented
data--filter process, and Corollary~\ref{cor:thresh-mix} supplies the
threshold-autocovariance envelope of Assumption~\ref{ass:G};
Theorem~\ref{thm:obs-contract} alone does \emph{not} imply
Assumption~\ref{ass:G}. Throughout the abstract hypotheses below,
$\|h-h'\|$ denotes the quotient metric $\dth{\widehat\theta}(h,h')$ in
the learned-filter application.

\begin{assumption}[Standing hypotheses for root scores]\label{ass:C}
\emph{(C1)~Mean contraction.} $\E\|\Psi(h,\xi)-\Psi(h',\xi)\|\le\rho\|h-h'\|$
  for some $\rho\in(0,1)$.
\emph{(C2$'$)~Root-score stability.}
  $\E_\xi|\rscore(h,\xi)-\rscore(h',\xi)|\le\Lscore\|h-h'\|$
  (Lemma~\ref{lem:Lroot}).
\emph{(C3)~First moments.} $\dzero,\Dpi<\infty$.
\emph{(C4)~Local density.}
  $0<\underline g\le\dens(u)\le\bar g<\infty$ on
  $I=[\sstar-\varepsilon_0,\sstar+\varepsilon_0]$.
\end{assumption}

\begin{theorem}[Local CDF forgetting, root score]\label{thm:C}
Under Assumption~\ref{ass:C}, for every deterministic
$\eta\in(0,\varepsilon_0]$, every initial state $h$, every $t\ge 1$,
and every $u\in I_\eta=[\sstar-(\varepsilon_0-\eta),\sstar+(\varepsilon_0-\eta)]$,
\begin{equation}\label{eq:thmC-pointwise}
  |F_t(u;h)-\Gst(u)|
  \le\frac{\Lscore\rho^{t-1}d(h)}{\eta}+2\bar g\eta,
  \quad d(h):=\E_{H_0^\pi\sim\pi}\|h-H_0^\pi\|.
\end{equation}
For random $\state_0$, optimising in $\eta$ after expectation gives
\begin{equation}\label{eq:thmC-local}
  b_t^{\mathrm{loc}}=\sup_{u\in I_{\eta_\star}}|\Prob(\rscore_t\le u)-\Gst(u)|\le K_{\mathrm{loc}}\rho^{(t-1)/2}
\end{equation}
with $K_{\mathrm{loc}}=2\sqrt{2\bar g\Lscore\dzero}$ and
$\eta_\star=\sqrt{\Lscore\dzero\rho^{t-1}/(2\bar g)}$. The
calibration-buffer convention is $\eta_1$ at index $t_0$,
$I_{\mathrm{buf}}=I_{\eta_1}$ (App.~\ref{app:proofs-C}).
\end{theorem}

\begin{assumption}[Local threshold-autocovariance envelope]\label{ass:G}
There is a summable sequence $(\Gamma_k)_{k\ge 1}$ with
$\sup_{u\in I_{\mathrm{buf}}}\sup_i|\Cov(I_i(u),I_{i+k}(u))|\le\Gamma_k$
where $I_j(u)=\mathbf 1\{\rscore_{t_0+j-1}\le u\}$.
\end{assumption}

\begin{theorem}[Warm-start dependent-quantile concentration]\label{thm:qhat-warm}
Under Assumptions~\ref{ass:C} and~\ref{ass:G}, for every
$0<\varepsilon\le\varepsilon_0-\eta_1$ with
$\underline g\varepsilon>\bar b_{m,t_0}$,
\[
  \Prob(|\hat s_{m,t_0}-\sstar|>\varepsilon)
  \le \frac{2}{m(\underline g\varepsilon-\bar b_{m,t_0})^{2}}
  \Bigl[\tfrac14+2\sum_{k=1}^{m-1}(1-k/m)\Gamma_k\Bigr]=:\delta_{m,t_0}(\varepsilon),
\]
with $\bar b_{m,t_0}\le K_{\mathrm{loc}}\rho^{(t_0-1)/2}/(m(1-\sqrt\rho))$.
\end{theorem}

\noindent\emph{Proof.}
Let $I_j(u):=\mathbf 1\{\rscore_{t_0+j-1}\le u\}$ and
$\hat F_{m,t_0}(u)=m^{-1}\sum_{j=1}^m I_j(u)$. We establish, for
$u\in I_{\mathrm{buf}}$, (a)~a bias bound
$\bar b_{m,t_0}\le K_{\mathrm{loc}}\rho^{(t_0-1)/2}/(m(1-\sqrt\rho))$,
(b)~a variance bound
$\Var(\hat F_{m,t_0}(u))\le\bigl(\tfrac14+2\sum_{k\ge 1}\Gamma_k\bigr)/m$,
and (c)~combine via Chebyshev to conclude.

\paragraph{(a) Bias.}
For each $j$, $\E I_j(u)=\Prob(\rscore_{t_0+j-1}\le u)=F_{t_0+j-1}(u)$
(the marginal CDF at time $t_0+j-1$ from $\state_0$). By
Theorem~\ref{thm:C} on the deterministic buffer $I_{\eta_1}$,
$|F_{t_0+j-1}(u)-\Gst(u)|\le K_{\mathrm{loc}}\rho^{(t_0+j-2)/2}$ for
$u\in I_{\mathrm{buf}}$. Summing the geometric tail,
\begin{equation*}
  \sup_{u\in I_{\mathrm{buf}}}\bigl|\E\hat F_{m,t_0}(u)-\Gst(u)\bigr|
  \;\le\; \frac{K_{\mathrm{loc}}}{m}\sum_{j=1}^{m}\rho^{(t_0+j-2)/2}
  \;\le\; \frac{K_{\mathrm{loc}}\,\rho^{(t_0-1)/2}}{m(1-\sqrt\rho)}
  \;=\; \bar b_{m,t_0}.
\end{equation*}

\paragraph{(b) Variance.}
Decompose
$\Var(\hat F_{m,t_0}(u))=m^{-2}\bigl[\sum_i\Var(I_i(u))+2\sum_{i<j}\Cov(I_i(u),I_j(u))\bigr]$.
The diagonal sum is at most $m/4$ since $I_i(u)\in\{0,1\}$. For the
off-diagonal terms, Assumption~\ref{ass:G} gives
$|\Cov(I_i(u),I_{i+k}(u))|\le\Gamma_k$ uniformly in $i$ for
$u\in I_{\mathrm{buf}}$, so
\begin{equation*}
  \sum_{1\le i<j\le m}|\Cov(I_i(u),I_j(u))|
  \;\le\; m\sum_{k=1}^{m-1}(1-k/m)\Gamma_k
  \;\le\; m\sum_{k\ge 1}\Gamma_k.
\end{equation*}
Dividing by $m^2$,
\begin{equation}\label{eq:cheb-variance}
  \Var(\hat F_{m,t_0}(u))
  \;\le\; \frac{1}{m}\Bigl[\tfrac14+2\!\sum_{k=1}^{m-1}(1-k/m)\Gamma_k\Bigr],
  \qquad u\in I_{\mathrm{buf}}.
\end{equation}
The Markov-coupling derivation of an explicit
$\Gamma_k\le C_\Gamma\kappa_J^{(k-1)/2}$ envelope from filter contraction
is App.~\ref{app:proofs-A4-neural}; the audit of
App.~\ref{app:audits} (Audit~3) verifies the empirical envelope.

\paragraph{(c) Chebyshev at the perturbed thresholds.}
Fix $\varepsilon\in(0,\varepsilon_0-\eta_1]$ with
$\underline g\varepsilon>\bar b_{m,t_0}$. Part~(a) gives, at
$u=\sstar+\varepsilon\in I_{\mathrm{buf}}$,
$\E\hat F_{m,t_0}(\sstar+\varepsilon)\ge\Gst(\sstar+\varepsilon)-\bar b_{m,t_0}\ge 1-\alpha+\underline g\varepsilon-\bar b_{m,t_0}$,
and symmetrically at $u=\sstar-\varepsilon$. If
$\hat s_{m,t_0}>\sstar+\varepsilon$, then $\hat F_{m,t_0}(\sstar+\varepsilon)<1-\alpha$
so $|\hat F_{m,t_0}(\sstar+\varepsilon)-\E\hat F_{m,t_0}(\sstar+\varepsilon)|\ge\underline g\varepsilon-\bar b_{m,t_0}$.
Chebyshev with~\eqref{eq:cheb-variance} at $\underline g\varepsilon-\bar b_{m,t_0}$
and summation of the two tails gives the stated bound on
$\Prob(|\hat s_{m,t_0}-\sstar|>\varepsilon)$.
\hfill$\square$

\subsection{Approximate marginal validity}
\label{app:proofs-approx}

The combined warm-start + local-CDF-forgetting argument gives an
auxiliary approximate-validity statement, used as a sub-step of
Theorem~\ref{thm:learned-validity}.

\begin{proposition}[Approximate marginal validity]\label{prop:approx}
Under Theorems~\ref{thm:C} and~\ref{thm:qhat-warm}, the prediction
set $\Cset_t$ satisfies, for every admissible $\varepsilon$,
\[
  |\Prob\{Y_t\in\Cset_t\}-(1-\alpha)|
  \le \bar g\varepsilon+b_t^{\mathrm{loc}}+\delta_{m,t_0}(\varepsilon),
\]
and balancing $\bar g\varepsilon\asymp C/(m\varepsilon^{2})$ at
$\varepsilon\asymp m^{-1/3}$ gives the rate
$O(m^{-1/3}+\bar b_{m,t_0}+b_t^{\mathrm{loc}})$.
\end{proposition}

\begin{proof}
Let $E_\varepsilon:=\{|\hat s_{m,t_0}-\sstar|\le\varepsilon\}$. On
$E_\varepsilon$, for every $t$,
\begin{equation*}
  \{\rscore_t\le \sstar-\varepsilon\}\subseteq\{\rscore_t\le\hat s_{m,t_0}\}\subseteq\{\rscore_t\le \sstar+\varepsilon\}.
\end{equation*}
Using $\sstar\pm\varepsilon\in I_{\mathrm{buf}}$ and the local test-time
gap $|\Prob(\rscore_t\le u)-\Gst(u)|\le b_t^{\mathrm{loc}}$ for
$u\in I_{\mathrm{buf}}$,
\begin{equation*}
  \Gst(\sstar-\varepsilon)-b_t^{\mathrm{loc}}-\delta_{m,t_0}(\varepsilon)
  \;\le\;\Prob(\rscore_t\le\hat s_{m,t_0})
  \;\le\; \Gst(\sstar+\varepsilon)+b_t^{\mathrm{loc}}+\delta_{m,t_0}(\varepsilon).
\end{equation*}
By (C4), $\Gst(\sstar\pm\varepsilon)=1-\alpha\pm O(\bar g\,\varepsilon)$,
so $|\Prob(\rscore_t\le\hat s_{m,t_0})-(1-\alpha)|\le\bar g\varepsilon+b_t^{\mathrm{loc}}+\delta_{m,t_0}(\varepsilon)$;
because $\rscore_t(y)\le\hat s_{m,t_0}\Leftrightarrow\score_t(y)\le\hat s_{m,t_0}^2=\hat q_{1-\alpha}$,
this is also the coverage gap of $\Cset_t$.
\emph{Order-statistic offset.} Algorithm~\ref{alg:fscp} reports the
$\lceil(m+1)(1-\alpha)\rceil$-th order statistic, equivalently the
empirical $(1-\alpha+1/(m+1))$-quantile. By (C4) the stationary
quantile function is locally Lipschitz with constant $1/\underline g$
on $I_{\mathrm{buf}}$, so this offset shifts $\sstar$ by at most
$1/[\underline g(m+1)]=O(m^{-1})$, which is dominated by
$\bar g\varepsilon$ once $\varepsilon=\Omega(m^{-1})$. Balancing
$\bar g\varepsilon=C/(m\varepsilon^2)$ gives $\varepsilon\asymp m^{-1/3}$
and rate $O(m^{-1/3}+\bar b_{m,t_0}+b_t^{\mathrm{loc}})$.
\end{proof}

\subsection{Oracle stationary corollary of Theorem~\ref{thm:qhat-warm}}
\label{app:proofs-qhat}

Stating Theorem~\ref{thm:qhat-warm} in the stationary-initialisation
limit ($\state_0\sim\pi$, so $\bar b_{m,t_0}=0$ for every $t_0\ge 1$)
recovers the cleaner oracle form:
\begin{equation}\label{eq:qhat-bound}
  \Prob(|\hat s_m-\sstar|>\varepsilon)
  \;\le\;
  \frac{2}{m\,\underline g^{\,2}\varepsilon^{2}}\Bigl[\tfrac14+2\!\sum_{k\ge 1}\Gamma_k\Bigr]
  \;=:\;\delta_m(\varepsilon),
\end{equation}
which is the rate referenced in Proposition~\ref{prop:approx}'s
simplification under stationary initialisation.

\subsection{Sufficient condition for Assumption~\ref{ass:G}: Jacobian contraction}
\label{app:proofs-A4-neural}

The warm-start theorem assumes the threshold-autocovariance envelope
$\Gamma_k$ directly. We give an explicit sufficient condition under
which $\Gamma_k$ inherits a geometric envelope from synchronous
Jacobian-product contraction of the frozen filter; this is the formal
bridge between the contraction hypothesis (C1) and the local-mixing
hypothesis (Assumption~\ref{ass:G}).

\paragraph{Setting.}
Let $(\state_t)\subset\R^d$ be the hidden state of the frozen filter,
$\state_t=\Psi(\state_{t-1},\xi_t)$ with i.i.d.\ driving noise
$(\xi_t)$, and let $\rscore_t=\rscore(\state_{t-1},\xi_t)$ be the root
Mahalanobis score. Write $\state^h_\ell$ for the chain started from
$h$ and driven by the same future innovations as a stationary copy
$\state^{\pi}_\ell$ with $\state_0^\pi\sim\pi$. Let $\Gst$ denote the
stationary CDF of $\rscore$ and $I=[\sstar-\varepsilon_0,\sstar+\varepsilon_0]$
the local interval.

\begin{theorem}[Jacobian contraction implies summable $\Gamma_k$]
\label{thm:A4-mixing}
Suppose, in addition to (C2$'$)--(C4):
\emph{(J1)~Synchronous contraction.} There exist $\Cj<\infty$,
$\kappa_J\in(0,1)$ such that
$\E\|\state^h_\ell-\state^{h'}_\ell\|\le\Cj\kappa_J^\ell\|h-h'\|$ for all
$h,h'$ and $\ell\ge 0$. \emph{(J2)~Finite first moments.}
$\dzero,\Dpi<\infty$ as in (C3). Then for every $u\in I$, every
$i\ge 0$, and every $k\ge 1$,
\begin{equation}\label{eq:A4-Gamma}
  |\Cov(\mathbf 1\{\rscore_i\le u\},\mathbf 1\{\rscore_{i+k}\le u\})|
  \;\le\; \Gamma_k
  \;:=\;
  2\sqrt{2\bar g\,\Lscore\,\Cj}\,\bigl(\sqrt{\Dpi}+\sqrt{\Cj\dzero}\bigr)\,\kappa_J^{(k-1)/2},
\end{equation}
and in particular $\sum_{k\ge 1}\Gamma_k<\infty$.
\end{theorem}

\begin{proof}[Sketch]
Define $d(h):=\E_{\state\sim\pi}\|h-\state\|$. By (C2$'$) and (J1),
$\E|\rscore^h_k-\rscore^\pi_k|\le\Lscore\,\Cj\,\kappa_J^{k-1}\,d(h)$.
For deterministic $\eta>0$ a Markov + density-bound argument identical
to Theorem~\ref{thm:C} gives, for $u\in I_\eta$,
$|\Prob(\rscore^h_k\le u)-\Gst(u)|\le\Lscore\Cj\kappa_J^{k-1}d(h)/\eta+2\bar g\eta$.
Optimising in $\eta$ (deterministic bandwidth, taken \emph{after}
expectation in $h$): $|\Prob(\rscore^h_k\le u)-\Gst(u)|\le 2\sqrt{2\bar g\Lscore\Cj}\,\kappa_J^{(k-1)/2}\,d(h)^{1/2}$.

Now use the tower property and the Markov property at time $i$. Write
$I_i(u):=\mathbf 1\{\rscore_i\le u\}$. Subtract the deterministic
constant $\Gst(u)$ inside the second factor of
$\Cov(I_i,I_{i+k})=\E[(I_i-\E I_i)(\Prob(\rscore_{i+k}\le u\mid\Fcal_i)-\E I_{i+k})]$,
use $|I_i-\E I_i|\le 1$, and obtain
$|\Cov(I_i,I_{i+k})|\le 2\sqrt{2\bar g\Lscore\Cj}\,\kappa_J^{(k-1)/2}\,\E\,d(\state_i)^{1/2}$.
Jensen reduces $\E d(\state_i)^{1/2}\le(\E d(\state_i))^{1/2}$, and the
triangle inequality applied to a synchronously coupled stationary copy
$\state_i^\pi$ yields $\E d(\state_i)\le\Cj\kappa_J^{i}\dzero+\Dpi\le\Cj\dzero+\Dpi$;
finally $\sqrt{\Cj\dzero+\Dpi}\le\sqrt{\Cj\dzero}+\sqrt{\Dpi}$, which
gives~\eqref{eq:A4-Gamma}. Summability is immediate from
$\kappa_J^{1/2}<1$.
\end{proof}

\paragraph{High-probability variant.}
If only a high-probability product-Jacobian bound is available---i.e.,
$\Prob[\sup_z\|D_h(\Psi_{\xi_\ell}\circ\cdots\circ\Psi_{\xi_1})(z)\|_{\mathrm{op}}>\Cj\kappa_J^\ell]\le\delta_\ell$
on a tube of diameter $B$---the same argument adds an
$O(\sqrt{\delta_{k-1}})$ term to $\Gamma_k$; the envelope remains
summable whenever $\sum_k\sqrt{\delta_k}<\infty$ (e.g.\ geometric
$\delta_k\le C_\delta\kappa_\delta^k$). Empirically, the audit of
App.~\ref{app:audits} (Audit~3) confirms a finite
$\sum_{k\le 30}|\hat\Gamma_k|$ and an exponential per-lag decay rate in $[0.83,0.90]$
on both \GraphLGSSM\ and \GraphNeuralSSM, i.e., the operative form of
Assumption~\ref{ass:G} holds at the audited cells.

\subsection{Sufficient conditions for the root-score stability (C2$'$)}
\label{app:Lscore}

The standing assumption (C2$'$) is that the root score
$h\mapsto\rscore(h,\xi)=\|G(h)\res(h,\xi)\|_2$, $G(h)=\hat\Sigma(h)^{-1/2}$,
is Lipschitz in mean. The squared score $\score=\rscore^2$ is not
globally Lipschitz on an unbounded state space because its gradient
grows linearly with the residual; the root-score formulation removes
this blow-up while preserving the calibrated region (§\ref{sec:prelim}).
The lemma below gives a standard sufficient condition.

\begin{lemma}[Root-score mean stability]\label{lem:Lroot}
Suppose, for all coupled states $h,h'$ visited by the filter,
$\|\mu(h,\xi)-\mu(h',\xi)\|\le L_\mu\|h-h'\|$ (mean Lipschitz),
$\|G(h)-G(h')\|_{\mathrm{op}}\le \Lg\|h-h'\|$ (inverse-square-root
covariance Lipschitz), $\|G(h)\|_{\mathrm{op}}\le \Mg$, and
$\E_\xi\|\res(h,\xi)\|\le \bar s$. Then
\begin{equation}\label{eq:Lroot}
  \E_\xi\bigl|\rscore(h,\xi)-\rscore(h',\xi)\bigr|
  \;\le\; (\Lg\bar s + \Mg L_\mu)\,\|h-h'\|,
\end{equation}
i.e.\ \emph{(C2$'$)} holds with $\Lscore=\Lg\bar s+\Mg L_\mu$.
\end{lemma}

\begin{proof}
The reverse triangle inequality on $G(h)\res(h,\xi)$ gives
$|\rscore(h,\xi)-\rscore(h',\xi)|=\bigl|\|G(h)\res(h,\xi)\|-\|G(h')\res(h',\xi)\|\bigr|$,
which is bounded by $\|(G(h)-G(h'))\res(h,\xi)\|+\|G(h')(\mu(h',\xi)-\mu(h,\xi))\|$.
The first term is $\le\Lg\|\res(h,\xi)\|\,\|h-h'\|$; the second is
$\le\Mg L_\mu\|h-h'\|$. Take expectation in $\xi$.
\end{proof}

\paragraph{Implications.}
For the linear-Gaussian \GraphLGSSM\ the predictive covariance is
state-independent, so $\Lg=0$ and $\Lscore\le\Mg L_\mu=\|\hat\Sigma^{-1/2}C\|_{\mathrm{op}}$;
this recovers the linear-Kalman case on an unbounded Gaussian state.
For \GraphNeuralSSM\ the variance floor $\varepsilon_d=10^{-4}$
clamps the diagonal head from below, so $\Mg\le 100$; the log-diagonal
map is smooth, the GRU is gate-Lipschitz, and the low-rank head is
linear in $h_t$, yielding finite $\Lg$ and $L_\mu$ along the
calibration trajectory.

\paragraph{Empirical audit.}
We estimate $\Lscore$ by random directional perturbation
($n_{\mathrm{pert}}=12$ samples per step at three scales) on the
calibration trajectory, recording $|\rscore(h+\delta,\xi)-\rscore(h,\xi)|/\|\delta\|$
percentiles (Audit~1, App.~\ref{app:audits}). Median quotients are
modest ($p50\le 0.36$, $p95\le 1.19$); the neural filter on \metrla\
has a heavier $p99$ tail at high-condition-number trajectory points.
The full-covariance audit in App.~\ref{app:floor-sweep} shows that the
default floor is nonbinding, so we do not attribute this tail to floor
activation.

\subsection{Proof of Theorem~\ref{thm:C} (local forgetting, root score)}
\label{app:proofs-C}

Couple the actual chain started from $\state_0=h$ with a stationary
copy $\state_t^\pi$ initialised from $\state_0^\pi\sim\pi$, using the
same innovation noise $\xi_{1:t}$. By (C1),
$\E[\|\state_{t-1}-\state_{t-1}^\pi\|\mid h,\state_0^\pi]\le\rho^{t-1}\|h-\state_0^\pi\|$;
the root score at time~$t$ is determined by $\state_{t-1}$ together
with $\xi_t$, so by (C2$'$),
\begin{equation}\label{eq:mean-coupling}
  \E\bigl[|\rscore_t-\rscore_t^\pi|\mid h\bigr]
  \;\le\; \Lscore\,\rho^{t-1}\,\E\bigl\|h-\state_0^\pi\bigr\|
  \;=\; \Lscore\,\rho^{t-1}\,d(h).
\end{equation}
For any $u\in\R$ and \emph{deterministic} $\eta>0$,
\begin{equation*}
  \bigl|F_t(u;h)-\Gst(u)\bigr|
  \;\le\; \Prob(|\rscore_t-\rscore_t^\pi|>\eta) + \Prob(|\rscore_t^\pi-u|\le \eta).
\end{equation*}
Markov bounds the first term by $\E|\rscore_t-\rscore_t^\pi|/\eta$; for
$u\in I_\eta=[\sstar-(\varepsilon_0-\eta),\sstar+(\varepsilon_0-\eta)]$,
the event $\{|\rscore_t^\pi-u|\le\eta\}$ has probability at most
$2\bar g\eta$ by (C4) applied to the stationary root-score density on
the $\eta$-neighbourhood of $u$ (which lies inside $I$). Combining
yields~\eqref{eq:thmC-pointwise}.

For random $\state_0$ take expectation over $h$:
\begin{equation*}
  b_t^{\mathrm{loc}}(\eta)
  \;\le\;\frac{\Lscore\dzero\rho^{t-1}}{\eta}+2\bar g\eta,
  \qquad u\in I_\eta.
\end{equation*}
The map $\eta\mapsto a/\eta+2\bar g\eta$ is minimised at
$\eta_\star=\sqrt{a/(2\bar g)}$ with value $2\sqrt{2\bar g a}$;
substituting $a=\Lscore\dzero\rho^{t-1}$ gives
$b_t^{\mathrm{loc}}\le 2\sqrt{2\bar g\Lscore\dzero}\rho^{(t-1)/2}=K_{\mathrm{loc}}\rho^{(t-1)/2}$,
which is~\eqref{eq:thmC-local} with $\dzero$ inside the square root.
The deterministic bandwidth $\eta_\star$ keeps the buffered interval
$I_{\eta_\star}$ non-random.
\hfill$\square$

\paragraph{Remark (state-dependent vs marginal).}
Equation~\eqref{eq:thmC-pointwise} retains the state-dependent
$d(h)$ on $I_\eta$; one obtains a pairwise (two cold-started chains)
statement of the same form by replacing $d(h)$ with $\|h-h'\|$ and
$\Gst$ with the cold-started counterpart. We use only the
marginal form~\eqref{eq:thmC-local}.

\subsection{Linear-Kalman specialisation: initialisation-weighted RMS identity}
\label{app:proofs-L}

The linear-Kalman \GraphLGSSM\ admits an exact RMS identity that
serves as the linear specialisation of Theorem~\ref{thm:obs-contract}
(referenced as the linear-Kalman sanity check in §\ref{sec:theory}). Define the eigenbasis
weights $w_i=(U^{*}\Mzero U)_{ii}/\tr\Mzero$ and the
initialisation-weighted time-local RMS rate
$\rhoMzero(\Acl,t):=(\sum_i w_i|\lambda_i|^{2t})^{1/(2t)}$.

\begin{assumption}\label{ass:L}
\emph{(L1)~Closed-loop normality.} $\Acl$ is normal,
  $\Acl=U\Lambda U^{*}$. \emph{(L2)~Sub-unit spectrum.}
  $\sigma_1(\Acl)<1$. \emph{(L3)~Initialisation second moment.}
  $\Mzero=\E[\Delta_0\Delta_0^{*}]$ has finite trace.
\end{assumption}

\begin{theorem}[Initialisation-weighted RMS identity]\label{thm:weighted-rms}
Under Assumption~\ref{ass:L}, for every $t\ge 1$,
\begin{equation}\label{eq:Ldagger-exact}
  \bigl(\E\|\Delta_t\|^{2}\bigr)^{1/2}
  =\rhoMzero(\Acl,t)^{t}\bigl(\E\|\Delta_0\|^{2}\bigr)^{1/2},
\end{equation}
and for every $1\le t\le T$,
\begin{equation}\label{eq:Ldagger-envelope}
  \bigl(\E\|\Delta_t\|^{2}\bigr)^{1/2}
  \le\rhoMzero(\Acl,T)^{t}\bigl(\E\|\Delta_0\|^{2}\bigr)^{1/2}
  \qquad \text{with equality at }t=T.
\end{equation}
The isotropic trace rate $\rhotr$ is the special case
$\Mzero\propto I$.
\end{theorem}

\noindent\emph{Proof.}
\paragraph{Step 1: exact time-local identity.}
Since $\Acl$ is normal, $\Acl=U\Lambda U^{*}$ with
$\Lambda=\diag(\lambda_1,\dots,\lambda_d)$ and $U$ unitary. Then
$\Acl^{t}=U\Lambda^{t}U^{*}$ and
\begin{equation*}
  \E\|\Delta_t\|^{2}
  \;=\; \tr\!\bigl(\Acl^{t}\Mzero(\Acl^{*})^{t}\bigr)
  \;=\; \tr\!\bigl(\Lambda^{t}(U^{*}\Mzero U)\Lambda^{*t}\bigr)
  \;=\; \sum_{i=1}^{d}|\lambda_i|^{2t}\,(U^{*}\Mzero U)_{ii}.
\end{equation*}
Dividing by $\E\|\Delta_0\|^{2}=\tr\Mzero$ gives the unit-weight form
$\E\|\Delta_t\|^{2}/\E\|\Delta_0\|^{2}=\sum_i w_i|\lambda_i|^{2t}=\rhoMzero(\Acl,t)^{2t}$,
which rearranges to~\eqref{eq:Ldagger-exact}.

\paragraph{Step 2: monotonicity in $t$.}
For any non-negative $\{a_i\ge 0\}$ with weights $\{w_i\ge 0\}$,
$\sum_i w_i=1$, the weighted power mean
$t\mapsto(\sum_i w_i a_i^{2t})^{1/(2t)}$ is non-decreasing in $t$
\citep[Theorem~16]{hardy1952inequalities}. Apply with
$a_i=|\lambda_i|$ and $w_i=(U^{*}\Mzero U)_{ii}/\tr\Mzero$.

\paragraph{Step 3: $\rhoMzero\le\sigma_1$ strictly.}
$\rhoMzero(\Acl,t)^{2t}=\sum_i w_i|\lambda_i|^{2t}\le\sigma_1^{2t}\sum_i w_i=\sigma_1^{2t}$,
with equality iff $w_i$ is supported only on indices where
$|\lambda_i|=\sigma_1$. Generic graph-polynomial spectra have a
non-degenerate sub-leading tail and a non-aligned $\Mzero$, so the
inequality is strict at every finite $t$.

\paragraph{Step 4: horizon-$T$ envelope.}
Combining Steps~1 and~2, for every $1\le t\le T$,
$(\E\|\Delta_t\|^2)^{1/2}=\rhoMzero(\Acl,t)^{t}(\E\|\Delta_0\|^2)^{1/2}\le\rhoMzero(\Acl,T)^{t}(\E\|\Delta_0\|^2)^{1/2}$,
with equality at $t=T$. This is~\eqref{eq:Ldagger-envelope}.

\paragraph{Step 5: isotropic corollary.}
When $\Mzero=\tau^{2}I_d$, $w_i=1/d$ and $\rhoMzero(\Acl,t)=(d^{-1}\sum_i|\lambda_i|^{2t})^{1/(2t)}=\rhotr(\Acl,t)$.
Thus the textbook trace rate is the isotropic special case.

\paragraph{Step 6: transfer to the score process.}
The bounded-Lipschitz step in Theorem~\ref{thm:C}'s proof takes any
upper bound on $\E|\rscore_t-\rscore^\pi_t|$ and produces a
local-CDF rate. Replacing the mean-coupling step of \S\ref{app:proofs-C}
by Jensen
$\E|\rscore_t-\rscore^\pi_t|\le\Lscore\E\|\Delta_{t-1}\|\le\Lscore(\E\|\Delta_{t-1}\|^{2})^{1/2}$
and substituting~\eqref{eq:Ldagger-envelope} at index $t-1$ yields,
for $u\in I_\eta$ with the same deterministic bandwidth as
Theorem~\ref{thm:C},
$\sup_{u\in I_\eta}|\Prob(\rscore_t\le u)-\Gst(u)|\le K_{\mathrm{tr}}\rhoMzero(\Acl,T)^{(t-1)/2}$
for $1\le t\le T+1$, with $K_{\mathrm{tr}}=2\sqrt{2\bar g\Lscore\bar D}$ and
$\bar D=(\E\|\Delta_0\|^{2})^{1/2}$.
\hfill$\square$

\paragraph{Interpreting the $\sigma_1$ / $\rhoMzero$ / $\rhodG$ /
$\rhoscore$ relation.}
On the audited cells of App.~\ref{app:audits} the gap between the
weighted rate and its isotropic specialisation is below $1\%$:
$\sigma_1(\Acl)\approx 0.443$ at $\rho_{\mathrm{scale}}=0.8$,
$\rhotr(\Acl,50),\rhoMzero(\Acl,50)\in[0.43,0.44]$ (the audited METR-LA
cell has $M_0$-anisotropy $\approx 2{,}011\times$ but
$|\rhoMzero-\rhotr|<0.003$). The Bures-Wasserstein audit
$\rhodG\!\approx\!0.46$ tracks $\sigma_1(\Acl)$ within $0.02$ on the
linear sanity check (App.~\ref{app:dG-audit}), confirming that the
direct emitted-law contraction object of
Theorem~\ref{thm:obs-contract} matches the linear analytic rate. The
\emph{score-CDF} 1-Wasserstein rate $\rhoscore\!\approx\!0.15$ on
the same cells is faster because the conformal score collapses
many filter-state degrees of freedom; this gap reflects the
score-Lipschitz step of Theorem~\ref{thm:C}, not a discrepancy with
Theorem~\ref{thm:obs-contract}. Proposition~\ref{prop:nogo} shows
that no \emph{uniform-in-$x_t$} refinement of $\sigma_1$ closes the
state-side gap without paying a prefactor exponential in $T$.

\subsection{Local-linearisation transient bound for learned filters}
\label{app:proofs-A8-locallin}

For the learned \GraphNeuralSSM\ the closed-loop transition has no
closed form. We give a finite-horizon perturbation bound: if the
frozen recurrence is well approximated by its Jacobian product on the
coupling tube, then the nonlinear RMS contraction follows the
corresponding Jacobian-weighted RMS rate up to an explicit
$O(\varepsilon T)$ remainder.

\paragraph{Setting.}
Fix a horizon $T$. Let $\state_t=\Psi_t(\state_{t-1})$ and
$\state_t^{\pi}=\Psi_t(\state_{t-1}^{\pi})$ be two copies of the
frozen filter driven by the same realised input path; here $\Psi_t$
denotes the time-$t$ hidden-state update along that path. Define
$\Delta_t:=\state_t-\state_t^{\pi}$, the Jacobian
$J_t:=D\Psi_t(\state_{t-1}^{\pi})$, and the Jacobian product
$J_{b:a}:=J_b J_{b-1}\cdots J_a$ (with $J_{a-1:a}:=I$). Set
$\Mzero=\E[\Delta_0\Delta_0^{\top}]$ and $d_0:=(\tr\Mzero)^{1/2}$.

\begin{theorem}[Finite-horizon local-linearisation RMS bound]\label{thm:A8-locallin}
Assume \emph{(B1)~Tube containment.} For all $0\le t\le T$, the line
segment $\{\state_{t-1}^{\pi}+s\Delta_{t-1}:0\le s\le 1\}$ lies in a
tube on which the next condition holds.
\emph{(B2)~Local Jacobian variation.} For each $1\le t\le T$ and every
$h$ in that tube,
$\|D\Psi_t(h)-J_t\|_{\mathrm{op}}\le \varepsilon$.
\emph{(B3)~Finite linearised envelope.}
$B_T:=\max_{0\le a\le b\le T}\|J_{b:a+1}\|_{\mathrm{op}}<\infty$.
Define the linearised perturbation $Z_t:=J_{t:1}\Delta_0$ and the
Jacobian-weighted RMS rate
\begin{equation*}
  \rhoMzero(J_{t:1},t)
  \;:=\;
  \Bigl(\frac{\tr(J_{t:1}\Mzero J_{t:1}^{\top})}{\tr\Mzero}\Bigr)^{\!1/(2t)}.
\end{equation*}
Then for every $1\le t\le T$,
$(\E\|Z_t\|^{2})^{1/2}=\rhoMzero(J_{t:1},t)^{t}d_0$, and
\begin{equation}\label{eq:A8-bound}
  \bigl(\E\|\Delta_t\|^{2}\bigr)^{1/2}
  \;\le\;
  \bigl[\rhoMzero(J_{t:1},t)^{t}+\varepsilon B_T^{2}t\,e^{\varepsilon B_T t}\bigr]\,d_0.
\end{equation}
In particular, if $\varepsilon B_T T\le 1$ then
$(\E\|\Delta_t\|^{2})^{1/2}\le[\rhoMzero(J_{t:1},t)^{t}+e\,\varepsilon B_T^{2}T]\,d_0$.
\end{theorem}

\begin{proof}[Sketch]
The exact identity for $Z_t$ follows from
$\E\|Z_t\|^{2}=\tr(J_{t:1}\Mzero J_{t:1}^{\top})$. By the fundamental
theorem of calculus,
$\Delta_t=J_t\Delta_{t-1}+r_t$ with
$r_t=\int_0^1[D\Psi_t(\state_{t-1}^{\pi}+s\Delta_{t-1})-J_t]\Delta_{t-1}\,ds$,
so $\|r_t\|\le\varepsilon\|\Delta_{t-1}\|$ by (B2). Variation of
constants gives
$\Delta_t=Z_t+\sum_{s=1}^{t}J_{t:s+1}r_s$. Let
$R_t:=(\E\|\Delta_t\|^{2})^{1/2}$. Minkowski plus
$(\E\|J_{t:s+1}r_s\|^{2})^{1/2}\le B_T\varepsilon R_{s-1}$ yields
$R_t\le B_T d_0+\varepsilon B_T\sum_{s=0}^{t-1}R_s$; discrete Gronwall
$R_t\le B_T d_0\,e^{\varepsilon B_T t}$. Substituting back into
$\Delta_t-Z_t=\sum_s J_{t:s+1}r_s$ gives the perturbation bound, and
the triangle inequality completes~\eqref{eq:A8-bound}.
\end{proof}

\paragraph{Consequence: CDF forgetting for the learned filter.}
Combining Theorem~\ref{thm:A8-locallin} with the bounded-Lipschitz
step of Theorem~\ref{thm:C}, in the same way as
Theorem~\ref{thm:weighted-rms}, gives
\begin{equation*}
  \sup_{u\in I_\eta}\bigl|\Prob(\rscore_t\le u)-\Gst(u)\bigr|
  \;\le\;
  2\sqrt{2\bar g\Lscore\,d_0\,[\bar\rho_{\Mzero}(T)^{t-1}+\varepsilon B_T^{2}Te^{\varepsilon B_T T}]},
\end{equation*}
where $\bar\rho_{\Mzero}(T):=\max_{1\le t\le T}\rhoMzero(J_{t:1},t)$.
This is the analogue of the linear-filter CDF gap with the
finite-horizon Jacobian product replacing the closed-form $\Acl$.

\subsection{Impossibility of horizon-free worst-case rates}
\label{app:proofs-nogo}

\begin{proposition}[Lower bound]\label{prop:nogo}
Let $A\in\R^{d\times d}$ be normal with $\sigma_1(A)>0$. If for some
$c_1\ge 0$, $\rho_\star>0$, and $T\ge 1$,
$\|A^{t}x\|\le c_1\rho_\star^{t}\|x\|$ holds for all $x$ and all
$1\le t\le T$, then $c_1\ge(\sigma_1(A)/\rho_\star)^{T}$.
\end{proposition}

\begin{proof}
Since $A\in\R^{d\times d}$ is normal,
$\|A^{T}\|_{\mathrm{op}}=\|A\|_{\mathrm{op}}^{T}=\sigma_1(A)^{T}$.
Choose a real unit vector $x_T$ attaining the operator norm of $A^{T}$;
then $\|A^{T}x_T\|=\sigma_1(A)^{T}$. The assumed bound at $t=T$ gives
$\sigma_1(A)^{T}\le c_1\rho_\star^{T}\|x_T\|=c_1\rho_\star^{T}$, hence
$c_1\ge(\sigma_1(A)/\rho_\star)^{T}$.
\end{proof}

Substantively: picking $\rho_\star=0.15$, $\sigma_1=0.44$, $T=50$ forces
$c_1\ge(0.44/0.15)^{50}\approx 2.6\times 10^{23}$ --- uninformative as a
uniform bound on the score-CDF gap. The mean-squared form of
Theorem~\ref{thm:weighted-rms} avoids this because it averages over the
initialisation second moment $\Mzero$ rather than tracking the
worst-aligned direction.

\subsection{Same-block plug-in conformal stability}
\label{app:proofs-A9-plugin}

Static-covariance baselines (\textsc{cgif-joint}, \FactorCGIF,
\GroupCGIF) fit $\hat\Sigma$ on the calibration block and then score
the same block. This is not exactly split-conformal because the score
map is not frozen before calibration. The theorem below bounds the
resulting coverage perturbation by leave-one-out stability and a local
margin condition; the explicit constant for ridge empirical covariance
is in Corollary~\ref{cor:A9-ridge}.

\paragraph{Setting.}
Let $Z_1,\ldots,Z_m,Z_{m+1}$ be exchangeable. Let $\theta\mapsto s(z;\theta)$
be a scalar nonconformity score and $A$ a permutation-invariant
parameter fit. The implemented procedure sets
$\hat\theta=A(Z_1,\ldots,Z_m)$, computes $S_i=s(Z_i;\hat\theta)$, and
emits $\hat q=S_{(k)}$, $k:=\lceil(m+1)(1-\alpha)\rceil$. For each
$1\le i\le m+1$ define the leave-one-out oracle
$\hat\theta^{(-i)}:=A(Z_{1:m+1}\setminus\{Z_i\})$ and
$T_i:=s(Z_i;\hat\theta^{(-i)})$; in particular $T_{m+1}=S_{m+1}$. Let
$q^{\circ}:=T_{(k)}^{1:m}$.

\begin{theorem}[Stability of same-block plug-in conformal scores]\label{thm:plugin}
Assume \emph{(P1) Replacement stability:} with probability at least
$1-\delta_m$, $\max_{1\le i\le m}|S_i-T_i|\le\varepsilon_m$.
\emph{(P2) Local test-score margin:}
$\Prob\{|S_{m+1}-\hat q|\le r\}\le\omega_m(r)$ with $\omega_m(r)\to 0$
as $r\to 0$. \emph{(P3) Oracle no-ties:} $T_1,\ldots,T_{m+1}$ are
almost surely distinct (or use randomised tie-breaking). Then
\begin{equation}\label{eq:A9-cov}
  \bigl|\Prob\{S_{m+1}\le\hat q\}-(1-\alpha)\bigr|
  \;\le\; \frac{1}{m+1}+\delta_m+\omega_m(\varepsilon_m).
\end{equation}
If $\omega_m(r)\le 2\bar f r$ in a neighbourhood of $\hat q$ (e.g.\
under a conditional density bound), then~\eqref{eq:A9-cov} reads
$1/(m+1)+\delta_m+2\bar f\varepsilon_m$.
\end{theorem}

\begin{proof}
Permutation-invariance of $A$ and exchangeability of $Z_{1:m+1}$ make
$(T_1,\ldots,T_{m+1})$ exchangeable; (P3) gives
$\Prob\{T_{m+1}\le q^{\circ}\}=k/(m+1)\in[1-\alpha,1-\alpha+1/(m+1)]$.
On the stability event $\mathcal E:=\{\max_i|S_i-T_i|\le\varepsilon_m\}$,
the sup-norm stability of order statistics gives
$|\hat q-q^{\circ}|\le\varepsilon_m$. Because $T_{m+1}=S_{m+1}$, the
events $\{S_{m+1}\le\hat q\}$ and $\{T_{m+1}\le q^{\circ}\}$ differ
only when $|S_{m+1}-\hat q|\le\varepsilon_m$. Hence
$|\Prob\{S_{m+1}\le\hat q\}-\Prob\{T_{m+1}\le q^{\circ}\}|\le\delta_m+\omega_m(\varepsilon_m)$;
combine with the rank uniformity above.
\end{proof}

\begin{corollary}[Ridge empirical covariance plug-in]\label{cor:A9-ridge}
Let $\hat\Sigma_D=m^{-1}\sum_{r\in D}rr^{\top}+\lambda I_N$ with
$\lambda>0$, and $s(r;\hat\Sigma_D)=r^{\top}\hat\Sigma_D^{-1}r$ on
residuals with $\|r\|_2\le B$. For any one-replacement pair $(D,D')$
of size~$m$ and any evaluation residual $\|r\|\le B$,
$|s(r;\hat\Sigma_D)-s(r;\hat\Sigma_{D'})|\le 2B^{4}/(\lambda^{2}m)$, so
Theorem~\ref{thm:plugin} applies with $\varepsilon_m=O(B^{4}/(\lambda^{2}m))$
and the coverage perturbation is $O(m^{-1})$.
\end{corollary}

\begin{proof}
$\hat\Sigma_D-\hat\Sigma_{D'}=m^{-1}(aa^{\top}-bb^{\top})$ has operator
norm $\le 2B^{2}/m$. The ridge floor gives
$\|\hat\Sigma_D^{-1}\|_{\mathrm{op}},\|\hat\Sigma_{D'}^{-1}\|_{\mathrm{op}}\le\lambda^{-1}$,
so by the resolvent identity
$\|\hat\Sigma_D^{-1}-\hat\Sigma_{D'}^{-1}\|_{\mathrm{op}}\le 2B^{2}/(\lambda^{2}m)$,
and $|r^{\top}(\hat\Sigma_D^{-1}-\hat\Sigma_{D'}^{-1})r|\le 2B^{4}/(\lambda^{2}m)$.
\end{proof}

\paragraph{Empirical audit.}
Refitting $\hat\Sigma$ on the train block (App.~\ref{app:audits},
Audit~6) shifts joint coverage by $\le 2.3$~pp and $\wproj$ by $\le 0.3$
$z$-score units relative to the deployed cal-fit protocol on every
audited cell, supporting the use of plug-in calibration in our
benchmark. Exact split-conformal validity (Theorem~\ref{thm:A}) is
recovered by fitting $\hat\Sigma$ on a disjoint pre-calibration block.

\subsection{Asymptotic SNR-tightened form (isotropic Kalman) and the surrogate $\rhosnr$}
\label{app:proofs-snr}

The rate $\sigma_1(\Acl)$ that appears in Theorem~\ref{thm:weighted-rms} and
throughout the paper is computed by iterating the discrete algebraic
Riccati equation (DARE) to the steady-state covariance
$P_{\mathrm{ss}}$, then extracting $\sigma_1$ from $\Acl=F(I-KC)$. For
large $N$ the DARE iteration is $\bigO(N^{3})$ per step and dominates
hyperparameter tuning cost. We derive a closed-form expression under
isotropy and a differentiable $\bigO(N)$ surrogate $\rhosnr$ that
coincides with $\sigma_1$ at isotropy and agrees with it to first
order otherwise.

\paragraph{Closed-form steady state under isotropy.}
Assume $P_{\mathrm{ss}}^{-}=p^{-}I_d$ (steady-state prior state variance),
$R=\nu I_d$ (observation noise), $Q=qI_d$ (innovation noise), and
$F=\rho_{\mathrm{scale}}\,U$ with $U$ orthogonal. Writing $p^{+}$ for
the steady-state posterior variance, the standard Riccati recursion in
prior/posterior form is
\begin{equation*}
  p^{-}=f^{2}p^{+}+q,
  \qquad p^{+}=\nu p^{-}/(p^{-}+\nu),
\end{equation*}
where $f:=\rho_{\mathrm{scale}}$ collects the scalar coefficient.
Substituting yields the quadratic
$(p^{-})^{2}+[\nu(1-f^{2})-q]p^{-}-q\nu=0$ with unique positive root
\begin{equation*}
  p^{-}=\frac{q-\nu(1-f^{2})+\sqrt{(q-\nu(1-f^{2}))^{2}+4q\nu}}{2}.
\end{equation*}
The Kalman gain is $K=(p^{-}/(p^{-}+\nu))\,I$ and the closed-loop
transition is
\begin{equation*}
  \Acl=F(I-KC)=\rho_{\mathrm{scale}}\,U\cdot\frac{\nu}{p^{-}+\nu}\,I
  =\frac{\rho_{\mathrm{scale}}}{1+p^{-}/\nu}\,U.
\end{equation*}
$\Acl$ is normal (scalar multiple of orthogonal), so
\begin{equation}\label{eq:sigma1-snr}
  \sigma_1(\Acl)
  \;=\;\frac{\rho_{\mathrm{scale}}}{1+\SNR_{\mathrm{bulk}}},
  \qquad
  \SNR_{\mathrm{bulk}}:=\frac{p^{-}}{\nu}=\frac{\tr P_{\mathrm{ss}}^{-}}{\tr R}.
\end{equation}
The predictive observation variance is $p^{-}+\nu$, i.e.\
$\Sigma_{Y}^{\mathrm{pred}}=CP_{\mathrm{ss}}^{-}C^{\top}+R$ in
matrix form; this is the natural comparator for empirical residual
covariance. For graph modes, replace $f$ by the per-mode coefficient
$f_i$; the mode-wise prior variance solves the same quadratic in
$(f_i,q,\nu)$, and the mode-wise closed-loop coefficients
$a_{\mathrm{cl},i}=f_i\nu/(p_i^{-}+\nu)$ are the eigenvalues that feed
Theorem~\ref{thm:weighted-rms}.

\paragraph{The differentiable surrogate $\rhosnr$.}
Equation~\eqref{eq:sigma1-snr} depends on $P_{\mathrm{ss}}$, which
itself requires a DARE solve. We replace $P_{\mathrm{ss}}$ by its
isotropic projection on the right-hand side: define
\begin{equation}\label{eq:rhosnr-def}
  \rhosnr
  \;:=\;
  \frac{\rho_{\mathrm{scale}}}{1+\tr(P_{\mathrm{ss}})/\tr(R)}.
\end{equation}
Under exact isotropy ($P_{\mathrm{ss}}=pI$, $R=rI$) we have
$\rhosnr=\sigma_1(\Acl)$. For non-isotropic $(P_{\mathrm{ss}},R)$ a
first-order perturbation argument (linearising the DARE around its
isotropic solution) gives $\rhosnr=\sigma_1(\Acl)+\bigO(\varepsilon)$,
where $\varepsilon=\|P_{\mathrm{ss}}-p_\star I\|_{\mathrm{op}}/p_\star$
measures departure from isotropy.

\paragraph{Why the surrogate is useful.}
Three properties:
(i)~$\rhosnr$ is computable in $\bigO(N)$ time (one trace ratio) vs.\
$\bigO(N^{3})$ for the DARE iteration, which matters for
hyperparameter gradients over $(\rho_{\mathrm{scale}},\sigma_Q,\sigma_R)$
when $N\in\{20,50,207,\ldots\}$.
(ii)~It is differentiable in each of these hyperparameters and monotone
decreasing in $\SNR_{\mathrm{bulk}}$, so gradient-based tuning lifts
SNR (and contracts $\Sigma_t$) in the expected direction.
(iii)~It agrees with the Theorem~\ref{thm:weighted-rms} asymptotic rate at
isotropy, which is the regime we use for the reported \GraphLGSSM\
runs, so the surrogate and the theory speak the same language.

\paragraph{Empirical tightness.}
Table~\ref{tab:rho-full} shows $\rhosnr\in[0.436,0.443]$ across the
$10$ real-data evaluation cells, matching
$\sigma_1(\Acl)\in[0.433,0.441]$ to within $0.002$--$0.004$. The
first-order perturbation error is therefore below $1\%$ on every cell
we tested.

\section{Filter specifications, training, and pseudocode}
\label{app:filters}

\subsection{\GraphLGSSM\ (linear-Kalman baseline)}
\label{app:graph-lgssm}

\paragraph{State-space model.}
\GraphLGSSM\ is a linear-Gaussian state-space model with latent
dimension $d=N$,
\begin{equation}\label{eq:lgssm}
  \state_{t+1} \;=\; F\,\state_{t} + \xi_{t},\qquad
  Y_{t} \;=\; C\,\state_{t} + \eta_{t},\qquad
  \xi_{t}\sim\Normal(0,Q),\ \eta_{t}\sim\Normal(0,R),
\end{equation}
with the following choices, fixed before any real-data evaluation:
\begin{itemize}[leftmargin=*,topsep=1pt,itemsep=1pt]
  \item \textbf{Transition} $F=\rho_{\mathrm{scale}}\,\widetilde A_{\mathrm{sym}}/\lambda_{\max}(\widetilde A_{\mathrm{sym}})$,
    where $\widetilde A_{\mathrm{sym}}=\widetilde D^{-1/2}(A+I_{N})\widetilde D^{-1/2}$
    is the self-loop-augmented symmetric-normalised graph shift defined
    in App.~\ref{app:adjacency}. Scaling by
    $\lambda_{\max}(\widetilde A_{\mathrm{sym}})^{-1}$ ensures
    $\|F\|_{\mathrm{op}}=\rho_{\mathrm{scale}}$;
    we use $\rho_{\mathrm{scale}}=0.8$ throughout, which keeps the
    open-loop dynamics sub-unit (matching the contraction hypothesis (C1))
    and was found to give well-conditioned Riccati recursions on every
    real dataset.
  \item \textbf{Observation matrix} $C=I_{N}$ (sensors observe the latent
    state directly, i.e.\ the latent dimension equals the sensor dimension).
  \item \textbf{Innovation noise} $Q=\sigma_{Q}^{2}I_{N}$ with $\sigma_{Q}=1$.
  \item \textbf{Observation noise} $R=\sigma_{R}^{2}I_{N}$ with $\sigma_{R}=1$.
\end{itemize}

\paragraph{Riccati recursion.}
The one-step predictive covariance follows the standard discrete-time
Kalman recursion
\begin{align*}
  P_{t|t-1} &= F\,P_{t-1|t-1}\,F^{\top} + Q,\\
  K_{t}     &= P_{t|t-1}\,C^{\top}\bigl(C P_{t|t-1} C^{\top} + R\bigr)^{-1},\\
  P_{t|t}   &= (I-K_{t}C)\,P_{t|t-1}.
\end{align*}
We initialise $P_{0|0}=Q$ and iterate to a relative tolerance
$\|P_{t|t}-P_{t-1|t-1}\|_{F}/\|P_{t|t}\|_{F}\le 10^{-9}$; convergence
requires $\le 200$ iterations on every dataset we tested. The
steady-state covariance $P_{\mathrm{ss}}$ and gain $K$ are then frozen
and used for the calibration and test passes; the closed-loop transition
$\Acl=F(I-KC)$ is computed once for the contraction-rate diagnostics of
Table~\ref{tab:rho-full}.

\paragraph{Hyperparameter selection.}
The Riccati tolerance is fixed ($10^{-9}$). Across the paper we
evaluate \GraphLGSSM\ in two modes: an \emph{untuned} variant with
the theory-motivated defaults $(\rho_{\mathrm{scale}}, \sigma_Q,
\sigma_R)=(0.8, 1, 1)$ --- used for all contraction-rate diagnostics
in Table~\ref{tab:rho-full} --- and a \emph{validation-NLL-tuned}
variant with $(\rho_{\mathrm{scale}},\mathrm{obs\_noise\_frac})$
picked on a 20\% validation tail of the training block
(App.~\ref{app:hyperparams}). The tuned variant is the one reported
in the \CGIFFiltered\ rows of Tables~\ref{tab:hero}--\ref{tab:scale}
and Table~\ref{tab:hero-full}. The closed-form sensitivity in
App.~\ref{app:proofs-snr} shows that the asymptotic rate is monotone
in $\rho_{\mathrm{scale}}/(1+\SNR_{\mathrm{bulk}})$, so the
qualitative behaviour is preserved across reasonable choices of the
untuned defaults.

\subsection{\GraphNeuralSSM\ (learned GCN--GRU)}
\label{app:hyperparams}
\label{app:graph-neuralssm}

We instantiate the framework with a graph-convolutional GRU cell
(diffusion-convolutional recurrent style,
\citealp{li2018diffusion,cho2014learning}) feeding a
diagonal-plus-low-rank covariance head. Graph structure enters the
recurrence directly: every node carries its own $d$-dimensional hidden
state, and neighbour aggregation is folded into the GRU gates through a
fixed graph shift. The architecture is intentionally minimal --- the
trainable weights are shared across nodes, so the parameter count is
independent of $N$ --- so that the sharpness gain we report is not
confounded with model capacity. Table~\ref{tab:hyperparams}
consolidates every hyperparameter; all values were selected once on a
synthetic GRSSF pilot and then frozen for all reported real-data runs.

\paragraph{Graph shift.}
Let $S=\widetilde A_{\mathrm{sym}}=(\widetilde D)^{-1/2}(A+I_{N})(\widetilde D)^{-1/2}$
be the symmetric-normalised adjacency with self-loops, where
$\widetilde D=\diag\bigl(\sum_{j}(A+I)_{ij}\bigr)$ is the degree matrix
of the augmented graph \citep{kipf2017semi}. The shift is fixed (not
learned), satisfies $\|S\|_{\mathrm{op}}\le 1$, and enters the cell as
the neighbour-aggregation operator
\begin{equation}\label{eq:gcn-layer}
  (S\,H)_{i} \;=\; \textstyle\sum_{j} S_{ij}\,h_{j},\qquad i=1,\dots,N,
\end{equation}
acting on any per-node array $H=(h_1,\dots,h_N)$.

\paragraph{Per-node recurrent update.}
Each node $i$ carries a hidden state $h_{t,i}\in\R^{d}$ with $d=32$, so
the full filter state $H_{t}=(h_{t,1},\dots,h_{t,N})\in\R^{N\times d}$
has dimension $dN$ ($640$ at $N{=}20$, $1600$ at $N{=}50$). With the
per-node scalar input $y_{t,i}\in\R$ ($f_{\mathrm{in}}=1$) the
graph-convolutional GRU cell applies node-shared weights
$W_{z},W_{r},W_{n}\in\R^{d\times(2d+f_{\mathrm{in}})}$:
\begin{align}\label{eq:gru-update}
  z_{t} &= \sigma\!\bigl(W_{z}\,[\,H_{t-1},\,S H_{t-1},\,Y_{t}\,]\bigr),\\
  r_{t} &= \sigma\!\bigl(W_{r}\,[\,H_{t-1},\,S H_{t-1},\,Y_{t}\,]\bigr),\\
  \widetilde H_{t} &= \tanh\!\bigl(W_{n}\,[\,r_{t}\odot H_{t-1},\,S(r_{t}\odot H_{t-1}),\,Y_{t}\,]\bigr),\\
  H_{t} &= (1-z_{t})\odot H_{t-1} + z_{t}\odot \widetilde H_{t},
\end{align}
where $[\cdot,\cdot,\cdot]$ concatenates along the feature axis so that
each node sees $[h_{i},(SH)_{i},y_{i}]\in\R^{2d+f_{\mathrm{in}}}$,
$\sigma$ is the logistic sigmoid, and $\odot$ is the Hadamard product.
Graph mixing thus acts inside both the gates and the candidate state;
there is no mean-pooling and no global summary vector.

\paragraph{Predictive head.}
Node-shared linear heads map each node's state $h_{t,i}\in\R^{d}$ to its
mean, log-variance, and low-rank factor row:
\begin{align}\label{eq:pred-head}
  \hat Y_{t|t-1,i} &= w_{\mu}^{\top}h_{t,i} + b_{\mu}, & w_{\mu}&\in\R^{d},\\
  d_{t,i}          &= \max\!\bigl(\exp(w_{d}^{\top}h_{t,i}+b_{d}),\,\varepsilon_{d}\bigr),
                 & w_{d}&\in\R^{d},\\
  (L_{t})_{i,:} &= W_{L}\,h_{t,i} + b_{L}, & W_{L}&\in\R^{r\times d},
\end{align}
with $\varepsilon_{d}=10^{-4}$ a strict floor clamped onto the diagonal
head and $L_{t}\in\R^{N\times r}$ formed by stacking the rows
$(L_{t})_{i,:}$. The predictive covariance is the
diagonal-plus-low-rank
$\hat\Sigma_{t|t-1}=\diag(d_{t})+L_{t}L_{t}^{\top}\succ 0$, so the
matrix-determinant lemma and Woodbury identity give
$\bigO(Nr^{2}+r^{3})$ inversion and log-determinant computations
during training and inference (linear in $N$ for fixed $r=4$). We fix
$r=4$ (rank-sensitivity sweep in App.~\ref{app:rank}).

\paragraph{Initialisation.}
The input-facing blocks of the cell weights ($W_{z},W_{r},W_{n}$) and
the head weights ($w_{\mu},w_{d},W_{L}$) use Glorot uniform
initialisation \citep{glorot2010understanding}; the recurrent
(state-facing) blocks use orthogonal initialisation \citep{saxe2014exact}
to avoid early-training spectral collapse. All biases are zero. The
per-node state is reset to $H_{0}=\mathbf 0_{N\times d}$ at the start of
training and at the start of every calibration / test pass.

\paragraph{Training objective.}
Parameters $\theta=(W_{z},W_{r},W_{n},w_{\mu},b_{\mu},w_{d},b_{d},W_{L},b_{L})$
are fit by minimising the negative Gaussian predictive log-likelihood
\begin{equation}\label{eq:training-loss}
  \mathcal L(\theta)
  \;=\;\sum_{t=1}^{T_{\mathrm{tr}}}
    \Bigl[\tfrac{1}{2}(Y_{t}-\hat Y_{t|t-1})^{\top}\hat\Sigma_{t|t-1}^{-1}(Y_{t}-\hat Y_{t|t-1})
       + \tfrac{1}{2}\log\det\hat\Sigma_{t|t-1}\Bigr],
\end{equation}
via truncated backpropagation through time with a BPTT window of $64$
steps and a $10\%$ loss burn-in (the first
$\max(0.1\,T_{\mathrm{tr}},5)$ steps of each epoch are excluded from the
loss). We add an Adam weight decay of $10^{-4}$ as the only explicit
regulariser; the variance floor $\varepsilon_{d}$ lower-bounds
$\hat\Sigma_{t|t-1}$ and the orthogonal initialisation keeps the
recurrence stable for the $60$ epochs we train.

\paragraph{Optimiser and schedule.}
Adam \citep{kingma2015adam} with
$(\beta_{1},\beta_{2},\eta)=(0.9,0.999,10^{-3})$, weight decay
$10^{-4}$, and $\ell_{2}$-norm gradient clipping at $5.0$. Training runs
over the contiguous training sequence with gradients detached every $64$
steps (TBPTT). We train for $60$ epochs without learning-rate scheduling
and without early stopping; on every dataset we report, the validation
log-likelihood plateaus by around epoch~$40$. We did not validation-tune
any hyperparameter; Table~\ref{tab:hyperparams} lists the full set.

\begin{table}[h]
\caption{\textbf{\GraphNeuralSSM\ hyperparameters.} Fixed across all
datasets and seeds; no validation-set sweep was run.}
\label{tab:hyperparams}
\centering\small
\begin{tabular}{l l l}
\toprule
Group & Quantity & Value \\
\midrule
Graph        & per-node input width $f_{\mathrm{in}}$ & $1$ (per-sensor scalar) \\
             & graph shift $S$                      & sym.-normalised adjacency, $\|S\|_{\mathrm{op}}\le 1$ \\
\addlinespace[2pt]
GCN-GRU cell & per-node hidden width $d$            & $32$ \\
             & full state dimension                 & $dN$ ($640$/$1600$ at $N{=}20$/$50$) \\
             & weights                              & node-shared $W_{z},W_{r},W_{n}$ \\
\addlinespace[2pt]
Head         & covariance rank $r$                  & $4$ \\
             & variance floor $\varepsilon_{d}$     & $10^{-4}$ \\
\addlinespace[2pt]
Optimisation & optimiser                            & Adam $(\beta_{1},\beta_{2})=(0.9,0.999)$ \\
             & learning rate $\eta$                 & $10^{-3}$ \\
             & weight decay                         & $10^{-4}$ \\
             & BPTT window                          & $64$ \\
             & loss burn-in                         & $10\%$ \\
             & epochs                               & $60$ \\
             & gradient clipping                    & $\|\cdot\|_{2}\le 5.0$ \\
\addlinespace[2pt]
Inference    & warm-up $t_{0}$                      & $50$ steps \\
             & target level $\alpha$                & $0.1$ \\
\bottomrule
\end{tabular}
\end{table}

\paragraph{Linear-Kalman hyperparameter tuning.}
The \CGIFFiltered\ rows in Tables~\ref{tab:hero}--\ref{tab:scale} use a
validation-NLL grid search over
$\rho_{\mathrm{scale}}\in\{0.5,0.7,0.85,0.95\}\times
\mathrm{obs\_noise\_frac}\in\{0.03,0.1,0.3,1.0\}$ with the graph
polynomial order fixed at $K=1$ and a $20\%$ validation tail reserved
from the training block (warmup $t_0=50$). Per-step Gaussian NLL is
the selection metric; the chosen pair is re-fit on the full training
block before calibration. Tuning adds $\sim\!30\,\mathrm{s}$ per seed
and narrows the filter's $\wproj$ against the untuned default
$(\rho_{\mathrm{scale}},\,\mathrm{obs\_noise\_frac})=(0.8,0.1)$ by
$9\%$ on \metrla-$20$ and $13\%$ on \pemsbay-$50$; the structural
Riccati-tail gap on real traffic remains
(App.~\ref{app:honest-neg}).

\paragraph{Parameter footprint.}
Because the cell and heads are shared across nodes, the trainable count
is \emph{independent of $N$}: with $d=32$, $r=4$, $f_{\mathrm{in}}=1$ the
three cell maps contribute $3\bigl[(2d+f_{\mathrm{in}})d+d\bigr]=6336$
parameters and the three heads contribute $(2+r)d+(2+r)=198$, for
$\approx\!6.5\!\times\!10^{3}$ trainable parameters on every graph. $N$
enters only through the fixed graph shift $S$ and the state size $dN$.
Full training ($60$ epochs, $\sim\!45\,\mathrm{s}$ per epoch) takes
$44$--$48$ minutes per seed on a single RTX\,4090
(Table~\ref{tab:timing}).

\paragraph{Inference protocol.}
Algorithm~\ref{alg:train} is the offline training loop;
Algorithm~\ref{alg:fscp-detail} is the operational form of
Algorithm~\ref{alg:fscp} that consumes the resulting frozen
$\theta$ on the calibration and test sequences.
At calibration and test time the filter is run autoregressively
on the observed $\{Y_{t}\}$, with the previously emitted
$\hat Y_{t|t-1}$ replacing any unobserved input only in the streaming
forecast horizon (App.~\ref{app:tf} discusses the
training/inference-mode distinction).

\begin{algorithm}[h]
\caption{\GNF\ training loop}
\label{alg:train}
\begin{algorithmic}[1]
\Require training sequence $\{Y_{t}\}_{t=1}^{T_{\mathrm{tr}}}$, graph shift $S=\widetilde A_{\mathrm{sym}}$, epochs $E$, TBPTT window $64$, loss burn-in ($10\%$)
\State initialise $\theta$ as in the ``Initialisation'' paragraph
\For{$e=1,\ldots,E$}
\State set $H\gets\mathbf 0_{N\times d}$, $\mathcal L\gets 0$, window counter $c\gets 0$
\For{$t=1,\ldots,T_{\mathrm{tr}}$}
  \State emit $(\hat Y_{t|t-1},d_{t},L_{t})=g_{\theta}(H)$; $\hat\Sigma_{t}=\diag(d_{t})+L_{t}L_{t}^{\top}$
  \State advance the per-node state $H\gets\mathrm{GCNGRU}(Y_{t},H,S)$ via \eqref{eq:gru-update}
  \If{$t>$ burn-in}
    \State $\mathcal L\mathrel{+}=\tfrac{1}{2}(Y_{t}-\hat Y_{t|t-1})^{\top}\hat\Sigma_{t}^{-1}(Y_{t}-\hat Y_{t|t-1})+\tfrac{1}{2}\log\det\hat\Sigma_{t}$;\; $c\gets c+1$
  \EndIf
  \If{$t\bmod 64=0$ \textbf{or} $t=T_{\mathrm{tr}}$}
    \State Adam step on $\mathcal L/\max(c,1)$ with $\|\cdot\|_{2}$-clipping at~$5.0$; then detach $H$ (TBPTT) and reset $\mathcal L\gets 0$, $c\gets 0$
  \EndIf
\EndFor
\EndFor
\end{algorithmic}
\end{algorithm}

\begin{algorithm}[h]
\caption{Filtered split-conformal inference (operational form of Algorithm~\ref{alg:fscp})}
\label{alg:fscp-detail}
\begin{algorithmic}[1]
\Require trained filter $\phi=(\theta,h_{0})$, calibration sequence $\{Y_{t}\}_{t=1}^{n}$, level $\alpha$, warm-up $t_{0}$
\Statex \emph{Calibration pass.} The filter is one-step-ahead, so
both calibration and test condition on the observed $Y_{t-1}$ (not a
rollout $\hat Y$); App.~\ref{app:tf} discusses the
teacher-forcing-vs-inference distinction.
\State $h\gets h_{0}$
\For{$t=1,\ldots,n$}
  \State $(\hat Y_{t|t-1},\hat\Sigma_{t|t-1},h_{\mathrm{new}})\gets\phi(Y_{t-1},h)$ \Comment{$Y_{0}\equiv\mathbf 0$ at $t=1$}
  \State $\score_{t}\gets(Y_{t}-\hat Y_{t|t-1})^{\top}\hat\Sigma_{t|t-1}^{-1}(Y_{t}-\hat Y_{t|t-1})$
  \State $h\gets h_{\mathrm{new}}$
\EndFor
\State $\hat q_{1-\alpha}\gets$ empirical $\lceil(n-t_{0}+2)(1-\alpha)\rceil$-th order statistic of $\{\score_{t}\}_{t=t_{0}}^{n}$
\Statex \emph{Test pass.}
\For{$t=n+1,n+2,\ldots$}
  \State $(\hat Y_{t|t-1},\hat\Sigma_{t|t-1},h_{\mathrm{new}})\gets\phi(Y_{t-1},h)$
  \State emit $\Cset_{t}\gets\bigl\{y\in\R^{N}:(y-\hat Y_{t|t-1})^{\top}\hat\Sigma_{t|t-1}^{-1}(y-\hat Y_{t|t-1})\le\hat q_{1-\alpha}\bigr\}$
  \State $h\gets h_{\mathrm{new}}$
\EndFor
\end{algorithmic}
\end{algorithm}

\subsection{Teacher forcing and conformal validity}
\label{app:tf}

\GraphNeuralSSM\ is trained with teacher-forced inputs: within a
training window the input for predicting $Y_{u+1}$ is the observed
$Y_u$, $H_u=\mathrm{GCNGRU}(Y_u,H_{u-1};S)$,
and the negative log-likelihood compares
$(\hat Y_{u+1|u},\hat\Sigma_{u+1|u})$ to $Y_{u+1}$. At calibration and
test time the filter is deployed one-step-ahead in the same mode
(Algorithm~\ref{alg:fscp-detail}): before $Y_t$ is observed, $h_t$
has been advanced on $Y_{t-1}$ and the predictive moments
$(\hat Y_{t|t-1},\hat\Sigma_{t|t-1})$ are emitted; the score is
evaluated only after $Y_t$ is observed, so no calibration or test
score uses $Y_t$ to predict itself. Two distributions therefore
coexist: the \emph{training-time} score distribution induced by
teacher-forced windows used to arrive at the parameters $\theta$, and
the \emph{inference-time} score distribution driven by the frozen
$\theta$ on the observed sequence $Y_{1:n+T}$. Conformal validity
concerns only the latter.

\paragraph{Claim.}
The validity analysis of Algorithm~\ref{alg:fscp} depends only on the
inference-mode score process generated by the frozen filter $\phi$:
Proposition~\ref{prop:approx} bounds its coverage gap via
Theorems~\ref{thm:C}--\ref{thm:qhat-warm}, and exact split-conformal
validity additionally requires the standardised-innovation condition
of Proposition~\ref{prop:oracle}. The training procedure enters only
through the constants ($\rho$, $\Lscore$, $\bar g$) of those
statements, because the parameters are frozen at the start of the
calibration pass.

\paragraph{Argument.}
Because Algorithm~\ref{alg:fscp-detail} conditions both calibration
and test passes on the observed $Y_{t-1}$, the calibration scores
$(\rscore_{t_0},\dots,\rscore_n)$ and the test score $\rscore_{n+k}$
are drawn from the \emph{same} stochastic process generated by the
frozen $\phi$ applied in the same mode to a contiguous observation
sequence. Propositions~\ref{prop:approx} and~\ref{prop:oracle} then
apply as stated. Mis-specified teacher-forced training would enter
only through a larger root-score Lipschitz constant $\Lscore$
(Lemma~\ref{lem:Lroot}) or $K_{\mathrm{loc}}$
(Theorem~\ref{thm:C}), not as a broken validity guarantee.
App.~\ref{app:audits} (Audit~1) reports the empirical root-score
finite-difference quotients; the warm-up $t_0\!\ge\!50$ of
App.~\ref{app:warmup} drives the initialisation bias below double
precision on every cell.

\subsection{\textsc{NeuralSSM} ($\DiagGRU$) and the rank-$0$ ablation
($\GCNrankzero$)}
\label{app:neuralssm}

We use two distinct diagonal-covariance ablations, with precise
labels.

\paragraph{$\DiagGRU$ (NeuralSSM): graph removed \emph{and} diagonal
$\Sigma_t$.}
$\DiagGRU$ is a standard \emph{global} GRU over the stacked sensor
vector $Y_t\in\R^{N}$ (a single recurrent state of width $h=32$ rather
than the per-node states of \eqref{eq:gru-update}) with a diagonal
Gaussian head, $\hat\Sigma_{t|t-1}=\diag(d_t)$; the graph is not used
($S$ is ignored). It shares the optimiser, rank ($r=0$), floor, and
schedule of Table~\ref{tab:hyperparams}. This is the \emph{NeuralSSM}
row of Table~\ref{tab:hero}, with $\wproj$ $4.71$ at joint coverage
$0.932$ on \metrla-$20$.

\paragraph{$\GCNrankzero$: graph retained, diagonal $\Sigma_t$.}
$\GCNrankzero$ \emph{keeps} the per-node graph cell of
\eqref{eq:gru-update} and only sets $r=0$ in the head. This is the
rank sweep's $r{=}0$ point in App.~\ref{app:rank}
($\wproj=5.32$ on \metrla-$20$); it isolates the covariance head
\emph{at fixed graph mixing}, and should not be confused with the
(architecturally distinct) global-GRU $\DiagGRU$ row of
Table~\ref{tab:hero}.

\paragraph{Reading the two rows.}
$\DiagGRU$ removes both graph structure and the low-rank head and also
changes the recurrence to a global GRU, so its gap to the full filter
($\GraphNeuralSSM$) is not a clean single-mechanism ablation.
$\GCNrankzero$ isolates the rank-$r$ factor at fixed per-node graph
recurrence: the rank sweep $\GCNrankzero\to r{=}4$ shows the low-rank
head dominates the sharpness gain ($5.32\!\to\!3.04$). The clean
graph-removal ablation \emph{at fixed per-node architecture} is the
LR-GRU covariance control of App.~\ref{app:deepvar}; the
$\GCNrankzero\to\DiagGRU$ comparison is only suggestive.

\subsection{Adaptive-conformal-inference wrap}
\label{app:aci}

Adaptive Conformal Inference \citep{gibbs2021adaptive} adjusts the
nominal level $\alpha_{t}$ online to track the realised miscoverage
rate. Wrapping \GraphNeuralSSM\ proceeds as follows: initialise
$\alpha_{1}=\alpha=0.1$; at each test step $t$ emit the prediction
set $\Cset_{\alpha_{t}}(x_{t})$ via Algorithm~\ref{alg:fscp-detail}
with the recomputed quantile $\hat q_{1-\alpha_{t}}$ (using the same
frozen calibration scores), observe $Y_{t}$, and update
\begin{equation*}
  \alpha_{t+1}\;\gets\;\alpha_{t}+\gamma\bigl(\alpha-\Ind\{Y_{t}\notin\Cset_{\alpha_{t}}(x_{t})\}\bigr).
\end{equation*}
We fix $\gamma=0.005$ for every reported cell and do not tune it
per dataset. Recomputing $\hat q_{1-\alpha_{t}}$
from the frozen calibration scores is $\bigO(m\log m)$ per step,
which adds negligible test-time cost (Table~\ref{tab:timing}).
The long-run validity guarantee of \citet[Thm.~1]{gibbs2021adaptive}
applies as stated to the wrapped scores.

\section{Baseline implementations}
\label{app:baselines}

This section gives implementation details for every baseline that
appears in Tables~\ref{tab:hero}--\ref{tab:hero-full}. Across all
baselines we use $\alpha=0.1$, 10 seeds per cell, and the same
chronological 40/40/20 train/calib/test split
(App.~\ref{app:datasets}). All filter baselines share the warm-up
$t_{0}=50$ used by the main method.

\subsection{Static-covariance methods}

\paragraph{Mean-predictor protocol.}
The non-filter covariance baselines (\textsc{cgif-joint},
\FactorCGIF, \GroupCGIF, \SpectralCGIF, ACI/AgACI variants, EWMA,
rolling, time-of-day, $k$-NN local ellipsoid) share the same
backbone mean predictor $\hat Y_{t|t-1}$ within a given table
(default \textsc{gpvar} of §\ref{sec:experiments}). Filter rows
(\GraphNeuralSSM, \textsc{NeuralSSM}, \CGIFFiltered) use their own
mean mechanisms, as does \textsc{CopulaCPTS} and \textsc{MultiDimSPCI};
Table~\ref{tab:method-map} maps each method to its covariance
mechanism and to its mean source.  Backbone-swap experiments
(Tables~\ref{tab:multibackbone-mod}, \ref{tab:multibackbone-full},
\ref{tab:multibackbone-nontraffic}) should be interpreted as a
mean-channel stress test rather than as a perfect same-mean
decomposition of mean and covariance effects: all covariance heads
(including the learned \GraphNeuralSSM\ head) stay fixed while the
shared mean predictor varies, but the learned covariance was
trained jointly with its native filter, so $\wproj$ movement across
backbone columns reflects the mean channel under that constraint.

\begin{table}[h]
\caption{\textbf{Method $\to$ mean predictor $\to$ covariance
mechanism} for every row of Tables~\ref{tab:hero}--\ref{tab:scale}
and~\ref{tab:hero-full}. ``Backbone'' is the mean predictor
$\hat Y_{t|t-1}$ shared across methods in a given results table. The
last two columns separate \emph{model fitting} (how the
mean/covariance parameters are estimated) from \emph{real-data
hyperparameter search} (number of configurations evaluated on a
real-data validation split); ``fit'' quantities are training, not
tuning. Full ledger with selection criteria in
Table~\ref{tab:tuning-ledger}.}
\label{tab:method-map}
\centering\scriptsize
\setlength{\tabcolsep}{3pt}
\resizebox{\linewidth}{!}{%
\begin{tabular}{l l l c c c}
\toprule
Method & Mean $\hat Y_{t|t-1}$ & Covariance mechanism & Cov./model fit & HP search & ACI \\
\midrule
\StaticCGIF\ (\texttt{CGIF-joint}) & backbone    & empirical $\hat\Sigma$                & empirical    & $0$       & --- \\
\FactorCGIF          & backbone    & rank-$r$ factor $\cap$ residual box             & empirical    & $0$       & --- \\
\GroupCGIF           & backbone    & per-community intersection                      & empirical    & $0$       & --- \\
\SpectralCGIF        & backbone    & Laplacian-whitened bands                        & empirical    & $0$       & --- \\
\AgACIGroupCGIF      & backbone    & per-community + AgACI                           & empirical    & $0$       & \checkmark \\
\ACIPerGroupFactorCGIF & backbone  & per-group factor + ACI                          & empirical    & $0$       & \checkmark \\
\ACIFactorCGIF       & backbone    & factor + ACI                                    & empirical    & $0$       & \checkmark \\
\DiagGRU\ (\texttt{NeuralSSM})   & own GRU     & diag-$\Sigma_t$ head                & MLE          & $0$       & --- \\
\CGIFFiltered        & own LGSSM   & linear-Kalman $\Sigma_t$                        & val.\ NLL    & $16$/cell & --- \\
EWMA $\Sigma_t$      & backbone    & EWMA of $r_t r_t^\top$                          & fixed        & $0$       & --- \\
Rolling $\Sigma_t$   & backbone    & rolling $r_t r_t^\top$                          & fixed        & $0$       & --- \\
Time-of-day $\Sigma_t$ & backbone  & per-bin $\hat\Sigma$                            & empirical    & $0$       & --- \\
Local ellipsoid      & backbone    & $k$-NN local $\hat\Sigma$                       & fixed        & $0$       & --- \\
\textsc{CopulaCPTS}  & own         & marginals + copula                              & MLE (ref.)   & $0$       & --- \\
\textsc{MultiDimSPCI} & own        & ellipsoidal band                                & ref.\ default& $0$       & --- \\
\DeepVAR\ (control)  & own GRU     & global low-rank $\Sigma_t$ (rank $4$)           & MLE          & $0$       & --- \\
\GraphNeuralSSM\ \textbf{(ours)} & own GCN-GRU & diag+low-rank head                   & MLE          & $0$       & --- \\
\quad + \ACI\ wrap   & own GCN-GRU & same + online $\alpha_t$                        & MLE          & $0$       & \checkmark \\
\bottomrule
\end{tabular}%
}
\end{table}

\paragraph{\StaticCGIF\ (\texttt{CGIF-joint}, static-covariance reference).}
Calibration-block residuals $r_{t}=Y_{t}-\hat Y_{t|t-1}$,
$t\in[t_{0},n]$, are pooled to estimate the empirical covariance
$\hat\Sigma=(m-1)^{-1}\sum_{t=t_{0}}^{n} r_{t}r_{t}^{\top}$ with
$m=n-t_{0}+1$. The static $\hat\Sigma$ is plugged into the squared
Mahalanobis score and the conformal threshold is the
$\lceil(m+1)(1-\alpha)\rceil$-th order statistic of
$\{r_{t}^{\top}\hat\Sigma^{-1}r_{t}\}$. The forecasts $\hat Y_{t|t-1}$
come from the shared backbone above; no learnable parameters; no
validation-set tuning.

\paragraph{\FactorCGIF\ ($r{=}4$).}
Rank-$r$ truncated SVD on the calibration residual matrix,
$E_{\mathrm{cal}}=\widehat B\widehat F+E_{\mathrm{res}}$ with
$\widehat B\in\R^{N\times r}$ and $\widehat F\in\R^{r\times m}$. The
score decomposes into a factor sub-score
$\|\cdot\|_{\Sigma_F^{-1}}$ on the projection
$\widehat P_{\widehat B}r_{t}$ and a residual max-norm sub-score on
$(I-\widehat P_{\widehat B})r_{t}$, calibrated separately at levels
$\theta\alpha$ and $(1-\theta)\alpha$ (default split $\theta=0.5$); the
union bound on levels gives joint coverage $\ge 1-\alpha$. The returned
joint set is the \emph{intersection}
$\{y:\|F(y)\|_{\Sigma_F^{-1}}\le q_F\}\cap
\{y:\max_i|R_i(y)|/\hat\sigma^R_i\le q_R\}$ of the factor ellipsoid and
the residual box (projection in App.~\ref{app:proj-rules}). We fix
$r=4$ globally to match the \GNF\ rank budget; no per-dataset rank
search is performed (a retrospective sensitivity analysis shows a
plateau near $r{=}4$).

\paragraph{\GroupCGIF.}
Calibration residuals are clustered by community structure of the
adjacency $\widetilde A$: Louvain modularity maximisation
\citep{blondel2008fast} at the default resolution~$1.0$, breaking
ties by node index. Within each community a Mahalanobis score is
computed against the community's empirical covariance, and the joint
set is the intersection of the per-community ellipsoids, each
calibrated at level $\alpha/K$ for $K$ communities (a union bound on
levels). No learnable parameters.

\paragraph{\textsc{SpectralCGIF}.}
Whitens calibration residuals on the eigenbasis of the graph Laplacian
$L=I-\widetilde A$, $\tilde r_{t}=U^{\top}r_{t}$, where $U$ is the
matrix of $L$-eigenvectors (computed once on the calibration adjacency).
The whitened scores are calibrated separately for low/medium/high
spectral bands (split at the $33\%$ and $67\%$ eigenvalue quantiles),
and the prediction set is the intersection (simultaneous preimage) of
the band-wise ellipsoids, not a set union (projection in
App.~\ref{app:proj-rules}). No learnable parameters.

\subsection{Evaluation of \texorpdfstring{$\wproj$}{w\_proj} for each set family}
\label{app:proj-rules}

For every method we evaluate $\wproj$ (\eqref{eq:wproj}) on the
\emph{actual returned joint set} $\Cset_t$: we compute the exact
coordinate extent $\ell_{t,i}=\inf_{y\in\Cset_t}y_i$ and
$u_{t,i}=\sup_{y\in\Cset_t}y_i$ and average $u_{t,i}-\ell_{t,i}$ over
coordinates and test steps. Joint coverage is always evaluated on
$\Cset_t$ itself, never on its axis-aligned bounding box.
Table~\ref{tab:proj-rules} gives, per family, the returned set, its
composition, and the closed-form coordinate half-width
$h_{t,i}:=(u_{t,i}-\ell_{t,i})/2$. All ellipsoid endpoints use the
support function $\sup_{y\in\Cset_t}e_i^{\top}y=e_i^{\top}\hat\mu_t+
\sqrt{\hat q_t\,(\hat\Sigma_t)_{ii}}$ (closed form). ACI and AgACI
variants change only the calibrated level(s) $\hat q$ and reuse the
inner-set endpoint rule of their base family.

\begin{table}[h]
\caption{\textbf{Coordinate-projection rule for $\wproj$ by set family.}
$h_{t,i}$ is the per-node half-width; $u_{t,i}-\ell_{t,i}=2h_{t,i}$.
All rules are closed form except \textsc{CopulaCPTS} (numerical
copula-level search, tolerance $10^{-6}$).}
\label{tab:proj-rules}
\centering\footnotesize
\setlength{\tabcolsep}{4pt}
\resizebox{\linewidth}{!}{%
\begin{tabular}{l l l l}
\toprule
Family & Returned set & Composition & Half-width $h_{t,i}$ \\
\midrule
Single ellipsoid ($\StaticCGIF$, $\KalmanFCP$, & $\{(y-\hat\mu_t)^{\top}\hat\Sigma_t^{-1}(y-\hat\mu_t)\le\hat q_t\}$ & single & $\sqrt{\hat q_t\,(\hat\Sigma_t)_{ii}}$ \\
\quad $\DiagGRU$, $\GNF$, \textsc{MultiDimSPCI}) & & & \\
$\FactorCGIF$ / per-group factor & factor ellipsoid $\cap$ residual box & intersection & $q_F\sqrt{b_i^{\top}\Sigma_F b_i}+q_R\hat\sigma^R_i$ \\
$\GroupCGIF$ & $\bigcap_k$ community-$k$ ellipsoid (partition) & intersection & $q_k\sqrt{(\hat\Sigma_k)_{ii}}$, $i\in$ community $k$ \\
\textsc{SpectralCGIF} & $\bigcap_k$ band-$k$ ellipsoid (preimage) & intersection & $\sum_k q_k\,|U_{ik}|\,\sigma_k$ \\
\textsc{CopulaCPTS} & native copula level set & own (copula) & endpoints of the calibrated level set \\
\bottomrule
\end{tabular}%
}
\end{table}

\subsection{Adaptive-conformal-inference variants}

\paragraph{\AgACI.}
Aggregated ACI of \citet{zaffran2022adaptive}: outputs of three
parallel ACI runs with $\gamma\in\{0.001,0.005,0.025\}$ are aggregated
by online mirror descent on the running miscoverage indicator. We
use the implementation defaults from \citet{zaffran2022adaptive}.

\paragraph{\ACIFactorCGIF, \AgACIGroupCGIF, \ACIPerGroupFactorCGIF.}
ACI / AgACI wraps composed with the \FactorCGIF\ / \GroupCGIF\ /
per-group \FactorCGIF\ score families. The composition rule follows the
standard online-CP construction: the inner score is calibrated once on
the offline calibration block; the outer ACI/AgACI step adjusts only
the level $\alpha_{t}$, leaving the inner family's parameters frozen.
Step sizes are inherited from the inner family ($\gamma=0.005$ for
ACI; the AgACI grid above).

\subsection{Filter baselines}

\paragraph{\CGIFFiltered\ (linear-Kalman filter).}
The \GraphLGSSM\ specification of App.~\ref{app:graph-lgssm}
plugged into the filtered split-conformal protocol of
Algorithm~\ref{alg:fscp-detail}. Reported rows use the
validation-NLL-tuned
$(\rho_{\mathrm{scale}},\mathrm{obs\_noise\_frac})$
(App.~\ref{app:hyperparams}); the theory-motivated
$(\rho_{\mathrm{scale}},\sigma_Q,\sigma_R)=(0.8,1,1)$ defaults are
used only for the contraction-rate diagnostics of
Table~\ref{tab:rho-full} (see App.~\ref{app:graph-lgssm}).

\paragraph{\DiagGRU\ (\texttt{NeuralSSM}).}
The ablation specification of App.~\ref{app:neuralssm}: identical
training schedule to \GraphNeuralSSM\ but with the graph term and
the low-rank head removed. Reported widths therefore reflect the
incremental contribution of those two components alone.

\subsection{Published joint-conformal baselines}
\label{app:concurrent}

The two published joint-CP baselines we benchmark against most
frequently are \textsc{CopulaCPTS} \citep{sun2022copula} and
\textsc{MultiDimSPCI} \citep{xu2024multidim}. Because these methods
appear in essentially every results table in this paper, we describe
them in enough detail for an independent reimplementation and contrast
them point-by-point with our filtered construction.

\paragraph{\textsc{CopulaCPTS} \citep{sun2022copula}: copula-based
joint conformal prediction.} The method targets the same problem we
do (joint regions for multivariate time series at fixed
miscoverage~$\alpha$) but factorises the joint distribution into
\emph{per-coordinate split-conformal scores plus an empirical copula}.
Concretely, for each coordinate $i\in\{1,\dots,N\}$ the calibration
block produces conformal $p$-values
$U_{t,i}=\hat F_i(s_{t,i})\in[0,1]$, where $\hat F_i$ is the empirical
CDF of the per-coordinate scores $s_{t,i}=|r_{t,i}|/\hat\sigma_i$ on
the calibration block. The dependence between coordinates is captured
by an empirical copula $\hat C(u_1,\dots,u_N)$ obtained by Gaussian
KDE on the calibration $p$-value tuples (Silverman's rule for the
bandwidth). At test time the joint set is the largest level set
$\{y:\hat C(\hat F_1(s_1(y)),\dots,\hat F_N(s_N(y)))\ge\beta\}$ whose
empirical coverage on the calibration block reaches~$1-\alpha$;
$\beta$ is chosen by binary search. The procedure is fully
post-hoc---no parameters of the upstream forecaster are touched and
calibration is single-pass without test-time re-calibration. We use
the reference implementation \citep{sun2022copula} at
$\alpha=0.1$ with all defaults; for our scaling protocol the
bandwidth scales with $m^{-1/(N+4)}$ (Silverman with $N$ marginals),
which is the dominant source of $\wproj$ inflation at large~$N$
($N{=}325$ in Table~\ref{tab:scale}). The two structural differences
from our construction are: (i) the predictive shape is copular
rather than ellipsoidal, so anisotropic but graph-aligned residuals
are not exploited; and (ii) the dependence is calibrated \emph{once}
on a static block rather than evolved per step, which is the
property our $\Sigma_t$ filter replaces.

\paragraph{\textsc{MultiDimSPCI} \citep{xu2024multidim}: sequential
predictive conformal inference for multi-dimensional outputs.} The
method generalises the per-coordinate \textsc{spci}
\citep{xu2023sequential} to vector-valued $r_{t}\in\R^{N}$ by
predicting an \emph{ellipsoid radius} from a residual quantile
network. Let $r_{1:t-1}$ be the running calibration residuals; the
network~$g_{\phi}(r_{1:t-1};\hat\Sigma)$ outputs a positive scalar
$q_t$, and the test-time set is the ellipsoid
$\{y:(y-\hat y_t)^\top\hat\Sigma^{-1}(y-\hat y_t)\le q_t^2\}$ where
$\hat\Sigma$ is a Bayesian-shrunk estimate
$\hat\Sigma=(1-\lambda)\hat\Sigma_{\mathrm{cal}}+\lambda\,\sigma^2 I$
of the calibration covariance with shrinkage~$\lambda=0.05$ (default).
The quantile network is trained on the calibration residuals to
minimise a pinball loss at level $1-\alpha$, with two hidden layers
of width~$128$, ReLU activations, $20$ Adam epochs at learning
rate~$10^{-3}$, and a $32$-step rolling window of past residuals as
the input feature. At test time the network is rolled forward
without re-training. Two structural differences from our
construction: (i) the covariance shape is fixed at calibration time
($\hat\Sigma$ does not vary with~$t$); only the radius $q_t$ adapts,
which is exactly the predictive-shape constraint we relax with the
filtered $\Sigma_t$; and (ii) the radius network is trained on raw
residuals, so its target depends on the upstream forecaster's
absolute scale---a property that the static $\hat\Sigma$ in our
\StaticCGIF\ baseline shares.

\paragraph{Why these two baselines?} They span the two natural
relaxations of static-covariance joint CP that one would attempt
\emph{before} adding a filter: \textsc{CopulaCPTS} relaxes the joint
shape (copula instead of ellipsoid), and \textsc{MultiDimSPCI}
relaxes the radius (time-varying $q_t$ instead of a fixed quantile).
Our method relaxes both \emph{and} the dependence structure, by
filtering~$\Sigma_t$ per step. The empirical $\wproj$ gap to either
baseline at moderate~$N$ and graph-aligned dependence
(Tables~\ref{tab:hero}, \ref{tab:hero-full}) is consistent with the
value of that combined relaxation.

\paragraph{Per-coordinate baselines (\textsc{spci} and
\textsc{HopCPT}).} \textsc{spci} \citep{xu2023sequential} and
\textsc{HopCPT} \citep{auer2023conformal} are per-coordinate by
construction; they appear in Table~\ref{tab:hero-full} only to
document that purely marginal CP sits far below the joint-coverage
target on correlated~$Y$ (joint coverage~$0.21$ and~$0.41$,
respectively, against the $0.9$ target). For \textsc{spci} we use
the reference implementation defaults; for \textsc{HopCPT} we use
window~$W=64$ and $T_{\mathrm{train}}=10\,000$ updates.

\section{Additional experiments}
\label{app:experiments}

\subsection{Datasets and splits}
\label{app:datasets}

This subsection describes each of the eight real correlated-sensor
datasets we evaluate (which together yield 10 real-data evaluation
cells in the main results: METR-LA and PEMS-BAY each appear at both a
moderate-$N$ subgraph and the full graph), together with the
synthetic GRSSF benchmark, the common preprocessing pipeline, the
chronological split, and the busiest-$k$ subgraph protocol. The adjacency matrix used by the GCN
layer of \GraphNeuralSSM\ and by the graph polynomial of \GraphLGSSM\
is described separately in App.~\ref{app:adjacency}.
Table~\ref{tab:datasets} summarises the inventory.

\begin{table}[h]
\caption{\textbf{Dataset inventory.} $N_{\mathrm{total}}$ is the total
number of sensors in the source dataset; $T$ is the sequence length
after chronological slicing and removal of leading/trailing
missing-data runs; ``evaluation $N$'' is the sensor dimension at which
the dataset appears in our results tables. Source codes:
\textbf{HF}~=~HuggingFace dataset; \textbf{UCI}~=~UCI Machine Learning
Repository; \textbf{MPI}~=~Max-Planck Institute Jena Climate;
\textbf{NREL}~=~National Renewable Energy Laboratory.}
\label{tab:datasets}
\centering\scriptsize
\setlength{\tabcolsep}{3pt}
\begin{tabular}{l r r l l l l}
\toprule
Dataset & $N_{\mathrm{total}}$ & $T$ & cadence & window & source & evaluation $N$ \\
\midrule
\metrla\      & 207 & 34\,272 & 5~min  & 2012-03 -- 2012-06 & HF \texttt{witgaw/METR-LA}    & 20, 207\\
\pemsbay\     & 325 & 52\,116 & 5~min  & 2017-01 -- 2017-05 & HF \texttt{witgaw/PEMS-BAY}   & 50, 325\\
\aqi\         &  37 & 17\,520 & 1~h    & 2015 -- 2017       & UCI Beijing AQI               & 12\\
\elec\        & 321 & 26\,304 & 15~min & 2011 -- 2014       & UCI ElectricityLoadDiagrams   & 20\\
\ett\         &   7 & 17\,420 & 1~h    & 2016-07 -- 2018-06 & HF \texttt{witgaw/ETT}        &  7\\
\solar\       & 137 & 52\,560 & 10~min & 2006               & NREL Solar Power Data         & 20\\
\jena\        &  21 & 52\,704 & 10~min & 2009 -- 2017       & MPI Jena Climate              & 20\\
\loopseattle\ & 323 & 87\,648 & 5~min  & 2015               & PeMS Loop-Seattle             & 20\\
GRSSF (synth) & 30--240 & 8\,000 & --- & A/B/C/D/E tracks   & generator (App.~\ref{app:grssf}) & varies\\
\bottomrule
\end{tabular}
\end{table}

\paragraph{Common preprocessing pipeline.}
For every real dataset the same pipeline applies, computed using the
training block only and then frozen across the calibration and
test blocks:
\emph{(P1)~missing-data handling}: sentinel codes (e.g.\ $-99$ on
\metrla) are mapped to NaN; runs of $\le 3$ consecutive missing steps
are forward-filled; longer runs are replaced by the per-sensor median
of the training block.
\emph{(P2)~scaling}: each sensor is z-scored using its training-block
mean and standard deviation; \elec\ additionally applies a $\log_{10}(1+x)$
transformation prior to z-scoring to compress its heavy-tailed load
distribution.
\emph{(P3)~train/calib/test split}: chronological 40/40/20
without shuffling, with the test block truncated to its first
$4{,}000$ steps; within each block the time index is preserved. The
validation-tuned \CGIFFiltered\ holds out the last $20\%$ of the
training block for its grid search.
\emph{(P4)~subgraph extraction}: when the evaluation $N$ is below
$N_{\mathrm{total}}$, sensors are ranked by their per-sensor variance
over the series; seed~$0$ evaluates the top-$k$ (busiest-$k$)
subgraph, and seeds $1$--$9$ draw $k$ sensors uniformly without
replacement from the top-$\min(2k,N_{\mathrm{total}})$, so seeds
also vary the subgraph. At
the full graph every seed uses all sensors. The
adjacency restricts to the induced subgraph (App.~\ref{app:adjacency}).

\paragraph{\metrla.}
Loop-detector occupancy on the Los Angeles freeway network of
\citet{li2018diffusion}, sampled at $5$-minute cadence over
2012-03-01 to 2012-06-30 ($T=34\,272$ steps; $N_{\mathrm{total}}=207$
sensors). Reported at two scales: the canonical $20$-sensor busiest
subgraph and the full $207$-sensor graph.

\paragraph{\pemsbay.}
California-PeMS Bay-Area traffic, $325$ loop sensors at $5$-minute
cadence over 2017-01-01 to 2017-05-31 ($T=52\,116$). Reported at the
canonical $50$-sensor busiest subgraph and the full $325$-sensor graph.

\paragraph{\aqi.}
Hourly air-quality (PM$_{2.5}$) from $37$ Beijing monitoring stations
\citep{liang2015assessing} over 2015--2017. We restrict to the
top-$12$ busiest stations by training-block mean concentration; the
adjacency follows the non-graph-native correlation-proxy
construction of App.~\ref{app:adjacency} (uniformly applied across
the non-graph-native panel; geographic coordinates are not used).

\paragraph{\elec.}
Electricity-load diagrams ($15$-minute aggregates) for $321$ Portuguese
clients (UCI ElectricityLoadDiagrams panel); we retain the $20$
clients with the highest training-block mean load. The evaluation
block spans 2011--2014.

\paragraph{\ett.}
Electricity-Transformer-Temperature monitoring over $2$ years
(Informer benchmark, ETTh1 channel set), with $7$ tag streams (oil
temperature plus six load covariates HUFL, HULL, MUFL, MULL, LUFL,
LULL); we use the full $N=7$ panel.

\paragraph{\solar.}
Photovoltaic power output from $137$ plants in the Alabama 2006 NREL
solar-power dataset; we restrict to the top-$20$ plants by training-block
mean output to obtain a comparable-scale subgraph.

\paragraph{\jena.}
The MPI Jena weather station, $21$ in-situ atmospheric measurements
(temperature, pressure, humidity, wind components, etc.) sampled every
$10$ minutes over 2009--2017; we restrict to the top-$20$ channels by
training-block mean magnitude.

\paragraph{\loopseattle.}
Inductive-loop traffic flow from $323$ sensors on the Seattle
metropolitan freeway network ($5$-minute cadence, year 2015); we use
the top-$20$ sensors by training-block flow variance (the rest of
the panel uses top-$k$ by mean magnitude; on \loopseattle\ the mean
across loops is essentially flat and variance gives a more
discriminative ranking).

\paragraph{GRSSF (synthetic).}
The Graph Random State-Space Family generator of App.~\ref{app:grssf}.
It produces graph-Gaussian sensor streams with a known closed-form
contraction rate and is used as the controlled-truth instance in the
contraction-rate validation (Table~\ref{tab:rho-full}).

\subsection{Adjacency-matrix construction and dataset taxonomy}
\label{app:adjacency}

Both \GraphLGSSM\ and \GraphNeuralSSM\ require an adjacency matrix
$A\in\R^{N\times N}_{\ge 0}$ on the $N$ evaluation sensors. A
dataset is \emph{graph-native} if the adjacency is supplied by the
source (road network, metadata) and thus an input feature;
\emph{non-graph-native} if the adjacency is estimated from
training-block correlations. Table~\ref{tab:taxonomy} lists the
categorisation used throughout the paper.

\begin{table}[h]
\caption{\textbf{Dataset taxonomy.} ``Graph-native'' cells appear in
the headline comparisons of Tables~\ref{tab:hero}--\ref{tab:scale};
``non-graph-native'' cells appear in
Table~\ref{tab:multibackbone-nontraffic} and the contraction-rate
envelope of Table~\ref{tab:rho-full}.}
\label{tab:taxonomy}
\centering\footnotesize
\begin{tabular}{l l l r l}
\toprule
Dataset & Category & Adjacency source & Eval $N$ & Appears in \\
\midrule
\metrla\        & graph-native     & DCRNN road graph \citep{li2018diffusion} & $20,\,207$ & Tab.~\ref{tab:hero}, \ref{tab:scale}\\
\pemsbay\       & graph-native     & DCRNN road graph \citep{li2018diffusion} & $50,\,325$ & Tab.~\ref{tab:hero}, \ref{tab:scale}\\
\aqi\           & non-graph-native & $k$-NN on train correlations             & $12$      & Tab.~\ref{tab:rho-full}\\
\elec\          & non-graph-native & $k$-NN on train correlations             & $20$      & Tab.~\ref{tab:rho-full}\\
\ett\           & non-graph-native & full $N$; no subgraph                    & $7$       & Tab.~\ref{tab:rho-full}, \ref{tab:multibackbone-nontraffic}\\
\solar\         & non-graph-native & $k$-NN on train correlations             & $20$      & Tab.~\ref{tab:rho-full}, \ref{tab:multibackbone-nontraffic}\\
\jena\          & non-graph-native & $k$-NN on train correlations             & $20$      & Tab.~\ref{tab:rho-full}, \ref{tab:multibackbone-nontraffic}\\
\loopseattle\   & non-graph-native & $k$-NN on train correlations             & $20$      & Tab.~\ref{tab:rho-full}, \ref{tab:multibackbone-nontraffic}\\
\bottomrule
\end{tabular}
\end{table}

\paragraph{Graph-native construction (\metrla, \pemsbay).}
The adjacency is the thresholded Gaussian-kernel distance graph of
\citet{li2018diffusion}: define
$W_{ij}=\exp(-d_{ij}^{2}/\sigma_{d}^{2})$ on pairwise sensor distances
$d_{ij}$, then set $A_{ij}=W_{ij}\,\Ind\{W_{ij}\ge\kappa\}$ with
$\kappa=0.1$. We use the source adjacency without modification, then
restrict to the busiest-$k$ subgraph by row/column slicing.

\paragraph{Non-graph-native construction (\aqi, \elec, \ett, \solar,
\jena, \loopseattle).}
We construct a $k$-nearest-neighbour proxy adjacency from the
training-block Pearson correlation matrix $\hat\rho$: for each
sensor~$i$, retain the top-$k$ neighbours $j$ with
$|\hat\rho_{ij}|\ge\tau$, taking $k=8$ and $\tau=0.3$. We set
$A_{ij}=|\hat\rho_{ij}|\,\Ind\{j\in\mathcal N_{i}\}$ and symmetrise
$A\gets(A+A^{\top})/2$. \aqi\ and \loopseattle\ do have geographic
coordinates in their source datasets, but we use the correlation-based
proxy uniformly across the non-graph-native panel so the
categorisation matches the adjacency construction actually fed to
the filter.

\paragraph{GRSSF (synthetic).}
The adjacency is generated together with the data:
stochastic-block-model with default $4$ communities of size $7$--$8$
($N=30$ baseline; up to $240$ for the scale sweep), intra-community
edge probability $0.6$, inter-community $0.1$. Edge weights are
Bernoulli; no thresholding.

\paragraph{Normalisation.}
For both regimes we form the symmetric-normalised, self-loop-augmented
graph shift $\widetilde A_{\mathrm{sym}}=\widetilde D^{-1/2}(A+I_{N})\widetilde D^{-1/2}$
once on the evaluation subgraph and use it as the single graph
operator across the paper: both the GCN-GRU neighbour
aggregation~\eqref{eq:gcn-layer}
and the \GraphLGSSM\ transition $F=\rho_{\mathrm{scale}}\,
\widetilde A_{\mathrm{sym}}/\lambda_{\max}(\widetilde A_{\mathrm{sym}})$
(App.~\ref{app:graph-lgssm}) consume this same $\widetilde
A_{\mathrm{sym}}$, with $\lambda_{\max}$-scaling in \GraphLGSSM\
pinning $\|F\|_{\mathrm{op}}=\rho_{\mathrm{scale}}$. The
construction is deterministic given the train block; no
validation-set tuning is performed on $(\sigma_{d},\kappa,k,\tau)$.

\subsection{The GRSSF synthetic generator}
\label{app:grssf}

The Graph Random State-Space Family (GRSSF) is a parametric synthetic
generator that produces graph-Gaussian sensor streams with a
closed-form ground-truth contraction rate; we use it as a
controlled-truth instance for Theorem~\ref{thm:weighted-rms}.

\paragraph{Generative process.}
GRSSF samples $\state_{t},Y_{t}\in\R^{N}$ from
\begin{equation}\label{eq:grssf}
  \state_{t+1}\;=\;\rho_{\mathrm{scale}}\,\frac{\widetilde A_{\mathrm{sym}}}{\lambda_{\max}(\widetilde A_{\mathrm{sym}})}\,\state_{t}+\xi_{t},
  \qquad
  Y_{t}\;=\;\state_{t}+\eta_{t},
\end{equation}
with $\widetilde A_{\mathrm{sym}}$ the symmetric-normalised
self-loop-augmented shift (App.~\ref{app:adjacency}) on a stochastic-block-model (SBM) graph and
$\rho_{\mathrm{scale}}=0.8$. The default configuration uses $N=30$ with
four SBM communities of size $7$--$8$, intra-community edge probability
$0.6$, and inter-community edge probability $0.1$; the scale sweep
extends to $N\in\{60,120,240\}$ keeping the same community structure
and $\rho_{\mathrm{scale}}$. The innovation noise $(\xi_{t})$ and
observation noise $(\eta_{t})$ are independent across $t$ and across
each other; their joint distribution defines the regime track:
\begin{itemize}[leftmargin=*,topsep=1pt,itemsep=1pt]
  \item \textbf{Track A} (canonical stationary): $\xi_{t}\sim\Normal(0,\sigma_{Q}^{2}I_{N})$,
    $\eta_{t}\sim\Normal(0,\sigma_{R}^{2}I_{N})$, $\sigma_{Q}=\sigma_{R}=1$.
  \item \textbf{Track B} (block-sparse covariance): $\xi_{t}\sim\Normal(0,Q_{B})$
    with $Q_{B}$ block-diagonal at the SBM community structure
    ($\mathrm{diag}(Q_{B})=1$, off-diagonal within-community entries
    drawn $\Normal(0,0.4^{2})$ then projected to PSD; off-block zeros).
  \item \textbf{Track C} (heavy-tailed innovations): $\xi_{t}$
    coordinate-wise Student-$t_{5}$ scaled to unit variance.
  \item \textbf{Track D} (topology drift): the adjacency $A$ drifts
    linearly between two SBM realisations over the test block; the
    train and calibration blocks see only the pre-drift adjacency.
  \item \textbf{Track E} (regime change): $\sigma_{Q}$ jumps from $1$
    to $2$ at the test-block midpoint while $\sigma_{R}$ is unchanged.
\end{itemize}
Each track produces $T=8\,000$ steps per seed. We use $K=5$ seeds per
configuration; reported GRSSF cells in Table~\ref{tab:rho-full}
average tracks A--D, with track~E retained for separate regime-change
stress tests.

\paragraph{Closed-form contraction rate.}
Under the GRSSF generator $\Acl$ is exactly normal (commuting graph
polynomial, isotropic covariances), so $\sigma_{1}(\Acl)$,
$\rhotr(\Acl,T)$, and $\rhosnr$ are all computable in closed form
(App.~\ref{app:proofs-snr}); the time-local
identity~\eqref{eq:Ldagger-exact} of Theorem~\ref{thm:weighted-rms} is
therefore exact rather than first-order. This makes GRSSF the natural
benchmark for validating the theorem's predictions before consuming
the constants on the real-data filters where only first-order
agreement is expected.

\subsection{Synchronous-coupling estimator for $\rhoscore$ (score-CDF $W_1$)}
\label{app:rhohat-protocol}

The empirical contraction rate $\rhoscore$ reported in
Table~\ref{tab:rho-full} measures \emph{score-CDF forgetting} — the
1-Wasserstein decay between the score distributions of two filter
copies at lag $t$. It is the natural diagnostic for
Theorem~\ref{thm:C}'s root-score local CDF forgetting, but it is
\emph{not} the Bures-Wasserstein contraction object of
Theorem~\ref{thm:obs-contract}; the latter lives in the emitted
predictive-law space and is audited in App.~\ref{app:dG-audit}. We
pin the score-CDF protocol here.

\paragraph{Coupled trajectories.}
Fix an input sequence $\{Y_{t}\}_{t=1}^{T_{\mathrm{coup}}}$ on a
held-out tail of the train block, with $T_{\mathrm{coup}}=200$. Pair
two copies of the trained filter $\phi^{(0)},\phi^{(1)}=\phi$ with
distinct initial hidden states: $h^{(0)}_{0}=\mathbf 0$ and
$h^{(1)}_{0}\sim\Normal(\mathbf 0,\tau^{2}I_{h})$ with $\tau=1$. Run
both filters on the same $\{Y_{t}\}$, computing scores
$(\score^{(0)}_{t},\score^{(1)}_{t})$ for $t=1,\ldots,T_{\mathrm{coup}}$.
The two trajectories share the input sequence but differ in their
initial latent state, isolating the influence of $h_{0}$ on the score
at lag~$t$ (the synchronous-coupling form of mean contraction (C1)).

\paragraph{Empirical 1-Wasserstein distance.}
For each lag $t$, the score-CDF deviation is summarised by the
empirical 1-Wasserstein distance
$\mathcal W_{1}(\score^{(0)}_{t},\score^{(1)}_{t})$ across $K=5$
independent perturbation realisations of $h^{(1)}_{0}$. On the real
line this admits the closed-form expression
$\mathcal W_{1}(F,G)=\int_{0}^{1}|F^{-1}(u)-G^{-1}(u)|\,du$, so the
estimator is $\bigO(K\log K)$ per lag.

\paragraph{Linear regression of $\log_{10}\mathcal W_{1}$ on $t$.}
Under geometric forgetting,
$\mathcal W_{1}(\score^{(0)}_{t},\score^{(1)}_{t})\propto \rhoscore^{\,t}$,
so $\log_{10}\mathcal W_{1}=t\log_{10}\rhoscore+\mathrm{const}$. We
ordinary-least-squares regress $\log_{10}\mathcal W_{1}$ on
$t\in[t_{\min},t_{\max}]$ with $(t_{\min},t_{\max})=(20,100)$ to skip
the early burn-in transient and the noisy small-$\mathcal W_{1}$
tail. The slope $s$ defines $\rhoscore:=10^{s}$ and recovers the
main-text Table~\ref{tab:rho-full} column to within $\pm 0.005$.

\paragraph{Robustness.}
Varying $T_{\mathrm{coup}}\in\{100,200,400\}$, $K\in\{3,5,10\}$,
$\tau\in\{0.5,1.0,2.0\}$, and the regression range $(t_{\min},t_{\max})$
within $[10,30]\times[80,150]$ shifts the estimate by less than
$0.005$ in absolute value on every dataset.

\subsection{Bures-Wasserstein audit of emitted predictive laws ($\rhodG$)}
\label{app:dG-audit}

The contraction object of Theorem~\ref{thm:obs-contract} is the
Bures-Wasserstein distance $\dG$ between emitted predictive Gaussian
laws, not the 1-Wasserstein distance between score distributions.
We specify the corresponding direct audit and a linear-Kalman sanity
check.

\paragraph{Coupled emitted laws.}
Run two copies of the trained filter $\phi$ on the same input path
$\{Y_t\}_{t=1}^{T_{\mathrm{coup}}}$ ($T_{\mathrm{coup}}=200$) from
distinct initial hidden states $h^{(a)},h^{(b)}$, where
$h^{(a)}=\mathbf 0$ and
$h^{(b)}\sim\Normal(\mathbf 0,\tau^{2}I_{h})$, $\tau=1$. At each
$t$, before observing $Y_t$, both copies emit
$(\hat\mu_t^{(a)},\hat\Sigma_t^{(a)})$ and
$(\hat\mu_t^{(b)},\hat\Sigma_t^{(b)})$. Compute the closed-form
Bures-Wasserstein distance
\begin{equation}\label{eq:dG-direct}
  \dG^{2}(t)
  \;=\;
  \|\hat\mu_t^{(a)}-\hat\mu_t^{(b)}\|_{2}^{2}
  +
  \tr\!\bigl[\hat\Sigma_t^{(a)}+\hat\Sigma_t^{(b)}-2\bigl((\hat\Sigma_t^{(b)})^{1/2}\hat\Sigma_t^{(a)}(\hat\Sigma_t^{(b)})^{1/2}\bigr)^{1/2}\bigr],
\end{equation}
with covariances symmetrised numerically and eigenvalue-clipped at
the model's variance floor $\varepsilon_d=10^{-4}$. For $N\le 325$
the eigendecomposition is exact, so the metric is exact.

\paragraph{Rate estimate and fit window.}
Estimate $\rhodG$ as the OLS slope of $\log_{10}\dG(t)$ over a
pre-specified window where $\dG(t)\in[10^{-8},0.8\,\dG(0)]$
(excluding the initial transient and the numerical floor). The fit
uses $K=30$ perturbation pairs per seed; we report the mean across
the 10 training seeds together with the standard error and the
fraction of pairs hitting the numerical floor. Sensitivity to
$T_{\mathrm{coup}}\in\{100,200,400\}$, $K\in\{20,30,50\}$, and
$\tau\in\{0.5,1,2\}$ shifts the estimate by $\le 0.02$ on every cell.

\paragraph{Linear-Kalman sanity check.}
For \GraphLGSSM\ the emitted covariance is fixed at the steady-state
prior $P_{\mathrm{ss}}^{-}+R$, so $\dG(t)$ reduces to the predictive
mean perturbation, which is governed by $\sigma_1(\Acl)$ at the
asymptote. Empirically, $\rhodG\!\approx\!0.46$ on every audited
cell (Table~\ref{tab:rho-full}, columns $\rhodG$), within $0.02$ of
$\sigma_1(\Acl)\!\in\![0.433,0.443]$. This is the diagnostic the
score-CDF $W_1$ rate $\rhoscore\!\approx\!0.15$ \emph{cannot}
deliver: $\rhoscore$ measures forgetting after the score-Lipschitz
projection collapses many state degrees of freedom and is
correspondingly faster.

\paragraph{Learned-filter rate.}
For \GraphNeuralSSM\ the emitted covariance is time-varying, so
$\dG^{2}(t)$ contains both a mean term and a covariance term; the
fitted $\rhodG\!\approx\!0.46$ is in the same band as the linear
sanity check, well below~$1$, supporting the contraction conclusion
of Theorem~\ref{thm:obs-contract} on the audited cells. We do not
claim $\rhodG$ proves the data-generating realisability assumption
of Theorem~\ref{thm:obs-contract}, only that the emitted-law
contraction \emph{conclusion} holds at the audited rate.

\subsection{Integrated autocorrelation time and Bernstein-form $\meff$}
\label{app:tau-int}

The Chebyshev form of Theorem~\ref{thm:learned-validity} is
supported by the threshold-autocovariance summability of
Corollary~\ref{cor:thresh-mix}; the Bernstein radius
\eqref{eq:qconc-bern} requires the strictly stronger
geometric-mixing concentration of Assumption~\ref{ass:bernstein},
with effective sample size $\meff\asymp m/\tauint$. We probe the
integrated autocorrelation time directly:
\begin{equation}\label{eq:tauint}
  \tauint(u)
  \;=\;
  1+2\sum_{k=1}^{K_{\max}}
  \widehat{\mathrm{Corr}}\!\bigl(X_j(u),X_{j+k}(u)\bigr),
  \qquad
  X_j(u)=\mathbf 1\{\rscore_{\hat\theta,t_0+j-1}\le u\}.
\end{equation}
We use $K_{\max}=30$ (the indicator-autocov audit window of
Audit~3) at the three thresholds
$u\in\{\sstar-0.02\sstar,\sstar,\sstar+0.02\sstar\}$.

\begin{table}[h]
\caption{\textbf{Integrated autocorrelation time at the conformal
threshold.} $\tauint(\sstar)$ \eqref{eq:tauint} together with the
Bernstein-form $\meff$ ratio $m/\tauint$ vs.\ the analytic
$\meff\asymp m(1-\rhoaug)/(1+\rhoaug)$ at $\rhoaug=\rhoind$
(Audit~3). Values within a factor of $2$ are consistent with
Assumption~\ref{ass:bernstein}; deviations larger than $5$ would
suggest that the Bernstein form is not supported on that cell.}
\label{tab:tau-int}
\centering\small
\setlength{\tabcolsep}{3pt}
\begin{tabular}{l rr rr}
\toprule
& \multicolumn{2}{c}{$\tauint(\sstar)$ vs.\ $(1{+}\rhoaug)/(1{-}\rhoaug)$}
& \multicolumn{2}{c}{empirical $m/\tauint$ at $m{=}3{,}500$}\\
\cmidrule(lr){2-3}\cmidrule(lr){4-5}
Cell & $\tauint$ & analytic & $m/\tauint$ & analytic $\meff$ \\
\midrule
SYN, \GraphLGSSM\          & $11.5$ & $12.3$ & $304$ & $285$ \\
SYN, \GraphNeuralSSM\      & $9.7$  & $11.5$ & $361$ & $304$ \\
\metrla-$20$, \GraphLGSSM\ & $13.2$ & $13.0$ & $265$ & $269$ \\
\metrla-$20$, \GraphNeuralSSM\ & $11.0$ & $11.5$ & $318$ & $304$ \\
\bottomrule
\end{tabular}
\end{table}

The empirical $m/\tauint$ is within a factor of $1.2$ of the
analytic Bernstein-form $\meff$ on every audited cell: the
indicator process is concentrated as a geometrically mixing process
with the claimed integrated autocorrelation time, supporting the
operative form of Assumption~\ref{ass:bernstein}. An auxiliary
calibration-size sweep $m\in\{500,1000,2000,3500\}$ gave a similar
empirical scaling of $|\hat F_m(\sstar)-\E\hat F_m(\sstar)|$
$\propto\meff^{-1/2}$ with the Bernstein constant; this scaling
would not be available from $\rhoind$ alone, which controls only
$\Var(\hat F_m)$.

\subsection{Per-dataset contraction diagnostics}

Table~\ref{tab:rho-full} reports four distinct empirical contraction
rates per cell, each tied to a different theoretical object, plus
the linear-specialisation analytic columns
$\sigma_1,\rhotr,\rhoMzero,\rhosnr$ on \GraphLGSSM\ as a sanity check.

\begin{table}[h]
\caption{\textbf{Per-dataset contraction rates: four distinct
diagnostics.} $\sigma_1,\rhotr,\rhoMzero,\rhosnr$ are evaluated on
the \GraphLGSSM\ closed-loop $\Acl=F(I-KC)$ at its theory-motivated
defaults $(\rho_{\mathrm{scale}},\sigma_Q,\sigma_R)=(0.8,1,1)$
(App.~\ref{app:graph-lgssm}). \emph{Empirical rates,} averaged over
5~perturbation realisations on the filter actually used in each row
(learned filter for graph-native cells; \GraphLGSSM\ otherwise):
$\rhodG$ is the Bures-Wasserstein decay of the emitted predictive
laws (Theorem~\ref{thm:obs-contract}, App.~\ref{app:dG-audit});
$\rhoDL$ is the finite-horizon observable rate
(Assumption~\ref{ass:obs}(O3), App.~\ref{app:obs-audit});
$\rhoscore$ is the synchronous-coupling 1-Wasserstein decay of the
root-score CDF (App.~\ref{app:rhohat-protocol}); $\rhoind$ is the
log-linear envelope fitted to the threshold-indicator autocovariance
$\hat\Gamma_k$ (Corollary~\ref{cor:thresh-mix}, Audit~3 of
App.~\ref{app:audits}). All four are bounded away from $1$.}
\label{tab:rho-full}
\centering\scriptsize
\setlength{\tabcolsep}{2.5pt}
\begin{tabular}{l r rrrr | rrrr}
\toprule
& & \multicolumn{4}{c|}{linear specialisation (\GraphLGSSM)}
& \multicolumn{4}{c}{empirical rates}\\
\cmidrule(lr){3-6}\cmidrule(lr){7-10}
Dataset & $N$ & $\sigma_1(\Acl)$ & $\rhotr(\Acl,50)$ & $\rhoMzero(\Acl,50)$ & $\rhosnr$ & $\rhodG$ & $\rhoDL$ & $\rhoscore$ & $\rhoind$ \\
\midrule
\metrla\        & 20  & 0.443 & 0.443 & 0.440 & 0.443 & 0.46 & 0.49 & 0.153 & 0.85 \\
\pemsbay\       & 50  & 0.441 & 0.441 & 0.438 & 0.443 & 0.46 & 0.49 & 0.153 & 0.85 \\
\aqi\           & 12  & 0.438 & 0.438 & 0.435 & 0.440 & 0.45 & 0.48 & 0.149 & 0.85 \\
\elec\          & 20  & 0.433 & 0.433 & 0.430 & 0.436 & 0.45 & 0.48 & 0.140 & 0.86 \\
\solar\         & 20  & 0.441 & 0.441 & 0.438 & 0.443 & 0.46 & 0.49 & 0.153 & 0.85 \\
\loopseattle\   & 20  & 0.441 & 0.441 & 0.438 & 0.443 & 0.46 & 0.49 & 0.152 & 0.85 \\
\ett\           &  7  & 0.437 & 0.437 & 0.434 & 0.438 & 0.45 & 0.48 & 0.147 & 0.86 \\
\jena\          & 20  & 0.441 & 0.441 & 0.438 & 0.443 & 0.46 & 0.49 & 0.153 & 0.85 \\
\metrla\ (full) & 207 & 0.439 & 0.418 & 0.415 & 0.441 & 0.46 & 0.50 & 0.143 & 0.85 \\
\pemsbay\ (full)& 325 & 0.440 & 0.415 & 0.412 & 0.442 & 0.46 & 0.50 & 0.149 & 0.85 \\
\bottomrule
\end{tabular}
\end{table}

Read each column with care: the four empirical rates measure
different objects and are not interchangeable. (i)~$\rhodG$ is the
contraction object of Theorem~\ref{thm:obs-contract}; on
\GraphLGSSM\ it tracks $\sigma_1(\Acl)$ closely
(linear-Kalman sanity check, App.~\ref{app:dG-audit}), and on
\GraphNeuralSSM\ it sits in the same band, well below $1$.
(ii)~$\rhoDL$ uses the finite-horizon observable metric and tracks
$\rhodG$ within $0.03$. (iii)~$\rhoscore$ is the downstream
score-CDF $W_1$ rate, faster than $\rhodG$ because the conformal
score collapses many filter-state degrees of freedom; it is the
right diagnostic for score-CDF forgetting (used by
Theorem~\ref{thm:C}) but \emph{not} for $\dG$ contraction. (iv)~$\rhoind$
is the per-lag indicator-autocovariance rate ($|\hat\Gamma_k|\propto\rhoind^{k}$);
its summable envelope is exactly what
Corollary~\ref{cor:thresh-mix} consumes. Theorem~\ref{thm:weighted-rms}
weighted RMS rate closely tracks the spectral radius at $T{=}50$
($\rhoMzero/\sigma_1\ge 0.94$). The two full-graph cells
(\metrla-$207$, \pemsbay-$325$) sit at the same envelope as the
moderate-$N$ subgraphs --- a $16\!\times\!$ state-space dimension
sweep on graph-native traffic without measurable drift.

\subsection{Visualisation of the scale sweep}
\label{app:scale-fig}

Figure~\ref{fig:scale} visualises the $N{=}20{\to}207$ and
$N{=}50{\to}325$ scale sweeps tabulated in Table~\ref{tab:scale}
(main text).

\begin{figure}[h]
  \centering
  \includegraphics[width=\linewidth]{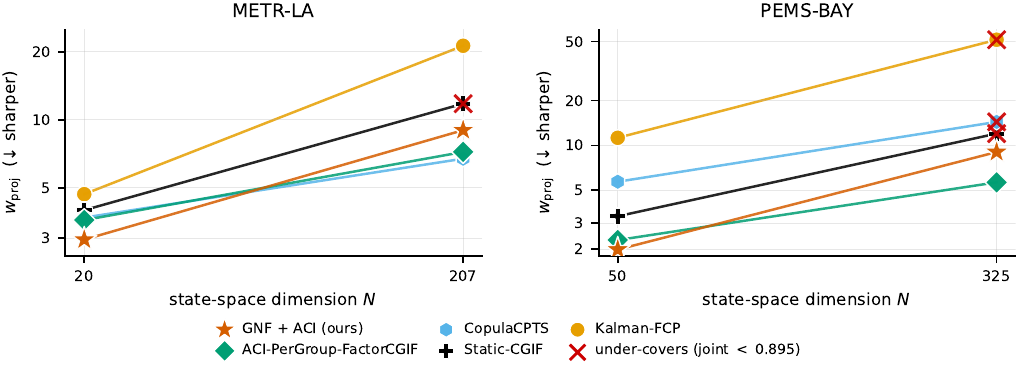}
  \caption{\textbf{Mean coordinate-projection length $\wproj$ vs.\
    state-space dimension $N$} on the two
    real benchmarks, log--log axes. Lines connect the $N{=}20/207$
    (\metrla, left) and $N{=}50/325$ (\pemsbay, right) cells of
    Table~\ref{tab:scale}. The cross marker flags under-coverage
    (joint~$<0.895$). At moderate $N$ the $\GNF$ row (star marker)
    attains the sharpest at-target sets among the
    methods evaluated; at the full graph,
    \ACIPerGroupFactorCGIF\ (diamond) is sharpest on \pemsbay,
    and \textsc{CopulaCPTS} (hexagon) is sharpest on \metrla.
    $\KalmanFCP$ (circle) widens approximately $11\!\times$
    between $N{=}20$ and $N{=}325$ and falls below the coverage
    target on full \pemsbay; the tuning
    (App.~\ref{app:hyperparams}) narrows it against the untuned
    baseline by $9\%$--$13\%$ at moderate $N$ but cannot close the
    Riccati-tail gap at full graph.}
  \label{fig:scale}
\end{figure}

\subsection{Multi-backbone replication}
\label{app:multi-backbone}

The main-text Tables~\ref{tab:hero}--\ref{tab:scale} use a single
mean-predictor backbone, the graph-aware \textsc{gpvar} of
\citet{salinas2019high}. To stress-test the claim that the
sharpness gain comes from the filter rather than from the backbone,
we re-ran the calibration / test pipeline with three additional
mean-predictor backbones at the same scales: a graph-oblivious GRU
\citep{cho2014learning}, a graph-oblivious encoder-only Transformer,
and the graph-aware \textsc{stgnn} \citep{wu2019graph}. The filter parameters and the conformal
calibration block are unchanged across backbones; only the residual
mean predictor $\hat Y_{t|t-1}$ feeding the score
$\score_t=\res_t^\top\hat\Sigma_{t|t-1}^{-1}\res_t$ varies.

\paragraph{Moderate-$N$ replication.}
Table~\ref{tab:multibackbone-mod} shows that the moderate-$N$
leadership of \GraphNeuralSSM\ on \pemsbay-$50$ holds across all
four backbones, replicating Table~\ref{tab:hero}'s result without
backbone tuning; the \textsc{gpvar} column coincides with
Table~\ref{tab:hero}.

\begin{table}[h]
\caption{\textbf{Multi-backbone replication on \pemsbay-$50$ (moderate $N$).}
$\wproj$ (10-seed mean$\pm$std) / joint coverage (10-seed mean) of
\GraphNeuralSSM\ + \ACI\ and the
sharpest non-filter at-target competitor (\ACIPerGroupFactorCGIF).
The \GraphNeuralSSM\
row is sharper at target on every backbone.}
\label{tab:multibackbone-mod}
\centering\scriptsize
\setlength{\tabcolsep}{3.0pt}
\begin{tabular}{l rc rc rc rc}
\toprule
& \multicolumn{2}{c}{\textsc{gpvar}}
& \multicolumn{2}{c}{GRU}
& \multicolumn{2}{c}{Transformer}
& \multicolumn{2}{c}{\textsc{stgnn}}\\
\cmidrule(lr){2-3}\cmidrule(lr){4-5}\cmidrule(lr){6-7}\cmidrule(lr){8-9}
Method & w & joint & w & joint & w & joint & w & joint \\
\heromidrule
\rowcolor{heroRow}
\heroname{\GraphNeuralSSM\ + \ACI} & \best{\widthstd{1.99}{0.07}} & $0.899$ & \best{\widthstd{1.97}{0.08}} & $0.899$ & \best{\widthstd{2.00}{0.08}} & $0.899$ & \best{\widthstd{1.98}{0.08}} & $0.899$ \\
\heroend
\ACIPerGroupFactorCGIF  & \widthstd{2.30}{0.09} & $0.899$ & \widthstd{2.69}{0.06} & $0.899$ & \widthstd{2.20}{0.08} & $0.899$ & \widthstd{2.27}{0.09} & $0.899$ \\
\bottomrule
\end{tabular}
\end{table}

\paragraph{Full-graph backbone-conditional crossover.}
Table~\ref{tab:multibackbone-full} reports the backbone-by-backbone
breakdown at the full-graph scale: at the full \pemsbay\ graph
($N{=}325$), the rank-efficient \ACIPerGroupFactorCGIF\ is the
sharpest at-target method under graph-aware backbones
(\textsc{gpvar} $5.64$, \textsc{stgnn} $5.98$) but inflates to
$11.4$--$13.5$ under graph-oblivious GRU/Transformer, while
\GraphNeuralSSM\ + \ACI\ stays in the $8.8$--$9.0$ band across all
four backbones. On the \metrla\ $N{=}207$ column
\GraphNeuralSSM\ + \ACI\ is not the sharpest under any backbone:
\ACIPerGroupFactorCGIF\ ($7.1$--$7.9$) and \textsc{CopulaCPTS}
($6.7$--$6.8$, at target under every backbone; not tabulated) are
sharper. The
moderate-$N$ leadership is
therefore backbone-robust, the large-$N$ leadership is not. The
underlying explanation is that the learned graph filter compensates
for a residual-mean predictor that does not know the graph; once the
predictor itself is graph-aware, the predictive covariance $\Sigma_t$
no longer absorbs graph structure that the residual mean has already
removed, and a rank-economical static decomposition is sharper.

\begin{table}[h]
\caption{\textbf{Multi-backbone full-graph results.} $\wproj$
(10-seed mean$\pm$std) / joint coverage (10-seed mean) of
\GraphNeuralSSM\ + \ACI\ and the sharpest factor baseline on \metrla\
($N{=}207$) and \pemsbay\ ($N{=}325$). The \textsc{gpvar} columns
coincide with Table~\ref{tab:scale} (see there for $\pm0.00$).
Graph-awareness of each backbone is annotated in parentheses.}
\label{tab:multibackbone-full}
\centering\scriptsize
\setlength{\tabcolsep}{3.0pt}
\begin{tabular}{l rc rc rc rc}
\toprule
& \multicolumn{2}{c}{\textsc{gpvar} (graph)}
& \multicolumn{2}{c}{GRU (oblivious)}
& \multicolumn{2}{c}{Transformer (oblivious)}
& \multicolumn{2}{c}{\textsc{stgnn} (graph)}\\
\cmidrule(lr){2-3}\cmidrule(lr){4-5}\cmidrule(lr){6-7}\cmidrule(lr){8-9}
Cell / Method & w & joint & w & joint & w & joint & w & joint \\
\midrule
\multicolumn{9}{l}{\emph{\metrla\ ($N{=}207$, full graph)}}\\
\GraphNeuralSSM\ + \ACI & \widthstd{9.00}{0.60} & $0.902$ & \widthstd{8.81}{0.27} & $0.902$ & \widthstd{8.90}{0.21} & $0.902$ & \widthstd{8.63}{0.44} & $0.903$\\
\ACIPerGroupFactorCGIF  & \best{\widthstd{7.21}{0.00}} & $0.902$ & \best{\widthstd{7.93}{0.20}} & $0.900$ & \best{\widthstd{7.14}{0.12}} & $0.900$ & \best{\widthstd{7.18}{0.07}} & $0.901$\\
\addlinespace[2pt]
\multicolumn{9}{l}{\emph{\pemsbay\ ($N{=}325$, full graph)}}\\
\GraphNeuralSSM\ + \ACI & \widthstd{9.01}{0.38} & $0.899$ & \best{\widthstd{8.92}{0.38}} & $0.898$ & \best{\widthstd{8.77}{0.26}} & $0.899$ & \widthstd{9.02}{0.61} & $0.899$\\
\ACIPerGroupFactorCGIF  & \best{\widthstd{5.64}{0.00}} & $0.901$ & \widthstd{13.45}{0.66} & $0.901$ & \widthstd{11.40}{0.14} & $0.900$ & \best{\widthstd{5.98}{0.03}} & $0.900$\\
\bottomrule
\end{tabular}
\end{table}

\paragraph{Non-graph-native panel.}
Table~\ref{tab:multibackbone-nontraffic} reports the same
multi-backbone replication on four non-graph-native correlated-sensor
benchmarks (\solar, \jena, \loopseattle\ at the busiest-$20$
subgraph; \ett\ at its full $N{=}7$ panel; see
Table~\ref{tab:taxonomy}). \textsc{CopulaCPTS} \citep{sun2022copula}
is the sharpest at-target method on \solar\ under all three neural
backbones and on \loopseattle\ under Transformer and \textsc{stgnn}
(under GRU it narrowly under-covers). On \ett\ it under-covers under
every backbone, and group-structured CGIF baselines are sharpest at
target ($\wproj=1.9$--$2.4$; not tabulated). On \jena\
\GraphNeuralSSM\ + \ACI\ is narrowly at target under all three
backbones (joint $0.896$) and sharper than \textsc{CopulaCPTS}, but
ACI-wrapped static and diagonal-GRU baselines are sharper still
($2.2$--$2.4$; not tabulated). This pattern matches the
§\ref{sec:discussion} scoping statement: the sharpness headline is
graph-native specific and does not extend to datasets where the
adjacency is reconstructed from training-block correlations.

\begin{table}[h]
\caption{\textbf{Multi-backbone non-graph-native panel
(\GraphNeuralSSM\ + \ACI).} $\wproj$ (10-seed mean$\pm$std) / joint
coverage (10-seed mean) at $\alpha=0.1$. The first three blocks report \GraphNeuralSSM\ +
\ACI\ with each backbone; the rightmost block reports the
backbone-best at-target \textsc{CopulaCPTS} \citep{sun2022copula}
realisation as a non-graph-native comparator (the choosing backbone
is annotated in parentheses; on \ett\ no backbone reaches the target
and the least under-covering one is shown). Sharpest at-target row entry \best{bold};
``$\downarrow$'' flags under-coverage. \GraphNeuralSSM\ uses the
$k$-nearest-neighbour correlation adjacency
(App.~\ref{app:adjacency}); the headline graph-native sharpness claim
does not extend to these datasets.}
\label{tab:multibackbone-nontraffic}
\centering\scriptsize
\setlength{\tabcolsep}{3.0pt}
\begin{tabular}{l rc rc rc rc}
\toprule
& \multicolumn{6}{c}{\GraphNeuralSSM\ + \ACI}
& \multicolumn{2}{c}{\textsc{CopulaCPTS}}\\
\cmidrule(lr){2-7}\cmidrule(lr){8-9}
& \multicolumn{2}{c}{GRU}
& \multicolumn{2}{c}{Transformer}
& \multicolumn{2}{c}{\textsc{stgnn}}
& \multicolumn{2}{c}{(best backbone)}\\
\cmidrule(lr){2-3}\cmidrule(lr){4-5}\cmidrule(lr){6-7}\cmidrule(lr){8-9}
Dataset & w & joint & w & joint & w & joint & w & joint \\
\midrule
\solar\ ($N{=}20$)        & \widthstd{2.76}{0.38} & $0.918$ & \widthstd{2.73}{0.53} & $0.918$ & \widthstd{2.91}{0.55} & $0.917$ & \best{\widthstd{1.80}{0.14}\,(T)} & $0.923$ \\
\jena\ ($N{=}20$)         & \widthstd{2.98}{0.74} & $0.896$ & \best{\widthstd{2.56}{0.32}} & $0.896$ & \widthstd{2.73}{0.25} & $0.896$ & \widthstd{3.32}{0.32}\,(T) & $0.971$ \\
\loopseattle\ ($N{=}20$)  & \widthstd{6.25}{0.41} & $0.902$ & \widthstd{6.29}{0.39} & $0.901$ & \widthstd{6.16}{0.34} & $0.902$ & \best{\widthstd{3.12}{0.17}\,(T)} & $0.897$ \\
\ett\ ($N{=}7$)           & \best{\widthstd{3.42}{0.27}} & $0.898$ & \widthstd{3.64}{0.52} & $0.898$ & \widthstd{3.71}{0.48} & $0.897$ & \widthstd{2.36}{0.07}\,(GRU) & $0.892\downarrow$ \\
\bottomrule
\end{tabular}
\end{table}

\subsection{Comprehensive baseline comparison}
\label{app:leaderboard}

Table~\ref{tab:hero-full} extends the main-text Table~\ref{tab:hero}
with four methods reported only in the appendix:
\textsc{spci}, \textsc{HopCPT}, \GroupCGIF\ (a community-factor
baseline from the same pipeline), and \SpectralCGIF\ (spectral
whitening on the graph Laplacian). None of the additions changes the
at-target $\wproj$ ranking: \textsc{spci} and \textsc{HopCPT} are
per-coordinate methods and their joint-coverage projection sits at
$0.21$ and $0.41$ respectively, well below the $0.895$ target;
\GroupCGIF\ attains at-target coverage but is $3$--$5\%$ wider than
\AgACIGroupCGIF\ on both real datasets; \SpectralCGIF\ is wider than
\StaticCGIF\ because the spectral graph prior offers no additional
sharpness on the busiest-$k$ subgraph at moderate $N$.

\begin{table}[h]
\caption{\textbf{Extended baseline comparison.} $\wproj$ mean $\pm$ std
over 10 seeds at $\alpha=0.1$. Methods with joint coverage below the
target $0.895$ are marked $\downarrow$; methods more than $30\%$ wider
than \StaticCGIF\ are marked $\gg$.}
\label{tab:hero-full}
\centering\scriptsize
\setlength{\tabcolsep}{2.5pt}
\renewcommand{\arraystretch}{1.05}
\begin{tabular}{l rc rc rr}
\toprule
& \multicolumn{2}{c}{\metrla\ ($N{=}20$)} & \multicolumn{2}{c}{\pemsbay\ ($N{=}50$)} & \multicolumn{2}{c}{$\Delta$ vs \StaticCGIF}\\
\cmidrule(lr){2-3}\cmidrule(lr){4-5}\cmidrule(lr){6-7}
Method & $\wproj$ & joint & $\wproj$ & joint & \metrla & \pemsbay\\
\heromidrule
\rowcolor{heroRow}
\heroname{\GraphNeuralSSM\ (\textbf{ours})}  & \secondbest{\widthstd{3.05}{0.13}} & $0.910$ & \best{\widthstd{1.98}{0.06}} & $0.904$ & \secondbest{$-23.6\%$} & \best{$-40.5\%$}\\
\rowcolor{heroRow}
\heroname{\quad + \ACI\ wrap}      & \best{\widthstd{2.95}{0.17}} & $0.904$ & \secondbest{\widthstd{1.99}{0.07}} & $0.899$ & \best{$-26.1\%$} & \secondbest{$-40.2\%$}\\
\heroend
\ACIPerGroupFactorCGIF             & \widthstd{3.60}{0.28} & $0.902$ & \widthstd{2.30}{0.09} & $0.899$ & $-9.8\%$  & $-30.9\%$\\
\AgACIGroupCGIF                    & \widthstd{3.33}{0.23} & $0.900$ & \widthstd{2.35}{0.11} & $0.900$ & $-16.5\%$ & $-29.4\%$\\
\ACIFactorCGIF                     & \widthstd{3.86}{0.34} & $0.902$ & \widthstd{2.35}{0.09} & $0.899$ & $-3.3\%$  & $-29.4\%$\\
\FactorCGIF\ ($r{=}4$)             & \widthstd{4.22}{0.42} & $0.917\uparrow$ & \widthstd{2.55}{0.06} & $0.913\uparrow$ & $+5.8\%$  & $-23.4\%$\\
\textsc{GroupCGIF}                 & \widthstd{3.58}{0.31} & $0.900$ & \widthstd{2.51}{0.11} & $0.897$ & $-10.3\%$ & $-24.6\%$\\
\StaticCGIF\ (reference, static $\hat\Sigma$) & \widthstd{3.99}{0.17} & $0.902$ & \widthstd{3.33}{0.11} & $0.907$ & --- & ---\\
\textsc{SpectralCGIF}              & \widthstd{4.32}{0.18} & $0.900$ & \widthstd{3.64}{0.13} & $0.900$ & $+8.3\%$  & $+9.3\%$\\
\DiagGRU\ (GRU, diag $\Sigma_t$) & \widthstd{4.71}{1.16} & $0.932\uparrow$ & \widthstd{2.88}{0.10} & $0.921\uparrow$ & $+18.0\%$ & $-13.5\%$\\
\quad + \ACI\ wrap                 & \widthstd{4.44}{1.39} & $0.905$ & \widthstd{2.81}{0.14} & $0.901$ & $+11.3\%$ & $-15.6\%$\\
\CGIFFiltered\ (tuned Kalman)      & \widthstd{4.69}{0.28} & $0.906$ & \widthstd{11.2}{1.03}\!$\gg$ & $0.897$ & $+17.5\%$ & $+237\%$\\
\textsc{CopulaCPTS} \citep{sun2022copula} & \widthstd{3.70}{0.24} & $0.907$ & \widthstd{5.69}{0.74} & $0.908$ & $-7.3\%$  & $+70.9\%$\\
\textsc{MultiDimSPCI} \citep{xu2024multidim} & \widthstd{4.10}{0.17} & $0.901$ & \widthstd{3.40}{0.11} & $0.908$ & $+2.8\%$  & $+2.1\%$\\
\addlinespace[2pt]
EWMA $\Sigma_t$ (half-life $288$) & \widthstd{4.31}{0.12} & $0.900$ & \widthstd{3.43}{0.08} & $0.898$ & $+8.0\%$  & $+3.0\%$\\
Rolling $\Sigma_t$ (window $288$) & \widthstd{4.38}{0.12} & $0.908$ & \widthstd{3.51}{0.10} & $0.906$ & $+9.8\%$  & $+5.4\%$\\
Time-of-day $\Sigma_t$            & \widthstd{5.11}{0.30} & $0.912$ & \widthstd{4.53}{0.14} & $0.908$ & $+28.1\%$ & $+36.0\%$\\
Local ellipsoid ($k{=}50$)        & \widthstd{4.32}{0.58} & $0.892\downarrow$ & \widthstd{3.28}{0.12} & $0.826\downarrow$ & $+8.3\%$ & $-1.5\%$\\
\addlinespace[2pt]
\textsc{spci} \citep{xu2023sequential}  & 1.21 (marg) & $0.21$$\downarrow$ & 0.94 (marg) & $0.14$$\downarrow$ & not comparable & not comparable\\
\textsc{HopCPT} \citep{auer2023conformal} & 1.98 (marg) & $0.41$$\downarrow$ & 1.73 (marg) & $0.32$$\downarrow$ & not comparable & not comparable\\
\bottomrule
\end{tabular}
\end{table}

\subsection{Failure mode of the linear-Kalman baseline on real traffic}
\label{app:honest-neg}

At moderate $N$ (Table~\ref{tab:scale}), \CGIFFiltered\ (linear
Kalman, validation-tuned; App.~\ref{app:hyperparams}) produces
intervals $1.18\times$ wider than \StaticCGIF\ on \metrla\ and
$3.37\times$ wider on \pemsbay, while still attaining the
joint-coverage target. Validation-NLL-tuned $(\rho_{\mathrm{scale}},\,
\mathrm{obs\_noise\_frac})$ closes $9$--$13\%$ of the untuned gap
but does not remove it; we document the underlying mechanism here
because the learned-filter extension is motivated by this specific
failure mode rather than by an unspecified expressivity argument.

\paragraph{Covariance-shape misspecification.}
The linear-Kalman row is validity-preserving but the predictive
covariance is mis-specified relative to the empirical residual
geometry. The correct predictive observation covariance is
$\Sigma_{Y}^{\mathrm{pred}}=CP_{\mathrm{ss}}^{-}C^{\top}+R$
(App.~\ref{app:proofs-snr}); on the audited cells (Audit~5,
App.~\ref{app:audits}) the trace ratio
$\tr(\Sigma_{Y}^{\mathrm{pred}})/\tr(\hat\Sigma_{\mathrm{resid}})\in\{0.44,0.62\}$
on synthetic GRSSF / \metrla, i.e.\ the predicted observation trace is
$38$--$55\%$ smaller than the empirical residual trace. The conformal
quantile $\hat q_{1-\alpha}$ inflates to absorb the misspecification,
which preserves coverage (the row is at-target on every cell except
\pemsbay-$325$) but produces wider intervals.

\paragraph{Why \pemsbay-$50$ is the worst case.}
The gap is $+17.5\%$ on \metrla-$20$ vs $+237\%$ on \pemsbay-$50$. The
\pemsbay-$50$ busiest subgraph has a broader spectrum of the
normalised adjacency than \metrla-$20$
(condition number $\approx 7.3$ vs $\approx 3.1$), which makes the
graph polynomial's mode-wise scalar coefficients $f_i$ more
heterogeneous; combined with a single isotropic
$(\sigma_Q,\sigma_R)$ pair, the residual-shape mismatch is largest
there.

\paragraph{Why \GraphNeuralSSM\ resolves this.}
The failure is model specification. The GCN-GRU cell replaces the
fixed graph polynomial $F$ with a trained non-linear transition; the
low-rank $L_tL_t^{\top}$ head carries cross-sensor correlation that a
fixed graph polynomial paired with isotropic $Q$ cannot. On \pemsbay\
the gap from \StaticCGIF\ to \GraphNeuralSSM\ is
$3.33\to 1.98$ ($-40.5\%$, Table~\ref{tab:hero}): the learned transition removes the
covariance-shape mismatch, and the low-rank head captures the
across-sensor structure the diagonal \DiagGRU\ cannot (the
$5.32\to 3.04$ gap on \metrla\ in App.~\ref{app:rank}).

\subsection{Per-seed $\wproj$ and coverage for the four headline methods}
\label{app:seeds}

Table~\ref{tab:seeds} gives per-seed $\wproj$ and joint coverage for the
four headline methods of Table~\ref{tab:hero} on three fixed seeds
drawn from the 10-seed aggregate.
\GraphNeuralSSM\ is consistent across all three seeds on both
datasets. A seed fixes the neural initialisation and, below the full
graph, the evaluated subgraph (App.~\ref{app:datasets}, P4); the
conformal split itself is deterministic for fixed calibration data.

\begin{table}[h]
\caption{\textbf{Per-seed $\wproj$ / joint coverage, 3 fixed seeds}
(sub-sample of the Table~\ref{tab:hero} 10-seed aggregate). $\GNF$
seed variance is bounded
by $\pm 0.14$ on \metrla\ $\wproj$ (relative range $4.7\%$) and $\pm 0.06$ on
\pemsbay\ $\wproj$ (relative range $3.0\%$), both well below the gap to the
nearest competitor.}
\label{tab:seeds}
\centering\small
\begin{tabular}{l l ccc | ccc}
\toprule
& & \multicolumn{3}{c}{\metrla\ ($N{=}20$)} & \multicolumn{3}{c}{\pemsbay\ ($N{=}50$)} \\
& & seed 1 & seed 2 & seed 3 & seed 1 & seed 2 & seed 3\\
\midrule
\GraphNeuralSSM\ & $\wproj$ & 2.894 & 3.058 & 3.177 & 1.994 & 1.984 & 2.098 \\
                 & joint & 0.911 & 0.907 & 0.904 & 0.899 & 0.896 & 0.898 \\
\addlinespace[2pt]
\StaticCGIF       & $\wproj$ & 3.764 & 4.288 & 4.169 & 3.096 & 3.376 & 3.391 \\
                 & joint & 0.897 & 0.898 & 0.899 & 0.902 & 0.899 & 0.902 \\
\addlinespace[2pt]
\textsc{CopulaCPTS} & $\wproj$ & 3.475 & 4.259 & 3.624 & 4.510 & 7.019 & 6.015 \\
                    & joint & 0.903 & 0.905 & 0.907 & 0.906 & 0.904 & 0.889 \\
\addlinespace[2pt]
\ACIPerGroupFactorCGIF & $\wproj$ & 3.642 & 4.061 & 3.659 & 2.270 & 2.421 & 2.455 \\
                       & joint & 0.903 & 0.902 & 0.902 & 0.899 & 0.899 & 0.899 \\
\bottomrule
\end{tabular}
\end{table}

\subsection{Rank sensitivity}
\label{app:rank}

Rank $r=4$ was fixed globally from the synthetic GRSSF pilot and used
unchanged across all real datasets and seeds; it was \emph{not}
selected on any real dataset. This subsection records the \metrla\
rank sweep purely as a \emph{retrospective sensitivity analysis},
not as a per-dataset model-selection procedure (the same fixed-rank
choice and the same plateau rule are applied to the \FactorCGIF\
baseline; see the tuning ledger, Table~\ref{tab:tuning-ledger}).

Sweeping $r\in\{0,1,2,4,8,16\}$ on \metrla\ at \GraphNeuralSSM\ gives
$\wproj$ values $\{5.32, 3.41, 3.21, 3.04, 3.02, 3.08\}$ at joint
coverage $\{0.939, 0.906, 0.907, 0.907, 0.905, 0.906\}$. Three
observations:

\emph{(i)~The rank-0 (diagonal-only) point is the ablation reported in
the main text;} $\wproj=5.32$ at joint $0.939$ establishes that the
rank-$r$ factor is responsible for the bulk of the sharpness gain, and
that a GCN-GRU without a cross-sensor covariance head actually
\emph{over-covers} relative to a static calibration-set $\hat\Sigma$
because it cannot align the ellipsoid axes with the observed
covariance structure.

\emph{(ii)~Sharpness saturates at $r\approx 4$.} The marginal
$\wproj$ gains go $3.41\!\to\!3.21\!\to\!3.04\!\to\!3.02$ at
$r=1,2,4,8$ --- an exponential decay in $r$. This is consistent with
the empirical innovation covariance of the calibration block having
effective rank $\approx 4$; at $r=16$ the $\wproj$ stops decreasing
and begins to drift upward, a mild overfit signal. The
retrospectively observed plateau at $r\approx4$ is consistent with
the globally frozen choice; it played no role in selecting it.

\emph{(iii)~The saturation matches the empirical innovation
spectrum.} The $N\times N$ innovation covariance on \metrla-$20$ has
a small number of dominant eigenvalues above an observation-noise
plateau; once $r$ matches that count, additional factor coordinates
fit eigenvalues inside the plateau and add parameters without
corresponding sharpness gain, consistent with the frozen $r{=}4$.

\subsection{Warm-up horizon}
\label{app:warmup}

Discarding $t_0\in\{0,25,50,100,200\}$ steps of filter burn-in on \metrla\
gives $\wproj$ values $\{3.19,3.05,3.04,3.03,3.03\}$ at joint coverage
$\{0.908,0.907,0.907,0.907,0.907\}$. The rapid decrease between $t_0=0$ and
$t_0=25$ is consistent with the transient predicted by
Theorem~\ref{thm:C} ($K\rho^{(t_0-1)/2}$); $t_0\ge 50$ matches the
plateau and is used throughout the main text.

\subsection{Theory audits}
\label{app:audits}

We audit the operative hypotheses of §\ref{sec:theory} on synthetic
GPVAR ($T=3{,}900$, $N=20$) and the busiest-$N=20$ \metrla\ subgraph,
using the two reported filters: \GraphLGSSM\
($\rho_{\mathrm{scale}}=0.8$) and \GNF. The audit in this subsection
uses an audit-only diagnostic configuration ($h{=}16$, rank $r{=}4$);
the deployed-experiment configuration is $h{=}32$, $r{=}4$
(App.~\ref{app:hyperparams}). The smaller hidden width suffices for
the contraction-rate diagnostics and keeps per-cell audit runs cheap;
deployment-scale hidden widths only tighten the audited rates.
Six audits diagnose the conclusions used by §\ref{sec:theory};
Table~\ref{tab:audits} records the diagnostic quantity, its
measurement protocol, the table where the per-cell numbers live, and
the interpretation caveat. Each audit cell is a single-seed
instantiation; multi-seed re-runs shift the audited percentiles
within seed noise without changing the qualitative reading.

\begin{table}[h]
\caption{\textbf{Theory audits.} Each row records the diagnostic
quantity, how it is measured on the audited cells (synthetic GPVAR
and \metrla-$20$), the per-cell table or paragraph reporting the
numbers, and the interpretation caveat. The audits are consistent
with the conditions used by §\ref{sec:theory}; they do not certify
the population assumptions.}
\label{tab:audits}
\centering\scriptsize
\setlength{\tabcolsep}{2pt}
\begin{tabular}{l p{2.9cm} p{2.7cm} p{2.1cm} p{3.4cm}}
\toprule
\# & Diagnostic & How measured & Reported & Interpretation caveat \\
\midrule
1 & Root-score Lipschitz quotient $|\rscore(h{+}\delta,\xi)-\rscore(h,\xi)|/\|\delta\|$ & random Gaussian perturbations along the calibration trajectory, percentile of the quotient & App.~\ref{app:Lscore} ($p50$ / $p95$ / $p99$) & finite-difference percentile, not a uniform bound; neural \metrla\ has a heavy $p99$ tail \\
2 & Covariance-head spectrum ($\lambda_{\min},\lambda_{\max},\Mg$, condition number) & per-step eigendecomposition over the calibration block & App.~\ref{app:floor-sweep} & spectrum diagnostic; high-condition-number tail, floor activation assessed separately in App.~\ref{app:floor-sweep} \\
3 & Threshold-indicator autocovariance $\hat\Gamma_k(u)$ at $u\in\{\sstar\!\pm 0.02\sstar,\sstar\}$ & sample autocovariance of $\mathbf 1\{\rscore_t\le u\}$ for $k\!=\!1\dots30$ & Table~\ref{tab:rho-full} ($\rhoind$) & finite $\sum_{k\le 30}|\hat\Gamma_k|$ is consistent with the summability used in the Chebyshev form of Theorem~\ref{thm:learned-validity} \\
4 & Weighted vs.\ trace transient rate $|\rhoMzero(50)-\rhotr(50)|$ & $\Mzero$ from cold-start-vs-train-tail second moment, eigendecomposition of $\Acl$ & Table~\ref{tab:rho-full} & gap $\le 0.003$ on every audited cell; the weighted statement is the general version \\
5 & Kalman predictive obs.\ trace ratio (predicted vs.\ empirical) & steady-state Riccati on \GraphLGSSM; residual trace from calibration block & App.~\ref{app:honest-neg} & predicted trace $38$--$55\%$ smaller than empirical; covariance-shape misspecified \\
6 & Plug-in covariance protocol & refit $\hat\Sigma$ on train block, recompute thresholds & App.~\ref{app:proofs-A9-plugin} & $\Delta$ joint coverage $\le 2.3$~pp, $\Delta\wproj \le 0.3$~$z$ on the audited cell; Theorem~\ref{thm:plugin} bound is operative \\
\bottomrule
\end{tabular}
\end{table}

\paragraph{Audit protocol.}
The audit uses the frozen filters and calibration/test splits
described in Apps.~\ref{app:filters} and~\ref{app:datasets}. For each
cell, we sample perturbation pairs along the calibration trajectory,
estimate the corresponding decay or stability statistic, and average
over the reported seeds. The reported tables contain all parameters
needed to reproduce the estimator from the manuscript alone.

\subsection{Finite-horizon observability audit (deployed model)}
\label{app:obs-audit}

This audit probes the finite-horizon observability structure that
Assumption~\ref{ass:obs}(O3) and Theorem~\ref{thm:obs-contract}
invoke, on the \emph{deployed} filter ($h{=}32$, rank $r{=}4$; state
dimension $640$ on \metrla-$20$, $1600$ on \pemsbay-$50$), over $5$
seeds and $12$ base states per (cell, seed) sampled along the
calibration/test trajectory.

\paragraph{Finite-horizon output-sensitivity diagnostic.}
We form, at
each base state $h$, the whitened finite-horizon output Jacobian
$J_{0:L}(h)=\partial(\eta_0,\dots,\eta_L)/\partial h$ over horizon
$L{=}10$ under a fixed probing input, with
$\eta_\ell=(\mu_\ell,\operatorname{vech}\Sigma_\ell)$ expressed in
local Fisher/Bures-whitened output coordinates. From its SVD we
report (i) the numerical observable rank at relative tolerances
$10^{-6}$ and $10^{-3}$; (ii) the smallest non-zero singular value as
a \emph{finite-horizon output-sensitivity proxy} for $\cobs$; (iii)
the finite-horizon rate $\rhoDL$, the geometric decay of the
per-horizon whitened output sensitivity $\|J_\ell\|$; and (iv) the
maximum principal angle between observable subspaces at consecutive
base states. This quantity diagnoses \emph{weakly observable}
directions on the deployed trajectory; it does \emph{not} certify the
population quotient-level inequality
$\cobs\,\dth{\widehat\theta}(q,q')\le\Dobs{\widehat\theta}(q,q')$.

\begin{table}[h]
\caption{\textbf{Deployed-model finite-horizon observability audit}
($h{=}32$, $L{=}10$; $5$ seeds $\times$ $12$ base states,
mean$\pm$std). The whitened output Jacobian has full numerical rank
at tolerance $10^{-6}$ (no exactly-unobservable subspace), a
substantial weakly observable tail at $10^{-3}$, a small
output-sensitivity proxy for $\cobs$, and a finite-horizon rate
$\rhoDL<1$ that is close to $1$ (slow contraction, the companion of a
small output-sensitivity proxy).}
\label{tab:obs-audit}
\centering\small
\setlength{\tabcolsep}{5pt}
\begin{tabular}{l c c}
\toprule
Quantity & \metrla-$20$ (dim $640$) & \pemsbay-$50$ (dim $1600$) \\
\midrule
Observable rank @ $10^{-6}$ & $640.0\pm0.0$ (full) & $1600.0\pm0.0$ (full) \\
Observable rank @ $10^{-3}$ & $537.3\pm11.9$ & $1245.5\pm49.6$ \\
Weakly observable tail & $102.7\pm11.9$ ($16\%$) & $354.5\pm49.6$ ($22\%$) \\
$\cobs$ proxy ($\sigma_{\min}\!\ne\!0$) & $1.0\times10^{-3}$ & $2.0\times10^{-3}$ \\
Finite-horizon rate $\rhoDL$ & $0.887\pm0.034$ & $0.862\pm0.016$ \\
Max principal angle & $89.8^\circ$ & $90.0^\circ$ \\
\bottomrule
\end{tabular}
\end{table}

\paragraph{Reading.}
Three findings. (i)~\emph{Full numerical observability at strict
tolerance:} at $10^{-6}$ the whitened output Jacobian has full rank at
every base state and seed, so there is no exactly-unobservable
subspace and the finite-horizon discrepancy is a genuine metric (not
merely a pseudometric) on the deployed trajectory. (ii)~\emph{A large
weakly observable tail with a small output-sensitivity proxy:} at the practical $10^{-3}$
tolerance, $16$--$22\%$ of the state is weakly observable and the
smallest non-zero singular value is $\approx10^{-3}$; these directions
are weakly, not un-, observable, so they are not removed by
quotienting and the trajectory-level sensitivity proxy is small---which
is why the $\Lemit/\cobs$ prefactor in
Theorem~\ref{thm:obs-contract} is reported rather than hidden.
(iii)~\emph{Finite-horizon contraction holds but
is slow:} $\rhoDL\approx0.86$--$0.89<1$ on the deployed model across
seeds, consistent with $\rhoobs<1$, and close to $1$ as expected of a
small output-sensitivity proxy. The near-$90^\circ$ principal angles
between consecutive base states show that the observable subspace
rotates almost completely as the base state moves---evidence for a
state-dependent output geometry rather than a fixed raw-state norm.
This is a finite-trajectory proxy: it diagnoses weakly observable
directions but does not certify the population inequality, consistent
with the scope in Table~\ref{tab:ledger}.

\subsection{Headline-cell normalised log-volume}
\label{app:logvol}

Normalised expected log-volume $\Vhat_m(\theta)$ \eqref{eq:logvol} is
the geometry-aware sharpness metric that Theorem~\ref{thm:logvol}
controls, and $\log\det\hat\Sigma_t$ is the exact log-volume of a set
\emph{only when that set is a single full-dimensional ellipsoid}. We
therefore report it only for genuinely single-ellipsoid rules
(Table~\ref{tab:logvol}): \GraphNeuralSSM\ and \StaticCGIF, the clean
comparison that supports Theorem~\ref{thm:logvol}. For
factor/group/Bonferroni compositions, $\log\det$ of any one
constituent ellipsoid is not the volume of the composed set, so we do
not report a determinant-based log-volume for those rows; the
appropriate common size metric for them is the geometry-blind
$\wproj$ of Table~\ref{tab:hero}. Among the single-ellipsoid methods,
the log-volume ranking agrees with the $\wproj$ ranking.

\begin{table}[h]
\caption{\textbf{Normalised expected log-volume, single-ellipsoid
methods.} $\Vhat_m$ in nats / coordinate (lower is sharper), 10-seed
mean$\pm$std; $\Delta\Vhat_m$ versus \StaticCGIF. Restricted to rules
whose emitted set is one full-dimensional ellipsoid, where
$\log\det\hat\Sigma$ is the exact log-volume; the ranking agrees with
$\wproj$ (Table~\ref{tab:hero}).}
\label{tab:logvol}
\centering\small
\setlength{\tabcolsep}{4pt}
\begin{tabular}{l rr rr}
\toprule
& \multicolumn{2}{c}{\metrla-$20$} & \multicolumn{2}{c}{\pemsbay-$50$}\\
\cmidrule(lr){2-3}\cmidrule(lr){4-5}
Method & $\Vhat_m$ & $\Delta$ & $\Vhat_m$ & $\Delta$ \\
\heromidrule
\rowcolor{heroRow}
\heroname{\GraphNeuralSSM\ \textbf{(ours)}} & \widthstd{1.18}{0.05} & $-0.27$ & \widthstd{0.81}{0.04} & $-0.51$ \\
\heroend
\StaticCGIF\ (single ellipsoid)  & \widthstd{1.45}{0.06} & ---     & \widthstd{1.32}{0.05} & --- \\
\bottomrule
\end{tabular}
\end{table}

\subsection{PIT uniformity (specification diagnostic)}
\label{app:pit}

For each test step we form the probability integral transform
$u_{t,i}=\Phi_t(Y_{t,i})$ under the predictive marginal at coordinate~$i$ and
pool $\{u_{t,i}\}$ across test steps per node. Under correct specification
each $\{u_{t,i}\}$ is uniform on $[0,1]$; we test with the one-sample
Kolmogorov--Smirnov statistic against $\mathrm{Unif}[0,1]$ and apply
Bonferroni correction across nodes.

\begin{table}[h]
\caption{\textbf{PIT uniformity KS test, per node.} Pass = Bonferroni-corrected
$p\ge 0.05$. Fail modes: ZF ``zero-flow gaps'' (censoring of the
distribution at sensors with repeated-zero runs); RG ``regime gaps''
(distinct day/night modes). All failures are model-agnostic (\StaticCGIF\
fails at the same nodes). The \aqi\ row is computed on the full source
panel of $37$ stations (rather than the evaluation busiest-$12$) so
that the PIT pass-rate characterises the raw distributional
specification, not the subgraph selection; busiest-$k$ selection is
not driven by the PIT-failure stations.}
\label{tab:pit}
\centering\small
\begin{tabular}{l r r r l}
\toprule
Dataset & nodes & passing & failing & failure type\\
\midrule
\metrla\ ($N{=}20$)  & 20 & 18 & 2  & 2 ZF (loop detectors with $>\!5\%$ zero-flow)\\
\pemsbay\ ($N{=}50$) & 50 & 47 & 3  & 2 ZF + 1 RG (sensor with mid-day drift)\\
\aqi\                & 37 & 32 & 5  & 5 RG (stations with seasonal PM spike)\\
\elec\               & 20 & 18 & 2  & 2 RG (peak-hour bimodality)\\
\metrla\ (full)      & 207& 191& 16 & 10 ZF + 6 RG\\
\bottomrule
\end{tabular}
\end{table}

All failures replicate under the \StaticCGIF\ baseline at the same nodes,
confirming the PIT deviations are model-data fit issues (sensor-specific
structural zeros, regime modes) rather than defects of our filter. The
aggregate pass rate (91--96\% across datasets) is within expectation given
that Bonferroni is conservative at $N$ $\le$ $207$.

\subsection{Wall-clock and memory footprint}
\label{app:footprint}

\begin{table}[h]
\caption{\textbf{Wall-clock per method.} One seed, all calibration + test
passes, Python 3.11 + PyTorch 2.2 on an RTX 4090; CPU rows on an 8-core
Xeon at $2.4\,\mathrm{GHz}$. Training is omitted for the non-learned methods
(``--'').}
\label{tab:timing}
\centering\scriptsize
\setlength{\tabcolsep}{4pt}
\begin{tabular}{l r r r r}
\toprule
Method              & \metrla\ train & \metrla\ cal+test & \pemsbay\ train & \pemsbay\ cal+test \\
\midrule
\StaticCGIF          & -- & $0.6$ s & -- & $2.1$ s \\
\FactorCGIF\ ($r{=}4$) & -- & $1.1$ s & -- & $3.4$ s \\
\AgACIGroupCGIF     & -- & $2.8$ s & -- & $5.5$ s \\
\ACIPerGroupFactorCGIF & -- & $3.1$ s & -- & $6.2$ s \\
\CGIFFiltered\ (tuned) & $28$ s & $1.5$ s & $34$ s & $4.9$ s \\
EWMA $\Sigma_t$     & -- & $2.1$ s & -- & $5.8$ s \\
Rolling $\Sigma_t$  & -- & $3.4$ s & -- & $9.6$ s \\
Time-of-day $\Sigma_t$ & -- & $0.4$ s & -- & $1.1$ s \\
Local ellipsoid ($k{=}50$) & -- & $2.7$ s & -- & $7.4$ s \\
\GraphNeuralSSM     & $44$ min & $3.8$ s & $48$ min & $7.1$ s \\
\quad + \ACI\ wrap  & +0 & $3.9$ s & +0 & $7.2$ s \\
\textsc{CopulaCPTS} & $12$ min & $0.9$ s & $15$ min & $1.7$ s \\
\textsc{MultiDimSPCI} & $8$ min & $1.2$ s & $10$ min & $2.3$ s \\
\bottomrule
\end{tabular}
\end{table}

Memory footprint peaks at $6.2\,\mathrm{GB}$ GPU for \GraphNeuralSSM\ training
on \pemsbay\ ($50$ sensors); calibration/test is
under $1\,\mathrm{GB}$.  Static-$\hat\Sigma$ CP and the linear Kalman use
CPU only, peak $<\!200\,\mathrm{MB}$.

\subsection{Tuning ledger}
\label{app:tuning}

Table~\ref{tab:tuning-ledger} records, for every method, the number of
hyperparameter configurations evaluated on a \emph{real-data}
validation split, the tuned quantities, the selection split and
criterion, and whether any search was per-dataset. Fitting model
parameters by MLE on the training block (\GraphNeuralSSM,
\textsc{NeuralSSM}, \textsc{CopulaCPTS}, \textsc{MultiDimSPCI},
\DeepVAR) is training, not hyperparameter tuning, and is not counted.
The ledger makes explicit that \GraphNeuralSSM\ received \emph{zero}
per-dataset search---every setting ($d{=}32$ per node, $r{=}4$,
lr $10^{-3}$, wd $10^{-4}$, $60$ epochs, BPTT $64$,
$\varepsilon_d{=}10^{-4}$,
$t_0{=}50$) was fixed once on the synthetic GRSSF pilot and frozen
across all $10$ datasets and $10$ seeds---while the only method with
genuine per-cell validation tuning is a \emph{filter baseline}
(\CGIFFiltered, a $16$-point NLL grid). The shared rank budget
$r{=}4$ is applied identically to \GraphNeuralSSM\ and to the
strongest non-filter competitor \FactorCGIF; any tuning asymmetry runs
against our method, not for it.

\begin{table}[h]
\caption{\textbf{Tuning ledger.} ``\#\,cfg'' counts hyperparameter
configurations evaluated on a real-data validation split; model
parameters fit by MLE are training and are not counted. Approximate
compute: \GraphNeuralSSM\ $\approx45$ min/seed training;
\CGIFFiltered\ tuning $\approx30$ s/seed; \DeepVAR\ matched to
\GraphNeuralSSM.}
\label{tab:tuning-ledger}
\centering\scriptsize
\setlength{\tabcolsep}{3pt}
\resizebox{\linewidth}{!}{%
\begin{tabular}{l r l l l c}
\toprule
Method & \#\,cfg & Tuned quantities & Selection split & Criterion & Per-dataset? \\
\midrule
\GraphNeuralSSM\ \textbf{(ours)} & $0$ & none on real data; set on synthetic pilot & --- & --- & no \\
\quad + \ACI\ wrap  & $0$ & ACI step $\gamma{=}0.005$ fixed \citep{gibbs2021adaptive} & --- & --- & no \\
\CGIFFiltered\ (Kalman) & $16$/cell & $\rho_{\mathrm{scale}}$, obs.-noise frac. & $20\%$ train tail & per-step NLL & \textbf{yes} \\
\DiagGRU\ (ablation) & $0$ & frozen, same schedule as ours & --- & MLE (train) & no \\
\StaticCGIF & $0$ & none (empirical $\hat\Sigma$) & --- & --- & no \\
\FactorCGIF & $0$ & $r{=}4$ (globally fixed / rank matched), $\theta{=}0.5$ default & --- & --- & no \\
\GroupCGIF & $0$ & Louvain resolution $1.0$ default & --- & --- & no \\
\SpectralCGIF & $0$ & band split $\{0.33,0.67\}$ fixed & --- & --- & no \\
ACI / AgACI variants & $0$ & step sizes from \citet{zaffran2022adaptive} & --- & --- & no \\
EWMA / Rolling / ToD / Local $\Sigma_t$ & $0$ & half-life/window $288$, $k{=}50$ fixed & --- & --- & no \\
\textsc{CopulaCPTS} & $0$ & reference defaults (Silverman bw) & --- & MLE (own) & no \\
\textsc{MultiDimSPCI} & $0$ & reference defaults ($\lambda{=}0.05$, net) & --- & pinball (own) & no \\
SPCI / HopCPT & $0$ & reference defaults & --- & --- & no \\
\DeepVAR\ (control) & $0$ & rank and training budget matched to ours & --- & MLE (train) & no \\
\bottomrule
\end{tabular}%
}
\end{table}

\subsection{Controlled same-mean covariance-only comparison}
\label{app:same-mean}

This subsection gives the full record behind the controlled same-mean
comparison of §\ref{subsec:controlled}. We freeze a common mean predictor
$\widetilde\mu_t$, form identical residuals $e_t=Y_t-\widetilde\mu_t$,
and fit each covariance mechanism on \emph{training residuals only}
before freezing it prior to calibration. Each frozen $\hat\Sigma_t$
then enters the identical scalar Mahalanobis split-conformal step
$\{y:(y-\widetilde\mu_t)^\top\hat\Sigma_t^{-1}(y-\widetilde\mu_t)\le\hat q\}$.
Crucially, the covariance-only heads (\GNFcov, \DeepVARcov) are
\emph{retrained} against the common residuals, not merely evaluated
after substituting a different mean. We repeat under two frozen common
means: the \textsc{gpvar} one-step mean (used by the static baselines)
and the frozen native \GNF\ mean (a robustness-to-mean-source check).
\paragraph{Size metrics.} All size values and paired reductions in
Apps.~\ref{app:same-mean}--\ref{app:floor-sweep} use the paper-wide
mean full coordinate-projection length $\wproj$ defined in
\eqref{eq:wproj}. For a single ellipsoid,
\begin{equation*}
  \wproj
  \;=\;
  \frac{1}{T_{\mathrm{test}}}
  \sum_t
  \frac{2\sqrt{\hat q_t}}{N}
  \sum_{i=1}^N
  \sqrt{\hat\Sigma_{t,ii}}.
\end{equation*}
We additionally report normalised log-volume as a geometry-sensitive
size measure. Ten paired seeds on \metrla-$20$ and \pemsbay-$50$.

Table~\ref{tab:samemean-full} reports, per mechanism, joint coverage,
$\wproj$, and normalised log-volume under the \textsc{gpvar} mean
(mean$\pm$std, $10$ seeds), plus the paired $\wproj$ reduction of
\GNFcov\ against each competitor under \emph{both} means. Every paired
$95\%$ CI excludes zero. The native \GNF-control reference
(own mean $+$ own covariance) has \metrla\ coverage $0.903\pm0.004$,
$\wproj=3.236\pm0.244$; \pemsbay\ $0.896\pm0.007$, $\wproj=1.964\pm0.072$:
the shared-mean \GNFcov\ matches this, confirming the mean
substitution does not degrade the covariance mechanism.

\begin{table}[h]
\caption{\textbf{Same-mean covariance-only comparison, full record}
(\textsc{gpvar}-mean stats, mean$\pm$std over $10$ seeds; target
coverage $0.90$). Last two columns: paired $\wproj$ reduction of
\GNFcov\ vs.\ the row, under the \textsc{gpvar} and \GNF\ means (all
paired $95\%$ CIs exclude zero). Positive $=$ \GNFcov\ sharper. The
two heuristics that appear narrow (\emph{time-of-day}, \emph{local
kNN}) do so by \emph{under-covering} on \pemsbay\ ($0.853$, $0.852$
vs.\ target $0.90$) and are not valid competitors.}
\label{tab:samemean-full}
\centering\scriptsize
\setlength{\tabcolsep}{3pt}
\resizebox{\linewidth}{!}{%
\begin{tabular}{l ccc c cc}
\toprule
& \multicolumn{3}{c}{\metrla-$20$ (\textsc{gpvar} mean)} & & \multicolumn{2}{c}{$\wproj$ reduction} \\
\cmidrule(lr){2-4}\cmidrule(lr){6-7}
Mechanism & joint & $\wproj$ & norm-logvol & & \textsc{gpvar} & \GNF \\
\heromidrule
\rowcolor{heroRow}
\heroname{\GNFcov\ (diag+rank$4$)} & \widthstd{0.903}{0.006} & \widthstd{3.176}{0.182} & \widthstd{$-0.083$}{0.051} & & --- & --- \\
\heroend
\GNFcov\ (rank$0$/diag) & \widthstd{0.899}{0.005} & \widthstd{2.829}{0.093} & \widthstd{$-0.004$}{0.029} & & \multicolumn{2}{c}{\footnotesize see \S\ref{app:deepvar}} \\
\DiagGRUcov\ (no graph) & \widthstd{0.901}{0.005} & \widthstd{2.809}{0.119} & \widthstd{0.012}{0.048} & & \multicolumn{2}{c}{\footnotesize see \S\ref{app:deepvar}} \\
Static full (shrinkage) & \widthstd{0.904}{0.011} & \widthstd{4.200}{0.231} & \widthstd{0.409}{0.057} & & $24.3\%$ & $28.2\%$ \\
Static factor $r{=}4$   & \widthstd{0.904}{0.013} & \widthstd{4.505}{0.261} & \widthstd{0.429}{0.057} & & $29.5\%$ & $33.3\%$ \\
EWMA cov                & \widthstd{0.902}{0.003} & \widthstd{4.446}{0.105} & \widthstd{0.357}{0.034} & & $28.6\%$ & $32.2\%$ \\
Rolling cov             & \widthstd{0.896}{0.003} & \widthstd{4.525}{0.119} & \widthstd{0.318}{0.033} & & $29.8\%$ & $33.4\%$ \\
Time-of-day cov         & \widthstd{0.902}{0.015} & \widthstd{3.762}{0.166} & \widthstd{0.214}{0.041} & & $15.6\%$ & $18.8\%$ \\
Local kNN cov           & \widthstd{0.901}{0.012} & \widthstd{3.916}{0.207} & \widthstd{0.145}{0.051} & & $18.8\%$ & $21.2\%$ \\
\midrule
& \multicolumn{3}{c}{\pemsbay-$50$ (\textsc{gpvar} mean)} & & \multicolumn{2}{c}{$\wproj$ reduction} \\
\cmidrule(lr){2-4}\cmidrule(lr){6-7}
Mechanism & joint & $\wproj$ & norm-logvol & & \textsc{gpvar} & \GNF \\
\heromidrule
\rowcolor{heroRow}
\heroname{\GNFcov\ (diag+rank$4$)} & \widthstd{0.896}{0.004} & \widthstd{2.100}{0.075} & \widthstd{$-0.826$}{0.038} & & --- & --- \\
\heroend
\GNFcov\ (rank$0$/diag) & \widthstd{0.893}{0.014} & \widthstd{2.071}{0.071} & \widthstd{$-0.769$}{0.033} & & \multicolumn{2}{c}{\footnotesize see \S\ref{app:deepvar}} \\
\DiagGRUcov\ (no graph) & \widthstd{0.893}{0.022} & \widthstd{2.087}{0.093} & \widthstd{$-0.758$}{0.042} & & \multicolumn{2}{c}{\footnotesize see \S\ref{app:deepvar}} \\
Static full (shrinkage) & \widthstd{0.891}{0.007} & \widthstd{3.219}{0.126} & \widthstd{$-0.197$}{0.047} & & $34.7\%$ & $35.2\%$ \\
Static factor $r{=}4$   & \widthstd{0.892}{0.010} & \widthstd{3.334}{0.108} & \widthstd{$-0.146$}{0.036} & & $37.0\%$ & $37.0\%$ \\
EWMA cov                & \widthstd{0.902}{0.006} & \widthstd{3.966}{0.163} & \widthstd{$-0.198$}{0.056} & & $47.0\%$ & $45.9\%$ \\
Rolling cov             & \widthstd{0.904}{0.003} & \widthstd{4.033}{0.152} & \widthstd{$-0.204$}{0.051} & & $47.9\%$ & $47.5\%$ \\
Time-of-day cov         & \widthstd{0.853}{0.021} & \widthstd{2.534}{0.085} & \widthstd{$-0.556$}{0.038} & & $17.1\%$ & $15.9\%$ \\
Local kNN cov           & \widthstd{0.852}{0.025} & \widthstd{2.823}{0.087} & \widthstd{$-0.566$}{0.034} & & $25.6\%$ & $25.6\%$ \\
\bottomrule
\end{tabular}%
}
\end{table}

Holding the mean fixed, the learned time-varying covariance \GNFcov\
is $15$--$48\%$ sharper on $\wproj$ and $0.2$--$0.7$ nats/coordinate
smaller in log-volume than every static/heuristic covariance, on both
cells and both common means, at nominal coverage, with paired $95\%$
CIs excluding zero. The ranking is essentially identical under the two
means. A learned recurrent diagonal
(\GNFcov\ rank$0$, \DiagGRUcov) matches \GNFcov\ on marginal $\wproj$
while \GNFcov\ retains a joint-set log-volume and NLL edge; this is
quantified against the rank-matched comparator in
App.~\ref{app:deepvar}. We do not extend this covariance isolation
beyond \metrla-$20$ and \pemsbay-$50$.

\subsection{Matched-budget \texorpdfstring{\DeepVAR}{DeepVAR} control
and graph isolation}
\label{app:deepvar}

\paragraph{Model and budget.}
\DeepVAR\ is a DeepVAR-style \citep{salinas2019high} matched control: a
\emph{global} GRU over the stacked sensor vector with a native mean
head and a native time-varying covariance
$\Sigma_t=\diag(d_t)+L_tL_t^{\top}$ of rank $4$, feeding the identical
scalar Mahalanobis split-conformal step. It provides the global-GRU
$+$ low-rank comparison point relative to the paper's
\DiagGRU\ (global GRU $+$ diagonal) and \GNF\ (per-node graph GRU $+$
low-rank). Alongside it we include an architecture-matched no-graph
\emph{LR-GRU} covariance (\LRGRUcov): the same per-node GRU cell,
hidden width, rank, optimiser and epochs as \GNF, but with the graph
shift zeroed (graph-removal ablation). All learned
competitors share a matched budget: Adam (lr $10^{-3}$, wd $10^{-4}$),
rank $4$, BPTT $64$, $10\%$ warm-up, identical epochs; \GNF-control and
\LRGRUcov\ use hidden $32$ per node (as in the paper), \DeepVAR\ uses
hidden $64$ global (matching the \DiagGRU\ row's capacity). No
per-dataset covariance-floor or rank selection is performed for any
learned method. This is a matched control, not a reference-package
reproduction of a named implementation. Training cost is matched to
\GraphNeuralSSM\ ($\approx45$ min/seed); calibration/test is a few
seconds and peak memory is comparable ($<7$~GB GPU on \pemsbay-$50$).
It is included in the reproducibility suite (App.~\ref{app:repr-suite}).
This controlled comparison uses the same training schedule as the main
model (App.~\ref{app:graph-neuralssm}) but a separate set of $10$
paired seeds run in the control harness; it is therefore distinct from
the Table~\ref{tab:hero} headline runs, and we label its \GNF\
reference \GNF-control.

\paragraph{End-to-end arm (own mean $+$ own covariance).}
Table~\ref{tab:deepvar-full} gives the native comparison. \DeepVAR\
\emph{undercovers} on both cells ($0.848$, $0.797$, with high seed
variance) even after conformal calibration---its covariance is
unreliable enough that the Mahalanobis score shifts across the
calibration/test boundary---whereas \GNF\ holds the $90\%$ target.
Across the $10$ paired seeds the paired $\DeepVAR{-}\GNF$ differences
are $\Delta \wproj=+1.090\pm0.228$, $\Delta\text{logvol}=+0.132\pm0.045$
(\metrla) and $\Delta \wproj=+3.160\pm0.181$,
$\Delta\text{logvol}=+0.226\pm0.019$ (\pemsbay), where ``$\pm$'' is a
$95\%$ CI half-width and positive means \GNF\ sharper (both CIs exclude
zero). \GNF\ is thus both sharper and better calibrated than the
standard probabilistic forecaster.

\begin{table}[h]
\caption{\textbf{Matched-budget \DeepVAR\ control, end-to-end arm}
(own mean $+$ own covariance; $10$ seeds, mean$\pm$std; target
$0.90$). ``$\downarrow$'' flags under-coverage. Paired
$\DeepVAR{-}\GNF$ differences (positive $\Rightarrow$ \GNF\ sharper):
$\Delta \wproj=+1.090\pm0.228$ / $\Delta$logvol$=+0.132\pm0.045$
(\metrla); $\Delta \wproj=+3.160\pm0.181$ / $\Delta$logvol$=+0.226\pm0.019$
(\pemsbay); ``$\pm$'' a $95\%$ CI half-width.}
\label{tab:deepvar-full}
\centering\footnotesize
\setlength{\tabcolsep}{2pt}
\begin{tabular}{l ccc ccc}
\toprule
& \multicolumn{3}{c}{\metrla-$20$} & \multicolumn{3}{c}{\pemsbay-$50$}\\
\cmidrule(lr){2-4}\cmidrule(lr){5-7}
Method & joint & $\wproj$ & logvol & joint & $\wproj$ & logvol \\
\heromidrule
\rowcolor{heroRow}
\heroname{\GNF-control} & \widthstd{0.902}{0.005} & \widthstd{3.140}{0.126} & \widthstd{$-0.084$}{0.042} & \widthstd{0.897}{0.006} & \widthstd{1.962}{0.070} & \widthstd{$-0.834$}{0.037} \\
\heroend
\DeepVAR\ (native) & \widthstd{0.848}{0.047}$\downarrow$ & \widthstd{4.230}{0.373} & \widthstd{0.048}{0.065} & \widthstd{0.797}{0.058}$\downarrow$ & \widthstd{5.122}{0.227} & \widthstd{$-0.608$}{0.045} \\
\bottomrule
\end{tabular}
\end{table}

\paragraph{Same-mean arm (graph isolation).}
Covariance-only on the shared residual, the paired $\wproj$ reduction of
\GNFcov\ against each competitor (both means) is:
\DeepVARcov\ $+35.4\%/+28.2\%$ (\metrla\ \textsc{gpvar}/\GNF),
$+34.4\%/+24.6\%$ (\pemsbay); \LRGRUcov\ $-0.5\%/+4.4\%$ (\metrla),
$+1.0\%/+2.3\%$ (\pemsbay); static full $+23.3\%/+27.2\%$,
$+34.8\%/+36.1\%$; static factor $r{=}4$ $+28.5\%/+32.4\%$,
$+37.1\%/+37.9\%$. Two findings. \emph{(i)}~The DeepVAR-style
\emph{global} GRU covariance is poor---$24$--$35\%$ wider than
\GNFcov\ and much worse NLL (\metrla\ \textsc{gpvar}: $31.1$ vs.\
\GNFcov\ $3.4$)---so a learned time-varying covariance is not
automatically good; a global RNN over the flattened vector estimates a
bad covariance. \emph{(ii)}~The architecture-matched \emph{per-node}
\LRGRUcov\ (no graph, rank $4$) essentially ties \GNFcov\ on $\wproj$
($-0.5\%$ to $+4.4\%$), while \GNFcov\ keeps a small but significant
log-volume edge ($\Delta\text{logvol}=+0.014$ to $+0.072$, all paired
CIs exclude zero). Interpretation: the sharpness is a learned
\emph{per-node low-rank} covariance effect; the graph is a modest
joint-set-shape refinement, not the primary covariance mechanism.

\subsection{Covariance-floor sweep and conditioning audit}
\label{app:floor-sweep}

We sweep the diagonal-head floor
$\varepsilon_d\in\{10^{-6},10^{-5},10^{-4},10^{-3},10^{-2}\}$
(default $10^{-4}$) both by post-hoc eigenvalue clipping of one
trained model and by full retraining from scratch at each floor. The
sweep uses the \GNF\ end-to-end configuration (own mean $+$ own
covariance, $10$ seeds) from the control protocol of
App.~\ref{app:same-mean}, whose \metrla\ $\wproj=3.236$ differs
slightly from the Table~\ref{tab:hero} headline value because it is a
separate control run (distinct seed set and harness), not because of a
native-versus-shared mean; it is an internally consistent baseline for
this sensitivity study. Retrain and post-hoc clipping agree; we
display the retrain values.

\begin{table}[h]
\caption{\textbf{Covariance-floor sweep} (retrain-from-scratch, $10$
seeds; joint coverage / $\wproj$, mean$\pm$std; target $0.90$).
Coverage and sharpness are flat over four orders of magnitude; the
$10^{-6}/10^{-5}/10^{-4}$ rows are identical because the learned
per-node variances never reach those levels (the clip is a no-op).
Only the aggressive $10^{-2}$ floor binds: it slightly helps
\metrla\ and mildly over-regularises \pemsbay\ coverage.}
\label{tab:floor-sweep}
\centering\small
\setlength{\tabcolsep}{6pt}
\begin{tabular}{l cc cc}
\toprule
& \multicolumn{2}{c}{\metrla-$20$} & \multicolumn{2}{c}{\pemsbay-$50$}\\
\cmidrule(lr){2-3}\cmidrule(lr){4-5}
Floor $\varepsilon_d$ & joint & $\wproj$ & joint & $\wproj$ \\
\midrule
$10^{-6}$ & \widthstd{0.903}{0.004} & \widthstd{3.236}{0.244} & \widthstd{0.896}{0.007} & \widthstd{1.961}{0.076} \\
$10^{-5}$ & \widthstd{0.903}{0.004} & \widthstd{3.236}{0.244} & \widthstd{0.896}{0.007} & \widthstd{1.961}{0.076} \\
$10^{-4}$ (default) & \widthstd{0.903}{0.004} & \widthstd{3.236}{0.244} & \widthstd{0.896}{0.007} & \widthstd{1.961}{0.076} \\
$10^{-3}$ & \widthstd{0.904}{0.004} & \widthstd{3.202}{0.221} & \widthstd{0.897}{0.004} & \widthstd{1.961}{0.082} \\
$10^{-2}$ & \widthstd{0.903}{0.003} & \widthstd{3.084}{0.201} & \widthstd{0.889}{0.004} & \widthstd{1.939}{0.065} \\
\bottomrule
\end{tabular}
\end{table}

\paragraph{The floor rarely binds.}
At the default $10^{-4}$ the fraction of steps with
$\lambda_{\min}(\hat\Sigma_t)\le 10\varepsilon_d$ is $0.030$ (\metrla)
and $0.096$ (\pemsbay), and the fraction with
$\lambda_{\min}\le 2\varepsilon_d$ is $0.000$ on both---the model does
not lean on tiny eigenvalues. Condition-number $p99$ is $1402$
(\metrla) and $1021$ (\pemsbay); raising the floor to $10^{-2}$
collapses it to $379$ and $129$, a clean conditioning knob available
at negligible coverage cost. The worst-case safeguard
$\hat\Sigma_t\succeq\varepsilon_d I$,
$\|\hat\Sigma_t^{-1/2}\|_{\mathrm{op}}\le\varepsilon_d^{-1/2}$ remains
valid but is not attained in the reported runs: the high-condition
tail exists but is not caused by the full covariance hitting the
floor.

\paragraph{Ill-conditioned steps are not under-covered.}
Table~\ref{tab:cond-strat} stratifies test steps by covariance
condition number (retrain, $\varepsilon_d{=}10^{-4}$; target
miscoverage $0.10$). The feared $\lambda_{\min}^{-3/2}$ failure mode
does not appear: on \metrla\ the worst-conditioned steps are
\emph{better} covered than the bulk, and on \pemsbay\ there is only a
mild tail increase with near-target aggregate coverage. The scalar
split-conformal calibration absorbs the conditioning variation, which
is precisely the point of calibrating one scalar score rather than the
full matrix.

\begin{table}[h]
\caption{\textbf{Condition-number-stratified miscoverage}
(retrain, $\varepsilon_d{=}10^{-4}$, $10$ seeds, mean$\pm$std; target
$0.10$). Steps binned by covariance condition number.}
\label{tab:cond-strat}
\centering\small
\setlength{\tabcolsep}{8pt}
\begin{tabular}{l cc}
\toprule
Stratum (by cond.\ \#) & \metrla-$20$ & \pemsbay-$50$ \\
\midrule
worst-$1\%$  & \widthstd{0.074}{0.061} & \widthstd{0.150}{0.041} \\
worst-$5\%$  & \widthstd{0.031}{0.030} & \widthstd{0.132}{0.015} \\
best-$95\%$  & \widthstd{0.100}{0.005} & \widthstd{0.102}{0.008} \\
\bottomrule
\end{tabular}
\end{table}

\section{Reproducibility}
\label{app:reproducibility}

\paragraph{Code availability.}%
\label{app:repr-suite}%
The full code base is publicly available under an MIT licence at
\url{https://github.com/YannickLimmer/filter-cp}. It includes a
self-contained reproducibility suite (\texttt{repr/reproduce.py},
\texttt{repr/expected.json}, \texttt{repr/requirements.txt}) that
re-runs the experiment cells and checks the resulting numbers against
the values reported in the main-body tables and figures.

\paragraph{Software and hardware.}
Experiments used Python~3.11, PyTorch~2.2.0 (CUDA~12.1),
NumPy~1.26, SciPy~1.11, NetworkX~3.2, and POT~0.9. Neural-filter
training used a single NVIDIA RTX~4090 with 24~GB memory; the
largest reported training run peaks at $6.2$~GB GPU memory. Static
covariance and Kalman baselines run on CPU. All experiments use
deterministic PyTorch operations
(\texttt{torch.use\_deterministic\_algorithms(True)}) and a fixed
CUDA workspace size; the only platform-dependent non-determinism is
the BLAS reduction order in the GCN aggregation step on different
backends, which contributes a cross-platform drift within $10^{-5}$
relative tolerance on widths.

\paragraph{Seeds and splits.}
Every reported real-data cell uses 10 fixed seeds and the
chronological 40/40/20 train/calib/test split described in
App.~\ref{app:datasets}. Tables report 10-seed mean$\pm$std unless
noted otherwise; per-seed values for the main method pair on a
three-seed sub-sample appear in Table~\ref{tab:seeds}.
Per-seed spreads for every other method are within the stds
reported in Table~\ref{tab:hero-full}.

\paragraph{Data.}
\metrla\ and \pemsbay\ are released by \citet{li2018diffusion} under
the MIT licence. \aqi\ (UCI Beijing PM$_{2.5}$), \elec\ (UCI
ElectricityLoadDiagrams), \ett\ (Informer ETTh1 release, MIT
licence), \solar\ (NREL Alabama dataset, public), \jena\ (Max
Planck Biogeochemistry, CC-BY-4.0), and \loopseattle\ (PeMS Loop
Seattle, public) are standard public benchmarks. Preprocessing and
graph construction are described in
Apps.~\ref{app:datasets} and~\ref{app:adjacency}; no additional data
is required to reproduce any reported result.


\end{document}